\documentclass{article}

\usepackage{arxiv}

\usepackage{lipsum}		

\usepackage{amsmath,amsfonts,bm}

\def\eqref#1{equation~\ref{#1}}

\def\1{\bm{1}}

\DeclareMathAlphabet{\mathsfit}{\encodingdefault}{\sfdefault}{m}{sl}
\SetMathAlphabet{\mathsfit}{bold}{\encodingdefault}{\sfdefault}{bx}{n}

\DeclareMathOperator*{\argmax}{arg\,max}

\usepackage[utf8]{inputenc}
\usepackage[T1]{fontenc}

\usepackage{microtype}
\usepackage{graphicx}
\usepackage{subcaption}
\usepackage{booktabs}
\usepackage{placeins}

\usepackage{amsmath}
\usepackage{amssymb}
\usepackage{amsfonts}
\usepackage{amsthm}
\usepackage{mathtools}
\usepackage{bm}
\usepackage{bbm}
\usepackage{nicefrac}

\usepackage{xcolor}
\usepackage{soul}
\usepackage{comment}
\usepackage{enumitem}
\usepackage{refcount}
\usepackage[multiple]{footmisc}
\usepackage[ruled,vlined]{algorithm2e}
\usepackage[textsize=tiny]{todonotes}
\usepackage{natbib}

\definecolor{citecolor}{HTML}{0071BC}
\definecolor{linkcolor}{HTML}{ED1C24}
\usepackage[
  pagebackref=false,
  breaklinks=true,
  letterpaper=true,
  colorlinks,
  bookmarks=false,
  citecolor=citecolor,
  linkcolor=linkcolor,
  urlcolor=gray
]{hyperref}

\usepackage[capitalize,noabbrev]{cleveref}

\theoremstyle{plain}
\newtheorem{theorem}{Theorem}[section]
\newtheorem{proposition}[theorem]{Proposition}

\theoremstyle{definition}

\theoremstyle{remark}

\newsavebox{\imagebox}

\newcommand{\HVI}{\textsc{HVI}}

\newcommand{\SPI}{\textsc{SPI}}

\newcommand{\aESPI}{\alpha_{\textsc{ESPI}}}
\newcommand{\hataESPI}{\hat{\alpha}_{\textsc{ESPI}}}
\newcommand{\aNESPI}{\alpha_{\textsc{NESPI}}}
\newcommand{\hataNESPI}{\hat{\alpha}_{\textsc{NESPI}}}

\newcommand{\best}[1]{\textbf{#1}}

\newcommand{\papertitle}{Do We Really Need to Approach the Entire Pareto Front in Many-Objective Bayesian Optimisation?}

\title{\papertitle}

\author{ Chao Jiang \\
	School of Computer Science\\
	University of Birmingham\\
	\texttt{cxj249@student.bham.ac.uk} \\
	\And
    Jingyu Huang \\
	School of Mathematics\\
	University of Birmingham\\
	\texttt{j.huang.4@bham.ac.uk} \\
	\And
	Miqing Li\thanks{{Corresponding author: m.li.8@bham.ac.uk}} \\
	School of Computer Science\\
	University of Birmingham\\
	\texttt{m.li.8@bham.ac.uk} \\
}

\renewcommand{\shorttitle}{\textit{arXiv} Template}

\hypersetup{
pdftitle={A template for the arxiv style},
pdfsubject={q-bio.NC, q-bio.QM},
pdfauthor={David S.~Hippocampus, Elias D.~Striatum},
pdfkeywords={First keyword, Second keyword, More},
}

\begin{document}
\maketitle

\begin{abstract}
	Many-objective optimisation, a subset of multi-objective optimisation, involves optimisation problems with more than three objectives. As the number of objectives increases, the number of solutions needed to adequately represent the entire Pareto front typically grows substantially. This makes it challenging, if not infeasible, to design a search algorithm capable of effectively exploring the entire Pareto front. This difficulty is particularly acute in the Bayesian optimisation paradigm, where sample efficiency is critical and only a limited number of solutions (often a few hundred) are evaluated. Moreover, after the optimisation process, the decision-maker eventually selects just one solution for deployment, regardless of how many high-quality, diverse solutions are available. In light of this, we argue an idea that under a limited evaluation budget, it may be more useful to focus on finding a single solution of the highest possible quality for the decision-maker, rather than aiming to approximate the entire Pareto front as existing many-/multi-objective Bayesian optimisation methods typically do. Bearing this idea in mind, this paper proposes a \underline{s}ingle \underline{p}oint-based \underline{m}ulti-\underline{o}bjective search framework (SPMO) that aims to improve the quality of solutions along a direction that leads to a good tradeoff between objectives. Within SPMO, we present a simple acquisition function, called expected single-point improvement (ESPI), working under both noiseless and noisy scenarios. We show that ESPI can be optimised effectively with gradient-based methods via the sample average approximation (SAA) approach and theoretically prove its convergence guarantees under the SAA. We also empirically demonstrate that the proposed SPMO is computationally tractable and outperforms state-of-the-arts on a wide range of benchmark and real-world problems.
\end{abstract}


\section{Introduction}

Multi-objective optimisation problems (MOPs)~\citep{emmerich2018tutorial,zheng2024boundary} involve scenarios where multiple objectives need to be optimised simultaneously. Unlike single-objective optimisation problems which typically have a single optimal solution, in MOPs there is a set of optimal solutions known as Pareto optimal solutions. The corresponding points in the objective space form what is known as the Pareto front. In general, a multi-objective optimisation algorithm aims to generate a set of solutions that well approximate the Pareto front, from which the decision-maker chooses a solution to deploy based on their preferences.  

In modern applications, optimisation is becoming increasingly complex, with a growing number of requirements and objectives that need to be considered at the same time~\citep{matrosov2015many,hierons2020many,lin2025few}. 
Taking the car cab design as an example, there are up to nine objectives to be optimised~\citep{deb2013evolutionary}, 
including cabin space, fuel efficiency, acceleration time, and road noise at various speeds. 
This has given rise to a new research topic, many-objective optimisation, focusing on MOPs involving more than three objectives~\citep{ishibuchi2008evolutionary,li2015many}.


At the same time, many real-world multi-/many-objective optimisation problems are black-box and costly in terms of solution evaluation. This is evident across a range of fields, including chemistry~\citep{park2018multi,shields2021bayesian,dunlap2023continuous}, materials science~\citep{liang2021benchmarking,low2024evolution,peng2025bayesian}, and transportation~\citep{deb2013evolutionary,jain2013evolutionary,cheaitou2019greening,deb2009reliability}. 
For example, in vehicle design optimisation, it can take about 20 hours to evaluate a vehicle design~\citep{youn2004reliability,daulton2021parallel}.
To tackle such problems, multi-objective Bayesian optimisation (MOBO) is a very effective approach~\citep{garnett2023bayesian}, along with other alternatives (e.g., surrogate-assisted evolutionary algorithms~\citep{jin2011surrogate,liang2024survey}). 
Over the past decades, a variety of effective MOBO methods have emerged, including scalarisation-based methods~\citep{knowles2006parego,paria2020flexible,lin2022pareto} which convert a multi-objective problem into a number of single-objective problems, and Pareto-based methods which consider Pareto dominance relations over objectives~\citep{daulton2020differentiable,emmerich2006single,tu2022joint}. Most works aim to find a good approximation of the entire Pareto front.  

However, with the increase of the objective number, the number of solutions needed to adequately represent the problem's Pareto front typically grows substantially. For example, for a 10-objective problem, it normally needs $220$ points (i.e., $\binom{12}{3}$) even if only three divisions on each objective are considered~\citep{das1998normal}.  
This difficulty is especially pronounced in Bayesian optimisation, where often only a few hundred solutions can be generated and evaluated. With such a limited budget, it is highly unlikely for an optimisation algorithm to reach or even be close to the Pareto front.
On top of that, after the optimisation process, the decision-maker eventually selects just one solution to deploy, regardless of how many high-quality, diverse solutions are provided.

Given the above, this paper argues an idea that under a very limited evaluation budget, it may be more useful to focus on finding a single solution of the highest possible quality for the decision-maker, rather than aiming to approximate the entire Pareto front.
That is, we may not need to care about diversifying solutions to represent the entire Pareto front, but focus on improving the quality of a single solution. Figure~\ref{fig:motivation} illustrates this idea in a bi-objective case. As can be seen from the figure, in contrast to aiming for a set of diversified solutions which existing MOBO methods typically do~\citep{li2025expensive,lin2022pareto}, our method aims for a single solution with better convergence (i.e., closer to the Pareto front). Although it may yield a worse \textit{hypervolume} (HV) value~\citep{Zitzler1999} compared to the diverse solution set obtained by existing methods, it may be more likely to be chosen by the decision-maker as it achieves a more favourable trade-off among the objectives.

\begin{figure}[h]
    \centering
    \begin{minipage}{0.42\linewidth}{
    \includegraphics[width=0.95\linewidth]{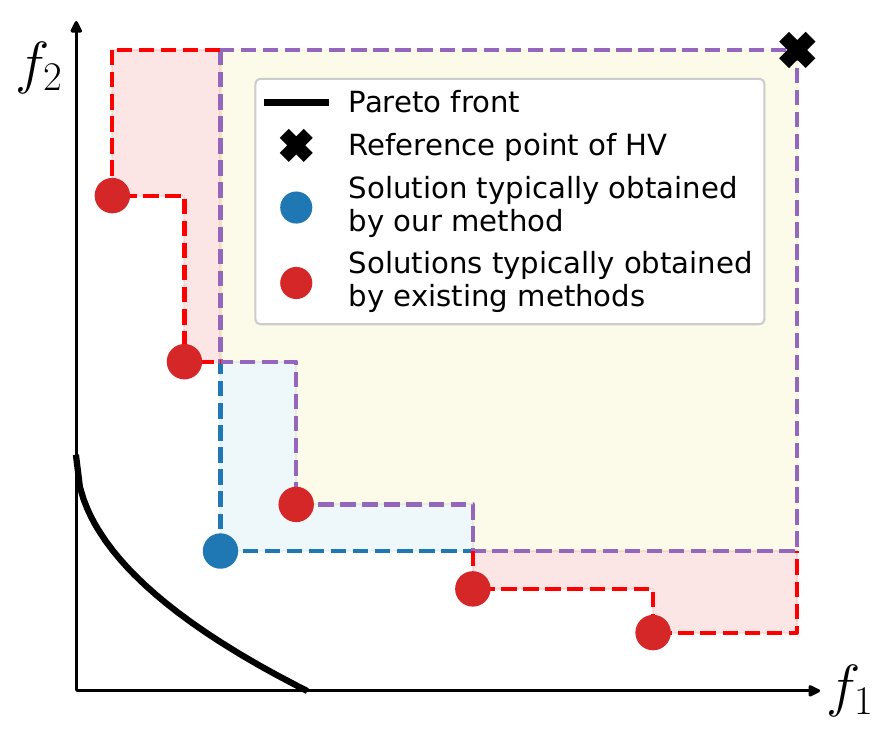}	
    }\end{minipage}
    \begin{minipage}{0.57\linewidth}{
    \caption{An illustration of our idea in a bi-objective case, in comparison with existing methods that aim to search for the entire Pareto front. The red points represent (nondominated) solutions that existing methods may obtain, and the blue point represents what our method aims for. It can be seen that the red points are much more diversified, hence, as a whole, having a better hypervolume (HV) value~\citep{Zitzler1999}. However, the blue point has better convergence (i.e., closer to the Pareto front) than any single red point, which may be more likely to be preferred by the decision-maker. 
    }
    \label{fig:motivation}
    }\end{minipage}
\end{figure}


Bearing this idea in mind, this paper proposes a single point-based multi-objective search framework, called SPMO.
The contributions of this work can be summarised as follows. 

\begin{itemize}
    \item We propose a novel MOBO search framework that does not aim to approximate the entire Pareto front, but rather focuses on improving the quality of solutions along a single direction that leads to a good tradeoff between objectives. 
    
    \item Within SPMO, we present a simple acquisition function, called Expected Single-Point Improvement (ESPI). We show that ESPI can be optimised effectively with gradient-based methods via the sample average approximation (SAA) approach from~\citet{balandat2020botorch} and also theoretically prove its theoretical convergence guarantees under the SAA. 

    \item We consider both noiseless and noisy cases, resulting in two versions of the proposed ESPI. 

    \item We verify SPMO through an extensive experimental study, including in comparison with various state-of-the-arts, under both sequential and batch optimisation settings, through sensitivity analysis, with different metrics embedded, and on a range of benchmark and real-world problems.

\end{itemize}

\section{Background and Related Work}
\subsection{Background}
\paragraph{Many-Objective Optimisation.} 
Many-objective optimisation, a subset of multi-objective optimisation, refers to an optimisation scenario having more than three objectives to be considered. 
Without loss of generality, this paper considers the problem of minimising a vector-valued function: $\bm{f}(\bm{x}): \mathcal{X} \to \mathbb{R}^m$, where $\bm{x}\in \mathcal{X}$ ($\mathcal{X} \subset \mathbb{R}^{d}$) and $m$ is the number of objectives. 
In multi-/many-objective optimisation, 
a solution $\bm x_1$ is said to dominate $\bm x_2$, denoted by $\bm x_1\prec \bm x_2$, if $\forall i \in\{1,...,m\}$, $f_i(\bm x_1) \le f_i(\bm x_2)$ and $\exists j \in \{1,...,m\}$, $f_j(\bm x_1) < f_j(\bm x_2)$. 
If a solution $\bm x_1\in\mathcal{X}$ is not dominated by any other solution, then $\bm x_1$ is said to be Pareto optimal. 
The collection of Pareto optimal solutions of a problem is called the Pareto set, and its mapping to the objective space is called Pareto front. 


\paragraph{Bayesian Optimisation (BO).} 
BO is a sample-efficient global optimisation approach that builds a probabilistic surrogate, typically a Gaussian process (GP), and uses an acquisition function $\alpha(\bm x): \mathcal{X} \to \mathbb{R}$ to decide which points to evaluate. 
In this work, we model each objective with an independent Gaussian process $f_i \sim \mathcal{GP}(m_i(\bm x),k_i(\bm x,\bm x'))$, where $m_i(\bm x): \mathcal{X} \to \mathbb{R}$ is the $i$th mean function, and $k_i(\cdot,\cdot): \mathcal{X} \times \mathcal{X} \to \mathbb{R}$ is the $i$th covariance function. 
Given $n$ observed points $\mathcal{D}^n= \{(\bm x^t,\bm y^t)\}_{t=1}^n$ where $\bm y^t = \bm{f}(\bm x^t) + \bm\zeta^t$ and the noise $\bm \zeta^t \sim \mathcal{N}(0, \text{diag} (\bm \sigma_\zeta^2))$, the posterior distribution of the $i$th objective at a new location $\bm x$ is a Gaussian distribution: 
$p(f_i(\bm x)|\mathcal{D}^n) \sim \mathcal{N}(\mu_i(\bm x),\sigma_i^2(\bm x))$ where $\mu_i(\bm x)$ and $\sigma_i^2(\bm x)$ are the mean and variance at $\bm x$, respectively. 
Detailed expressions of the mean and variance are given in Appendix~\ref{appdx:sec:BO}. 

\subsection{Related Work}


Over the past decades, various MOBO methods have been proposed~\citep{konakovic2020diversity,daulton2022robust}. 
They can be loosely divided into scalarisation-based and Pareto-based methods. 
In scalarisation-based methods~\citep{knowles2006parego,paria2020flexible}, a multi-objective problem is converted into a number of single-objective problems~\citep{chugh2020scalarizing}. 
Hence, one can leverage acquisition functions~\citep{lai1985asymptotically} from single-objective BO to decide which point to evaluate. 
For instance, using random augmented Tchebycheff scalarisations~\citep{miettinen1999nonlinear}, ParEGO~\citep{knowles2006parego} and TS-TCH~\citep{paria2020flexible} optimise expected improvement (EI)~\citep{jones1998efficient} and Thompson sampling (TS)~\citep{thompson1933likelihood}, respectively. 
In contrast, Pareto-based methods consider Pareto dominance relations over objectives~\citep{emmerich2006single,tu2022joint}. 
A popular idea is to use HV as maximising the HV value is equivalent to finding the entire Pareto front~\citep{shang2020survey}. 
Expected hypervolume improvement (EHVI) is commonly used in MOBO~\citep{couckuyt2014fast,daulton2020differentiable,daulton2021parallel}.
Another idea is to leverage information theory to guide exploration toward regions likely contributing to the Pareto front. 
For instance, joint entropy search (JES)~\citep{tu2022joint} selects points that maximise the joint information gain for optimal inputs (i.e., the approximated Pareto set) and outputs (i.e., the approximated Pareto front). 
All of the above methods aim to approach the entire Pareto front. 

It is worth noting that, similar to our approach, a few studies do not attempt to approximate the entire Pareto front. Some methods instead target a specific region of the front, such as the central area~\citep{gaudrie2018budgeted,gaudrie2020targeting,binois2020kalai}. Another line of work incorporates decision-maker preferences by dynamically adjusting the target region based on elicited or updated preferences during optimisation~\citep{abdolshah2019multi,astudillo2020multi,ozaki2024multi,ip2025user}. In contrast, our method assumes no prior knowledge of decision-maker preferences and seeks to identify a high-quality trade-off solution across objectives. A more detailed discussion of related work can be found in Appendix~\ref{Appendix:related_work}.

\section{The Proposed Method}


In this section, we first give the proposed MOBO framework. We then present the considered acquisition function (called ESPI), which is based on a simple distance-based metric. We note that analytically solving ESPI is not feasible, and thus consider its Monte Carlo approximation. Lastly, we consider ESPI under noisy cases, namely noisy ESPI (NESPI), and also its MC approximation.

\begin{algorithm*}[!ht]
    \DontPrintSemicolon
    \caption{Single Point-based Multi-Objective (SPMO) Framework}
    \label{alg:SPMO}
    
    \KwIn {$\bm f$: Expensive black-box problem with $m$ objectives;
    $T$: Maximum number of evaluations; \\
    \qquad \quad $g$: Metric that measures the quality of a single point;
    $\alpha$: Acquisition function; \\
    \qquad \quad $\mathcal{D}^{n_0}:= \{(\bm x^t,\bm{y}^t)\}_{t=1}^{n_0} $: Initial observed points.

    }
    
    \nl \For{$n=n_0+1:T$}{
        \nl $GPs \leftarrow$ Train $\mathcal{GP}s(\mathcal{D}^{n})$  \tcp*{Train $m$ Gaussian process models}
        \nl $\bm x^n \leftarrow \arg\max_{\bm x\in \mathcal{X}}\alpha\big(g(\bm x, GPs)\big)$ \tcp*{Maximise the acquisition function based \hspace*{21.3em} on the single-point quality metric $g(\cdot)$}
        \nl $\bm y^n \leftarrow \bm f(\bm x^n)+\bm\zeta^n$ \tcp*{Evaluate the solution $\bm x^n$}
        \nl $\mathcal{D}^{n} \leftarrow \mathcal{D}^{n-1} \cup \{(\bm x^n,\bm{y}^n)\}$  \tcp*{Augment the observed solution}
    }

    \KwOut{$\mathcal{D}^{T}$: Observed solutions.}

\end{algorithm*}








\paragraph{Single Point-based Multi-Objective (SPMO) Framework.}



Algorithm~\ref{alg:SPMO} gives the procedure of the proposed SPMO framework. As can be seen, SPMO is very similar to a standard MOBO algorithm, except for the step of maximising the acquisition function based on a single-point quality metric (line 3). In principle, any metric that can reflect the quality of a solution in achieving a good trade-off between objectives can be adopted. This includes distance-based metrics and scalarisation-based metrics, such as the weighted sum or augmented Tchebycheff~\citep{miettinen1999nonlinear} with a fixed weight vector (e.g., $(\frac{1}{m},\dots,\frac{1}{m}) \in \mathbb{R}^m$ in the $m$-objective case). 
Here, we consider a simple distance-based metric, and we will compare different metrics in our experiments (Section~\ref{sec:exp_results}; details in Appendix~\ref{appendix:sec:metric}). 

\paragraph{Single-Point Improvement (SPI).} 
We consider the distance of solutions to a utopian point: 
$g(\bm f(\bm x),\bm z^{*}) = \|\bm f(\bm x)-\bm z^*\|=\sqrt{\sum_{i=1}^m \bigl(f_i(\bm x) - z_i^*\bigr)^2}$, where $m$ denotes the number of objectives, and $\bm z^{*}=(z_1^*, z_2^*, \dots, z_m^*)$ is a utopian point, i.e., $z_i^{*}\leq\min_{\bm x\in\mathcal{X}} f_i(\bm x)$. 
In many real-world cases, the utopian value of an objective can be loosely estimated, for example, by assuming idealised conditions such as zero cost, time or error~\citep{branke2008multiobjective}. 
In our experimental evaluation, we perform a sensitivity analysis on the choice of the utopian point, and it shows that substantially different settings can yield consistent results.

We first consider noiseless cases, i.e., $\bar{\bm y}^t = \bm{f}(\bar{\bm x}^t)$. 
Let $\bar{\mathcal{D}}^n= \{(\bar{\bm x}^t,\bar{\bm y}^t\}_{t=1}^n$ be $n$ observed points and $\bar{X}^n=\{\bar{\bm x}^t\}_{t=1}^n$ be the set of all the observed decision vectors in $\bar{\mathcal{D}}^n$.  
For any point $\bm x \in \mathcal{X}$, we define the single-point improvement (SPI) as: 
\begin{equation}
    I_{SP}(\bm f(\bm x)|g^*, \bm z^{*},\bar{\mathcal{D}}^n)
    = 
    \max\Bigl( 0,\, g^* - \|\bm f(\bm x)-\bm z^*\| \Bigr)
\end{equation}
where $g^*= \min_{\bm x\in \bar{X}^n} g( \bm f(\bm x),\bm z^{*})$. 

\paragraph{Expected Single-Point Improvement (ESPI).} 
We now present ESPI to account for the posterior distribution $p(\bm f|\bar{\mathcal{D}}^n)$. 
Suppose that we independently model each objective
$f_i$ as a Gaussian process based on $\bar{\mathcal{D}}^n$, Then, the posterior of each $f_i$ at a new location $\bm x$ is a Gaussian random variable, i.e., $p(f_i(\bm x)|\bar{\mathcal{D}}^n) \sim \mathcal{N}\!\bigl(\mu_i(\bm x),\sigma_i^2(\bm x)\bigr)$, in which $f_1, \dots, f_m$ are mutually independent Gaussians. Let $\eta_i:=f_i-z_i^{*}$. 
Then we obtain $\eta_i \sim \mathcal{N}\!\bigl(\mu_i(\bm x)-z_i^{*},\sigma_i^2(\bm x)\bigr)$ which is a Gaussian as well. 
The proposed ESPI is defined as:
\begin{equation}\label{euqa:ESPI}
\begin{aligned}
    \aESPI(\bm x)
    &= \mathbb{E}\bigl[I_{SP}(\bm f(\bm x)|g^*, \bm z^{*},\bar{\mathcal{D}}^n)\bigr]\\
    &= 
    \mathbb{E}_{p(\bm \eta)}\bigl[\max( 0,\, g^* - \|\boldsymbol{\eta}\| )] \\
\end{aligned}
\end{equation}

An illustration of ESPI in a bi-objective case is given in Figure~\ref{fig:espi} of Appendix~\ref{appdx:sec:ESPI} for aiding understanding. 
Note that the integral in Eq.~\ref{euqa:ESPI} cannot be solved analytically as it involves the distribution of $\|\boldsymbol{\eta}\|$, whose PDF and CDF have no closed-form expressions and are typically computed via numerical methods~\citep{imhof1961computing,ruben1962probability,das2025new}.

Hence, we use the MC integration with samples from the posterior $\tilde{\bm f}_t(\bm x) \sim p(\bm f (\bm x)|\bar{\mathcal{D}}^n)$ for $t=1,\dots,N$ to estimate Eq.~\ref{euqa:ESPI}:
\begin{equation}
    \aESPI(\bm x) \approx \hataESPI(\bm x) = \frac{1}{N}\sum_{t=1}^N I_{SP}({\tilde{\bm f}_t(\bm x)}|g^*, \bm z^*,\bar{\mathcal{D}}^n)
    \label{eq_MC_ESPI}
\end{equation}


\paragraph{Noisy Expected Single-Point Improvement (NESPI).} 
In the real world, it is not uncommon to encounter an optimisation problem with noises: $\bm y^t = \bm{f}(\bm x^t) + \bm\zeta^t$, where $\bm \zeta^t \sim \mathcal{N}(0, \text{diag} (\bm \sigma_\zeta^2))$. 
In noisy cases, simply using the observed best distance $g^*$ may adversely affect the optimisation performance. 
Here, we present an extension of ESPI, i.e., noisy ESPI (NESPI). 
Let $\mathcal{D}^n= \{(\bm x^t,\bm y^t)\}_{t=1}^n$ be $n$ observed points and $X^n=\{\bm x^t\}_{t=1}^n$ be the set of all the observed decision vectors in $\mathcal{D}^n$. 
By considering the uncertainty in the function values at $X^n$, the proposed NESPI is defined as: 
\begin{equation}\label{equa:NESPI}
    \aNESPI(\bm x) = \int\aESPI(\bm x|\hat{g^*})p(\bm f|\mathcal{D}^n)d\bm f
\end{equation}
where $\hat{g^*}$ denotes the smallest distance to the utopian point over $\bm f(X^n)$. Note that in noiseless cases, NESPI is equivalent to ESPI. Additionally, ESPI and NESPI can be naturally extended to the parallel (batch) setting by using sequential greedy approximation~\citep{balandat2020botorch}.

Like in the noiseless case, the integral in Eq.~\ref{equa:NESPI} is also analytically intractable but can be approximated using the MC integration. 
Let $\tilde{\bm f_t}(\bm x) \sim p(\bm f(\bm x)|\mathcal{D}^n)$ for $t=1,\dots,N$ be samples from the posterior, and let $\hat{g^*}= \min_{\bm x \in X^n}\|\tilde{\bm f_t}(\bm x) - \bm z^*\|$ be the smallest distance to utopian point over the previously evaluated points under the sampled function $\tilde{\bm f_t}(\bm x)$. 
Then, $\aNESPI \approx \frac{1}{N}\sum_{t=1}^N \aESPI(\bm x|\hat{g^*},\bm z^*,\mathcal{D}^n)$. 
Using the MC integration, the inner expectation in $\aNESPI$ can be computed simultaneously using samples from the joint posterior $\tilde{\bm f_t}(\bm x,X^n) \sim p(f(\bm x,X^n)|\mathcal{D}^n)$ over $\bm x$ and $X^n$:
\begin{equation}
    \aNESPI(\bm x) \approx \hataNESPI(\bm x) = \frac{1}{N}\sum_{t=1}^N I_{SP}(\tilde{\bm f_t}|\hat{g^*},\bm z^*,\mathcal{D}^n)
    \label{eq_MC_NESPI}
\end{equation}


\section{Optimising ESPI and NESPI}\label{sec:optimising}

Having presented the MC estimators of ESPI and NESPI, we are now ready to optimise them.

\paragraph{Differentiability.} The MC estimators of ESPI and NESPI ($\hataESPI(\bm x)$ in Eq.~{\ref{eq_MC_ESPI}} and $\hataNESPI(\bm x)$ in Eq.~{\ref{eq_MC_NESPI}}) are differentiable with respect to $\bm x$. 
We are able to automatically compute exact gradients of the MC estimators of ESPI and NESPI ($\nabla_{\bm{x}} \hataESPI(\bm{x})$ and $\nabla_{\bm{x}} \hataNESPI(\bm{x})$) by leveraging the auto-differentiation in modern computational frameworks. 
This facilitates efficient gradient-based optimisation of ESPI and NESPI.

\paragraph{SAA Convergence Results.}
The sample average approximation (SAA) approach~\citep{kleywegt2002sample}, which addresses stochastic optimisation problems by using the MC simulation, has gained increasing popularity and has become a standard technique in BO for optimising MC-based acquisition functions~\citep{balandat2020botorch}. 
By fixing the base samples, the SAA yields a deterministic acquisition function which enables using (quasi-) higher-order optimisation algorithms to obtain fast convergence rates for acquisition optimisation. 
We now give the theoretical convergence guarantees of ESPI under the SAA.

\begin{theorem}
\label{thm:SAA_ESPI}
    Suppose that $\mathcal X$ is compact and $\bm f$ has a multi-output GP prior whose mean and covariance functions are continuously differentiable. 
    Let $\aESPI^* := \max_{\bm x \in \mathcal X} \aESPI(\bm x)$ denote the maximum of ESPI, 
    $S^* := \argmax_{\bm x \in \mathcal X} \aESPI(\bm x)$ denote the set of maximisers of $\aESPI$, 
    $\hataESPI^N(\bm x)$ denote the deterministic function via the base samples $\{\epsilon^t\}_{t=1}^N \sim \mathcal N(0,I_{m})$.
    Suppose $\hat{\bm x}^*_N \in \argmax_{\bm x \in \mathcal X} \hataESPI^N(\bm x)$, then
    \begin{enumerate}[label={(\arabic*)}]
        \item $\hataESPI^N(\hat{\bm x}^*_N) \rightarrow \aESPI^*$ a.s.
        \item $d(\hat{\bm x}_{\!N}^*, S^*) \rightarrow 0$ a.s., where $d(\hat{\bm x}_{\!N}^*,S^*):=\inf_{\bm x\in S^*}\lVert\hat{\bm x}_{\!N}^*-\bm x\rVert$. 
    \end{enumerate}

\end{theorem}

The proof of the theorem is given in Appendix~\ref{appdx:sec:proof1}. 
This theorem indicates that one is able to optimise the MC estimator of the acquisition function ESPI to obtain a solution that converges almost surely to the optimal solution of the original function. 
For the noisy case NESPI, the theorem of the theoretical convergence guarantees under the SAA (Theorem~\ref{thm:SAA_NESPI}), together with its proof, is provided in Appendix~\ref{appdx:sec:proof2}.


\section{Experimental Design}\label{sec:exp_design}

\paragraph{Compared Methods.} 
To evaluate the proposed SPMO, we consider six MOBO methods.
They include one baseline method (Sobol~\citep{sobol1967distribution}), four well-established methods that aim to approximate the entire Pareto front, and one method that aims at the trade-off region of the Pareto front. 
The four methods consist of two scalarisation-based methods, ParEGO~\citep{knowles2006parego} (along with its noisy variant NParEGO~\citep{daulton2021parallel}) and TS-TCH~\citep{paria2020flexible}, and two Pareto-based methods, EHVI~\citep{daulton2020differentiable} (along with its noisy variant NEHVI~\citep{daulton2021parallel}), and JES~\citep{tu2022joint}. 
For the method that does not aim at the entire Pareto front, we consider C-EHVI \citep{gaudrie2018budgeted,gaudrie2020targeting}. C-EHVI prefers the central region of the Pareto front, and we would like to see if it is competitive against our method in identifying a well-balanced solution. 
For all EI-based methods, i.e., ParEGO, NParEGO, EHVI, NEHVI, C-EHVI and SPMO, we use the log version as suggested by~\cite{ament2023unexpected}. 
Note that we only consider NESPI in this work, as it is equivalent to ESPI under noiseless cases.\footnote{Although NEHVI is equivalent to EHVI under noiseless cases, the wall time of NEHVI is much higher than EHVI (see Table~\ref{tbl:wall_time}). Hence we consider EHVI and NEHVI under noiseless and noisy cases, respectively.}

\paragraph{Benchmarks and Real-World Problems.} For benchmark problems, we first choose two most widely scalable functions, DTLZ1 and DTLZ2 \citep{deb2005scalable}. However, their Pareto fronts are rather homogeneous, i.e., with a simplex shape. We then include their inverted versions, i.e., inverted DTLZ1 and inverted DTLZ2 \citep{deb2013evolutionary}. 
These problems do not include the one with convex Pareto fronts nor different objective scales. We thus add convex DTLZ2 and scaled DTLZ2 \citep{jain2013evolutionary}. 
We also give the results of other DTLZ problems which have different features (e.g., degenerate and disconnected Pareto fronts) \citep{deb2005scalable,cheng2017benchmark} in Appendix~\ref{appendix:sec:noiseless}; those problems are widely used in many-objective optimisation \citep{deb2013evolutionary,li2014shift,li2014evolutionary,li2015bi}. Each problem is considered with 3, 5 and 10 objectives, following the practice in~\cite{deb2013evolutionary,jain2013evolutionary,li2014evolutionary}, and tested under both noiseless and noisy cases. 
We also consider two well-studied expensive real-world problems~\citep{tanabe2020easy}, i.e., car side impact design~\citep{jain2013evolutionary} and car cab design~\citep{deb2013evolutionary}. 
The former is a four-objective problem without noise. The latter, which has nine objectives, involves four stochastic variables (out of the total seven variables) that introduce noise into the optimisation process, thus a natural optimisation problem with noise. 
For the other problems without noise, to make their noise cases, we use the additive zero-mean Gaussian noise with a standard deviation of 0.1, as suggested in~\cite{hernandez2016predictive,jiang2025trading}. 
The details of the problem formulations are given in Appendix~\ref{appendix:subsec:problems}. 

\paragraph{Performance Metrics.} 
The proposed method aims to find a single trade-off solution between objectives which has a high chance of being favoured by the decision-maker, and therefore requires the use of appropriate metrics to fairly evaluate our method~\citep{li2019quality,li2022evaluate}. 
As such, we first consider two single-point-based metrics, i.e., the distance-based metric used in the proposed method (reported as log-distance for better visualisation) and the HV-based metric~\citep{Zitzler1999}, which measures the HV contribution of a solution. 
We expect that our method performs well on the distance-based metric as we directly optimise it. Notably, we are not certain whether our method performs best on the single-point HV since we do not directly optimise it. On the other hand, since most of the compared methods aim to achieve a good approximation of the entire Pareto front, we also consider the HV of the whole (nondominated) solution set obtained. 
It is expected that our method performs poorly, compared to other methods since we only optimise one point, rather than maximising the HV of the whole set (see Figure~\ref{fig:motivation}). 
For the reference point of the two HV metrics, we followed the practice in~\cite{ishibuchi2018specify,balandat2020botorch,chugh2020scalarizing,daulton2020differentiable} (see detailed settings in Appendix~\ref{appdx:sec:method}). 



\paragraph{Budget and Statistical Validation.} For all the methods, we allow a maximum of 200 evaluations, following the practice in~\cite{daulton2020differentiable,daulton2021parallel,konakovic2020diversity}. 
To enable statistical comparisons, each optimisation was repeated 30 times. 
We use the Wilcoxon rank-sum test~\citep{wilcoxon1992individual} at a significance level of $\alpha=0.05$ and Holm-Bonferroni correction~\citep{holm1979simple} to see if our method differs significantly from each peer method.



\section{Experimental Results}\label{sec:exp_results}

We first report the results under noiseless cases, then under noisy cases and under batch settings.
Next, we perform the sensitivity analysis of the utopian point used.
Then, we compare the proposed framework working with different single-point metrics, e.g., the distance-based, weighted sum-based and Tchebycheff-based. 
Lastly, we give the acquisition optimisation wall time for all the algorithms.

    

    

\paragraph{Noiseless Cases.} We begin our evaluation by considering the distance metric, for which we expect good performance obtained by our SPMO. Table~\ref{tbl:Dist_M5} shows the results of SPMO and the other methods on the seven benchmark (with five objectives) and real-world problems. Unsurprisingly, as can be seen from the table, our method significantly outperforms the other methods on all the problems. 
In addition, to understand the anytime performance, Figure~\ref{fig:convergence_distance_M5} presents the trajectories of the distance-based metric obtained by the six methods. 
As seen, SPMO demonstrates a clear advantage over the other methods, obtaining a better convergence rate from the very beginning.

We now compare the methods by the HV metric of a single solution. Table~\ref{tbl:HV_SP_M5} shows the results of the best solution (in terms of its HV value) obtained by SPMO and the five peer methods. 
As can be seen, SPMO is still very competitive, performing best on all the problems except DTLZ2 where EHVI obtains the best HV. 
To get a sense of what such a best-HV solution looks like, we use a spider chart to plot the solution on the five-objective inverted DTLZ1 problem in Figure~\ref{fig:spider_inverted_DTLZ1}.   
As seen, the solution of SPMO has the largest area, and it actually Pareto dominates the solutions of the other methods (i.e., better or at least equal on all five objectives).

\begin{table*}[!ht]
    \centering
    \caption{Results of the distance-based metric (log distance) obtained by the SPMO and the six peer methods on the benchmark problems with 5 objectives and the car side impact design problem on 30 runs. The method with the best mean is highlighted in bold. The symbols ``$+$'', ``$\sim$'' and ``$-$'' indicate that the method is statistically worse than, equivalent to and better than our SPMO, respectively.}
    \resizebox{\textwidth}{!}{%
    \begin{tabular}{llllllllc}
    \toprule
    \bfseries Method
    & \multicolumn{1}{c}{\bfseries DTLZ1} 
    & \multicolumn{1}{c}{\bfseries DTLZ2} 
    & \multicolumn{1}{c}{\bfseries Inverted DTLZ1} 
    & \multicolumn{1}{c}{\bfseries Inverted DTLZ2} 
    & \multicolumn{1}{c}{\bfseries Convex DTLZ2} 
    & \multicolumn{1}{c}{\bfseries Scaled DTLZ2} 
    & \multicolumn{1}{c}{\bfseries Car side impact} 
    & {\bfseries Sum up} \\ 
    & \multicolumn{1}{c}{Mean (Std)}
    & \multicolumn{1}{c}{Mean (Std)}
    & \multicolumn{1}{c}{Mean (Std)}
    & \multicolumn{1}{c}{Mean (Std)}
    & \multicolumn{1}{c}{Mean (Std)}
    & \multicolumn{1}{c}{Mean (Std)}
    & \multicolumn{1}{c}{Mean (Std)}
    & \multicolumn{1}{c}{$+$/$\sim$/$-$}  \\ \midrule

    \bfseries Sobol & 3.7e+0 (3.0e--1)$^+$ & 2.4e--1 (4.6e--2)$^+$ & 4.8e+0 (3.1e--1)$^+$ & 6.0e--1 (5.7e--2)$^+$ & -3.1e--1 (2.5e--1)$^+$ & 2.3e--1 (5.1e--2)$^+$ & -1.9e--1 (2.5e--2)$^+$ & \bfseries 7/ 0/ 0\\
    \bfseries ParEGO & 3.4e+0 (1.9e--1)$^+$ & 6.7e--2 (8.1e--2)$^+$ & 3.2e+0 (5.4e--1)$^+$ & 2.5e--1 (1.3e--2)$^+$ & -1.7e+0 (2.0e--1)$^+$ & 1.3e--1 (8.9e--2)$^+$ & -3.3e--1 (6.0e--3)$^+$ & \bfseries 7/ 0/ 0\\
    \bfseries TS-TCH & 3.9e+0 (2.2e--1)$^+$ & 2.0e--1 (4.1e--2)$^+$ & 4.8e+0 (3.2e--1)$^+$ & 4.6e--1 (1.2e--2)$^+$ & -7.1e--1 (2.1e--1)$^+$ & 2.6e--1 (3.8e--2)$^+$ & -2.7e--1 (1.5e--2)$^+$ & \bfseries 7/ 0/ 0\\
    \bfseries EHVI & 3.5e+0 (9.6e--2)$^+$ & 9.3e--3 (3.5e--3)$^+$ & 4.0e+0 (4.4e--1)$^+$ & 2.3e--1 (5.5e--3)$^+$ & -1.4e+0 (2.1e--1)$^+$ & 3.1e--1 (6.0e--2)$^+$ & -3.3e--1 (1.0e--2)$^+$ & \bfseries 7/ 0/ 0\\
    \bfseries C-EHVI & 3.6e+0 (1.4e--1)$^+$ & 3.3e--3 (3.8e--3)$^+$ & 3.9e+0 (4.6e--1)$^+$ & 2.5e--1 (1.8e--2)$^+$ & -3.5e--1 (2.9e--1)$^+$ & 7.6e--2 (8.0e--2)$^+$ & -3.3e--1 (1.0e--2)$^+$ & \bfseries 7/ 0/ 0\\
    \bfseries JES & 3.4e+0 (1.3e--1)$^+$ & 1.1e--1 (8.0e--2)$^+$ & 4.5e+0 (1.7e--1)$^+$ & 2.6e--1 (2.0e--2)$^+$ & -1.0e+0 (4.5e--1)$^+$ & 8.8e--2 (9.3e--2)$^+$ & -3.3e--1 (8.3e--3)$^+$ & \bfseries 7/ 0/ 0\\
    \bfseries SPMO & \best {3.1e+0} (\best {3.0e--1})  & \best {9.0e--4} (\best {8.3e--4})  & \best {2.9e+0} (\best {4.9e--1})  & \best {2.1e--1} (\best {5.0e--5})  & \best {-2.1e+0} (\best {7.7e--3})  & \best {1.7e--4} (\best {1.2e--4})  & \best {-3.4e--1} (\best {2.3e--6})  & \\
    \bottomrule
    \end{tabular}
    }
    \label{tbl:Dist_M5}
\end{table*}

\begin{figure*}[!ht]
    \centering
    \begin{subfigure}[b]{\textwidth}
    \begin{minipage}{\textwidth}
        \centering
        \includegraphics[width=\textwidth]{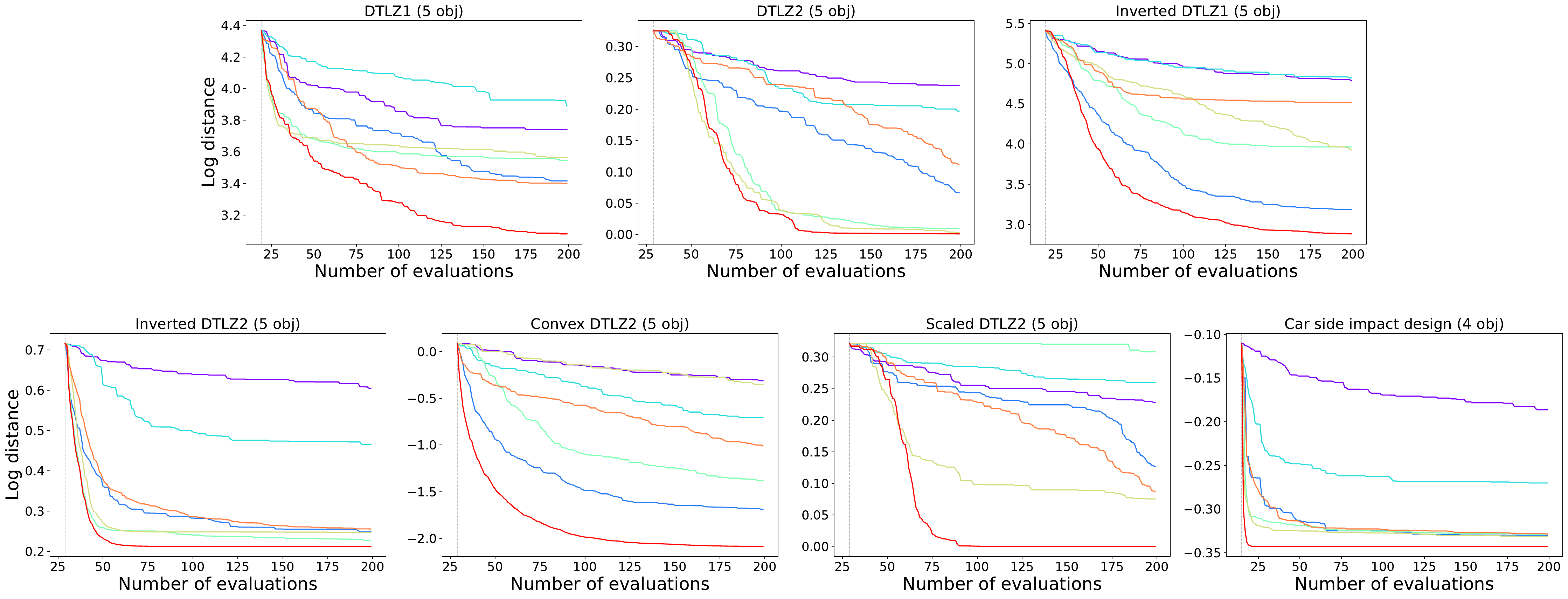}
    \end{minipage}
    \end{subfigure}

    \begin{minipage}{\textwidth}
        \centering
        \includegraphics[width=0.6\textwidth]{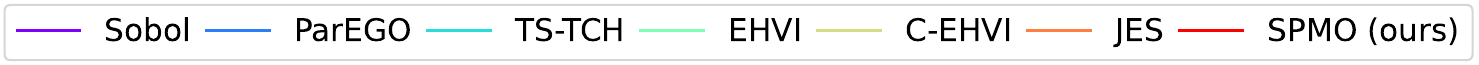}
    \end{minipage}

    \caption{Trajectories of the distance-based metric (log distance) obtained by the seven methods on the benchmark problems with 5 objectives and the car side impact design problem. Each coloured line represents the mean metric value on 30 independent runs (after the initial Sobol samples, represented by the dashed grey line). 
    }
    \label{fig:convergence_distance_M5}
\end{figure*}

Lastly, we consider the HV results of all the solutions obtained by the compared methods (Table~\ref{tbl:HV_M5}).
Interestingly, although SPMO does not aim to approximate the entire Pareto front, it still gets fairly good results, outperforming the other methods on at least 3 out of the 7 problems. One explanation for this is that within very tight budget, searching for improving convergence of solutions may play a bigger part than searching for improving diversity, thus contributing more to the HV value.

For simplicity, we here only show results of the benchmark problems with five objectives. 
Results on 3- and 10-objective problems can be found in Appendix~\ref{appendix:sec:noiseless}. 
A general pattern is that as the number of objectives increases, the advantages of SPMO become more pronounced. 
On the 3-objective problems, the differences between SPMO and the peer methods in terms of the two single-point metrics are relatively small, whereas on the 10-objective problems, the gaps in both metrics become substantially larger (see Figures~\ref{fig:violin_dist_M3} and~\ref{fig:violin_single_HV_M3} in the Appendix).
Regarding the HV of all evaluated solutions, SPMO statistically outperforms the peer methods on at least 2 and 5 out of the 6 problems on the 3-objective and 10-objective cases, respectively. 

\paragraph{Noisy Cases.}

We compare the proposed SPMO with the peer methods on the seven noisy benchmark and real-world problems. 
The results (the distance-based metric, two HV metrics, and convergence trajectories) are given in Appendix~\ref{appendix:sec:noisy}.
Like in the noiseless setting, SPMO significantly outperforms the peer methods on all the problems with respect to the single-point metrics (distance and HV). 
Regarding the HV of all the (nondominated) solutions obtained, SPMO achieves the best performance on the majority of the problems.
Figure~\ref{fig:RE91} shows the spider chart of the best solution (with respect to its HV) obtained by each algorithm in a typical run on the 9-objective car cab design problem. The violin plot of their HV values in 30 independent runs is given in the left panel for reference. As can be seen, the solution of SPMO is the best or close to the best on most of the objectives (except the first objective), thus having a clearly larger area.   


\begin{table*}[!ht]
    \centering
    \caption{The HV of the best solution (in terms of its HV value) obtained by SPMO and the peer methods on the benchmark problems with 5 objectives and the car side impact design problem on 30 runs. 
    The method with the best mean is highlighted in bold. The symbols ``$+$'', ``$\sim$'' and ``$-$'' indicate that the method is statistically worse than, equivalent to and better than our SPMO, respectively.
    } 
    \resizebox{\textwidth}{!}{%
    \begin{tabular}{llllllllc}
    \toprule
    \bfseries Method
    & \multicolumn{1}{c}{\bfseries DTLZ1} 
    & \multicolumn{1}{c}{\bfseries DTLZ2} 
    & \multicolumn{1}{c}{\bfseries Inverted DTLZ1} 
    & \multicolumn{1}{c}{\bfseries Inverted DTLZ2} 
    & \multicolumn{1}{c}{\bfseries Convex DTLZ2} 
    & \multicolumn{1}{c}{\bfseries Scaled DTLZ2} 
    & \multicolumn{1}{c}{\bfseries Car side impact} 
    & {\bfseries Sum up} \\ 
    & \multicolumn{1}{c}{Mean (Std)}
    & \multicolumn{1}{c}{Mean (Std)}
    & \multicolumn{1}{c}{Mean (Std)}
    & \multicolumn{1}{c}{Mean (Std)}
    & \multicolumn{1}{c}{Mean (Std)}
    & \multicolumn{1}{c}{Mean (Std)}
    & \multicolumn{1}{c}{Mean (Std)}
    & \multicolumn{1}{c}{$+$/$\sim$/$-$}  \\ \midrule

    \bfseries Sobol & 8.7e+12 (3.7e+11)$^+$ & 3.5e--2 (1.5e--2)$^+$ & 5.1e+12 (1.1e+12)$^+$ & 8.7e--4 (1.5e--3)$^+$ & 3.8e--1 (1.7e--1)$^+$ & 3.8e--2 (2.1e--2)$^+$ & 2.6e--1 (1.8e--2)$^+$ & \bfseries 7/ 0/ 0\\
    \bfseries ParEGO & 9.1e+12 (1.8e+11)$^+$ & 1.3e--1 (5.6e--2)$^+$ & 8.8e+12 (7.0e+11)$^+$ & 4.0e--2 (2.9e--3)$^+$ & 1.1e+0 (6.2e--2)$^+$ & 8.7e--2 (5.3e--2)$^+$ & 3.6e--1 (2.6e--3)$^+$ & \bfseries 7/ 0/ 0\\
    \bfseries TS-TCH & 8.4e+12 (3.8e+11)$^+$ & 3.2e--2 (1.3e--2)$^+$ & 4.9e+12 (1.2e+12)$^+$ & 7.1e--3 (1.1e--3)$^+$ & 6.5e--1 (1.4e--1)$^+$ & 2.8e--2 (1.5e--2)$^+$ & 3.1e--1 (9.3e--3)$^+$ & \bfseries 7/ 0/ 0\\
    \bfseries EHVI & 9.0e+12 (9.5e+10)$^+$ & \best {1.9e--1} (\best {7.2e--3})$^\sim$ & 7.5e+12 (9.7e+11)$^+$ & 4.5e--2 (1.4e--3)$^+$ & 1.0e+0 (9.7e--2)$^+$ & 1.5e--2 (1.4e--2)$^+$ & 3.6e--1 (6.8e--3)$^+$ & \bfseries 6/ 1/ 0\\
    \bfseries C-EHVI & 9.0e+12 (1.0e+11)$^+$ & 1.6e--1 (1.9e--2)$^+$ & 7.7e+12 (7.8e+11)$^+$ & 4.1e--2 (4.1e--3)$^+$ & 4.3e--1 (2.1e--1)$^+$ & 1.2e--1 (5.1e--2)$^+$ & 3.6e--1 (4.4e--3)$^\sim$ & \bfseries 6/ 1/ 0\\
    \bfseries JES & 9.0e+12 (1.2e+11)$^+$ & 1.0e--1 (3.9e--2)$^+$ & 6.2e+12 (5.0e+11)$^+$ & 3.9e--2 (4.5e--3)$^+$ & 8.5e--1 (2.4e--1)$^+$ & 1.1e--1 (5.5e--2)$^+$ & 3.6e--1 (5.1e--3)$^+$ & \bfseries 7/ 0/ 0\\
    \bfseries SPMO & \best {9.3e+12} (\best {2.8e+11})  & 1.9e--1 (1.4e--2)  & \best {9.2e+12} (\best {4.8e+11})  & \best {4.9e--2} (\best {1.2e--5})  & \best {1.3e+0} (\best {5.0e--3})  & \best {1.6e--1} (\best {1.6e--2})  & \best {3.6e--1} (\best {3.2e--3})  & \\
    \bottomrule
    \end{tabular}
    }
    \label{tbl:HV_SP_M5}
\end{table*}

\begin{table*}[!ht]
    \centering
    \caption{The HV of all the solutions obtained by the seven methods on the seven benchmark (with five objectives) and real-world problems on 30 independent runs. 
    The method with the best mean HV is highlighted in bold. The symbols ``$+$'', ``$\sim$'', and ``$-$'' indicate that a method is statistically worse than, equivalent to, and better than SPMO, respectively.}
    \resizebox{\textwidth}{!}{%
    \begin{tabular}{llllllllc}
    \toprule
    \bfseries Method
    & \multicolumn{1}{c}{\bfseries DTLZ1} 
    & \multicolumn{1}{c}{\bfseries DTLZ2} 
    & \multicolumn{1}{c}{\bfseries Inverted DTLZ1} 
    & \multicolumn{1}{c}{\bfseries Inverted DTLZ2} 
    & \multicolumn{1}{c}{\bfseries Convex DTLZ2} 
    & \multicolumn{1}{c}{\bfseries Scaled DTLZ2} 
    & \multicolumn{1}{c}{\bfseries Car side impact} 
    & {\bfseries Sum up} \\ 
    & \multicolumn{1}{c }{Mean (Std)}
    & \multicolumn{1}{c }{Mean (Std)}
    & \multicolumn{1}{c }{Mean (Std)}
    & \multicolumn{1}{c }{Mean (Std)}
    & \multicolumn{1}{c }{Mean (Std)}
    & \multicolumn{1}{c }{Mean (Std)}
    & \multicolumn{1}{c }{Mean (Std)}
    & \multicolumn{1}{c}{$+$/$\sim$/$-$}  \\ \midrule

    \bfseries Sobol & 1.0e+13 (3.1e+10)$^+$ & 7.9e--2 (2.5e--2)$^+$ & 5.6e+12 (8.7e+11)$^+$ & 9.3e--4 (1.5e--3)$^+$ & 5.8e--1 (2.1e--1)$^+$ & 8.6e--2 (3.1e--2)$^+$ & 5.2e--1 (1.0e--2)$^+$ & \bfseries 7/ 0/ 0\\
    \bfseries ParEGO & 1.0e+13 (3.8e+10)$^-$ & 3.6e--1 (2.1e--1)$^+$ & 9.1e+12 (6.2e+11)$^+$ & 1.4e--1 (1.0e--2)$^-$ & \best {1.5e+0} (\best {2.8e--2})$^-$ & 2.2e--1 (1.6e--1)$^+$ & 7.4e--1 (1.4e--2)$^-$ & \bfseries 3/ 0/ 4\\
    \bfseries TS-TCH & 1.0e+13 (2.6e+10)$^+$ & 5.9e--2 (2.7e--2)$^+$ & 5.4e+12 (1.0e+12)$^+$ & 2.0e--2 (3.6e--3)$^+$ & 1.0e+0 (1.5e--1)$^+$ & 5.1e--2 (2.7e--2)$^+$ & 6.5e--1 (1.1e--2)$^-$ & \bfseries 6/ 0/ 1\\
    \bfseries EHVI & \best {1.0e+13} (\best {3.3e+8})$^-$ & \best {8.3e--1} (\best {6.2e--2})$^-$ & 8.2e+12 (6.6e+11)$^+$ & \best {2.1e--1} (\best {1.9e--3})$^-$ & 1.5e+0 (6.9e--2)$^+$ & 2.0e--2 (1.7e--2)$^+$ & 7.4e--1 (9.6e--3)$^-$ & \bfseries 3/ 0/ 4\\
    
    \bfseries C-EHVI & 1.0e+13 (1.6e+11)$^\sim$ & 3.6e--1 (8.5e--2)$^+$ & 8.0e+12 (6.9e+11)$^+$ & {7.2e--2} ({9.2e--3})$^-$ & 5.8e--1 (2.7e--1)$^+$ & 2.1e--1 (1.1e--1)$^+$ & {5.9e--1} ( {3.0e--2})$^-$ & \bfseries 4/ 1/ 2\\
    
    \bfseries JES & 1.0e+13 (9.0e+9)$^-$ & 2.8e--1 (1.5e--1)$^+$ & 8.6e+12 (3.5e+11)$^+$ & 1.4e--1 (9.4e--3)$^-$ & 1.3e+0 (2.6e--1)$^+$ & \best {3.1e--1} (\best {2.0e--1})$^\sim$ & \best {7.4e--1} (\best {1.3e--2})$^-$ & \bfseries 3/ 1/ 3\\
    \bfseries SPMO & 1.0e+13 (8.1e+10)  & 5.0e--1 (8.2e--2)  & \best {9.6e+12} (\best {2.9e+11})  & 6.4e--2 (2.8e--3)  & 1.5e+0 (3.2e--2)  & 3.0e--1 (1.1e--1)  & 5.3e--1 (3.1e--2)  & \\
    \bottomrule
    \end{tabular}
    }
    \label{tbl:HV_M5}
\end{table*}


\begin{figure}[!ht]
    \centering
    \begin{minipage}{0.4\linewidth}{
    \includegraphics[width=0.9\linewidth]{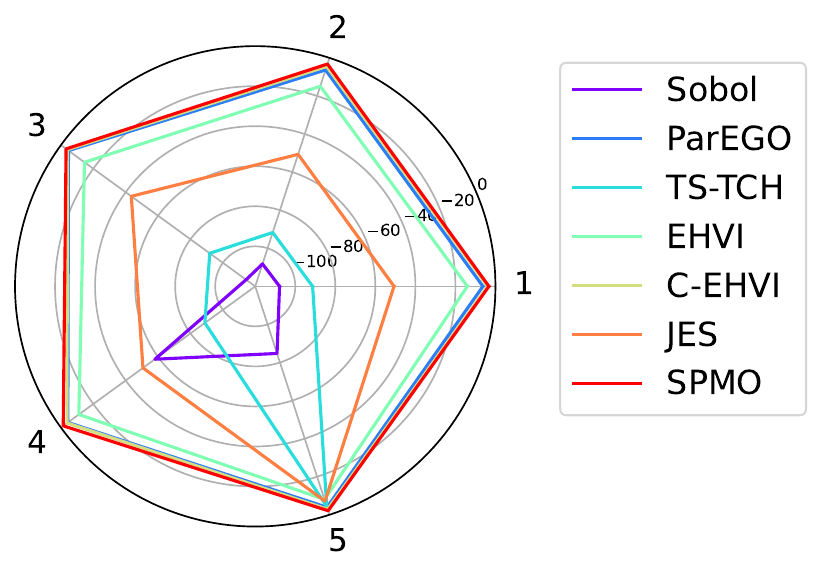}	
    }\end{minipage}
    \begin{minipage}{0.58\linewidth}{
    \caption{Spider chart of the best solution (in terms of its HV) obtained by the seven methods on the inverted DTLZ1 problem with 5 objectives in a typical run. 
    Each axis in the spider chart represents one objective. Here, the objective values are multiplied by $-1$ in this minimisation problem, such that a solution with a larger area indicates better quality. 
    }
    \label{fig:spider_inverted_DTLZ1}
    }\end{minipage}
\end{figure}

\paragraph{Batch Setting.} 
Previously, we considered the case in the sequential setting, where solutions are evaluated sequentially. Now we want to see if the proposed method works in the batch setting. Here, the batch size $q$ is set to $5$, a commonly used value~\citep{lin2022pareto}.
As can be seen in 
Tables \ref{tbl:Dist_M5_batch}--\ref{tbl:HV_M5_batch} and Figures~\ref{fig:violin_dist_M5_batch}--\ref{fig:violin_all_HV_M5_batch} (Appendix~\ref{appendix:sec:parallel}), similar to the results in the sequential setting, SPMO generally performs best. On the single-point metrics, it achieves the smallest distance on all the problems and the highest HV on 5 out of the 6 problems (except on DTLZ2).
As for the HV of all evaluated solutions, SPMO obtains the best HV on two problems and takes the second or third places on the remaining ones. This indicates that prioritising convergence is also very useful for many-objective problems in batch settings.


    

    

\paragraph{Sensitivity Analysis.}  
A parameter needed in the proposed method is the utopian point. In our experiment, we set it to be the vector consisting of the best value on each objective (i.e., the problem's ideal point). However, in real life, the ideal point is usually unknown before the optimisation. Hence, we would like to investigate how much different utopian points affect the performance. In this context, we consider three different settings. The first one is slightly better than the ideal point, i.e., with a difference of $0.01$, the second is fairly better than the ideal point (i.e. $0.1$),  and the last one is significantly better than the ideal point (i.e. $1.0$).
The results are given in Appendix~\ref{appendix:sec:sensitivity}. As can be seen, interestingly, SPMO with the three lower utopian points with different levels performs better than or at least equivalently to SPMO with our current setting. This indicates that 1) using the ideal point in the proposed method may not be the best choice (though it performs better than the other MOBO methods), and 2) SPMO's performance is robust to the choice of the utopian point and a liberal estimate is sufficient to achieve good results, e.g., considering zero cost in real-world cases.


\begin{figure*}[!ht]
    \centering
    \begin{minipage}[b]{0.4\textwidth}
        \centering
        \includegraphics[width=0.7\textwidth]{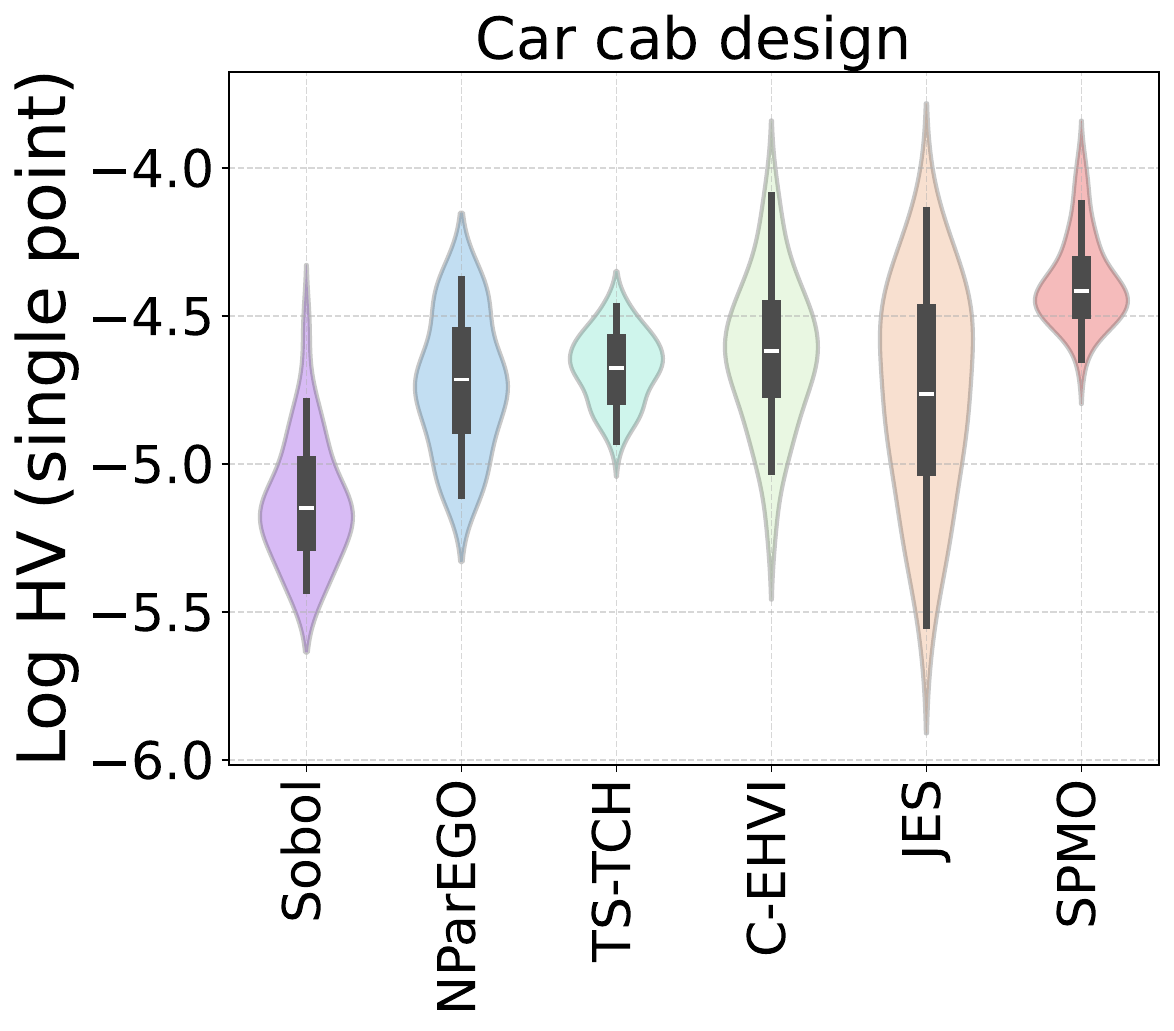}
    \end{minipage}
    \hfill
    \begin{minipage}[b]{0.55\textwidth}
        \centering
        \includegraphics[width=0.7\textwidth]{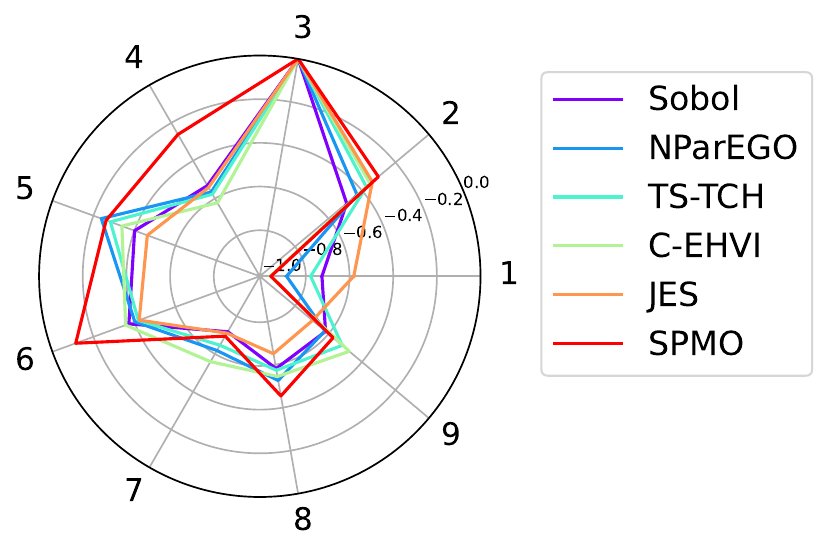}
    \end{minipage}
    \caption{\textit{Left:} Violin plots of the HV values of the best solution (with respect to its HV) obtained by all the methods on the car cab design problem in 30 independent runs. \textit{Right:} The objective values (normalised and multiplied by $-1$) of the best solution (with the highest HV) obtained by each method on the cab design problem in a typical run. 
    }
    \label{fig:RE91}
\end{figure*}

\paragraph{Comparison of Single-Point Metrics within SPMO.} 
In the proposed SPMO framework, we employ a distance metric (i.e., the distance of a solution to the utopian point), denoted as SPMO$_{dist}$. 
However, different metrics can be adopted provided that they can reflect the quality of a solution in achieving a good trade-off between objectives. 
We now consider two other well-known metrics, weighted sum and Tchebycheff scalarisation (with the same weights $(\frac{1}{m},\dots,\frac{1}{m})$), denoted by SPMO$_{ws}$ and SPMO$_{Tch}$, respectively. 
We compare these three versions of SPMO. The results (given in Appendix~\ref{appendix:sec:metric}) show that SPMO$_{dist}$ performs in general better than SPMO$_{ws}$ and SPMO$_{Tch}$. It obtains the best result on at least 4 out of the 6 problems on the HV of the best solution.
As for the HV of all the solutions, SPMO$_{dist}$ performs best on DTLZ1 and its variants, but worse than SPMO$_{Tch}$ on DTLZ2 and its variants (except convex DTLZ2). A possible explanation is that SPMO$_{Tch}$ has slower convergence and can be better in exploring different solutions, thus better on relatively easy-to-converge problems.

\paragraph{Acquisition Optimisation Wall Time.}\label{sec:walltime}


Lastly, we present the wall time for optimising the acquisition function (i.e., determining a solution to be evaluated). 
The results (Appendix~\ref{appendix:sec:time}) show that our method is among the fastest algorithms. 
When the number of objectives is 3 or 5, the time of all the methods is acceptable with a maximum of 98 seconds. 
As the number of objectives increases to 10, hypervolume-based methods (i.e., EHVI and NEHVI) become very expensive (taking about half an hour and more than 3 hours, respectively).\footnote{Wall time is measured based on the initial samples. Due to the exponentially increasing computational complexity with objectives, EHVI and NEHVI are fully evaluated only for problems with objectives $m\leq 5$.} 
The proposed SPMO method shows high computational efficiency, achieving the lowest time requirement in four out of the six instances. 


\section{Conclusion}\label{sec:con}

This work presented a multi-objective BO framework that aims to find a single trade-off solution of the highest possible quality with respect to multiple objectives, rather than seeking to explore their entire Pareto front. We theoretically proved the convergence guarantees under the SAA and empirically verified the proposed framework through extensive experiments, including on noiseless/noisy and sequential/batch cases, by sensitivity analysis, with different metrics for the acquisition function, and on a range of benchmark and real-world problems. 
A noticeable limitation of the proposed framework is that it focuses on finding a single trade-off point, thus failing to capture the information about the entire Pareto front; it thus may be less useful for certain applications where such information is valuable (e.g., the Pareto front's ranges and nadir points). A detailed discussion of its applicability is provided in Appendix~\ref{appendix:sec:discussion}. However, notably, the proposed framework showed its competitiveness against existing state-of-the-arts with respect to even the quality of the whole solution set (through HV of all solutions, see Table~\ref{tbl:HV_M5}). 
Future work includes studying and enhancing the scalability of the proposed methods (i.e., in higher-dimensional search space) and extending their applicability to other scenarios, e.g., multi-fidelity optimisation (see Appendix~\ref{appendix:sec:extensions} for more details).

\bibliographystyle{unsrtnat}
\bibliography{ref}  

@inproceedings{ishibuchi2008evolutionary,
  title={Evolutionary many-objective optimization: A short review},
  author={Ishibuchi, Hisao and Tsukamoto, Noritaka and Nojima, Yusuke},
  booktitle={2008 IEEE congress on evolutionary computation (IEEE world congress on computational intelligence)},
  pages={2419--2426},
  year={2008},
  organization={IEEE}
}

@article{li2022evaluate,
  title={How to evaluate solutions in pareto-based search-based software engineering: A critical review and methodological guidance},
  author={Li, Miqing and Chen, Tao and Yao, Xin},
  journal={IEEE Transactions on Software Engineering},
  volume={48},
  number={5},
  pages={1771--1799},
  year={2022},
  publisher={IEEE}
}

@article{li2019quality,
  title={Quality evaluation of solution sets in multiobjective optimisation: A survey},
  author={Li, Miqing and Yao, Xin},
  journal={ACM Computing Surveys (CSUR)},
  volume={52},
  number={2},
  pages={1--38},
  year={2019},
  publisher={ACM New York, NY, USA}
}

@article{li2014shift,
  title={Shift-based density estimation for Pareto-based algorithms in many-objective optimization},
  author={Li, Miqing and Yang, Shengxiang and Liu, Xiaohui},
  journal={IEEE Transactions on Evolutionary Computation},
  volume={18},
  number={3},
  pages={348--365},
  year={2014},
  publisher={IEEE}
}

@article{matrosov2015many,
  title={Many-objective optimization and visual analytics reveal key trade-offs for London’s water supply},
  author={Matrosov, Evgenii S and Huskova, Ivana and Kasprzyk, Joseph R and Harou, Julien J and Lambert, Chris and Reed, Patrick M},
  journal={Journal of Hydrology},
  volume={531},
  pages={1040--1053},
  year={2015},
  publisher={Elsevier}
}

@article{li2015bi,
  title={Bi-goal evolution for many-objective optimization problems},
  author={Li, Miqing and Yang, Shengxiang and Liu, Xiaohui},
  journal={Artificial Intelligence},
  volume={228},
  pages={45--65},
  year={2015},
  publisher={Elsevier}
}

@article{hierons2020many,
  title={Many-objective test suite generation for software product lines},
  author={Hierons, Robert M and Li, Miqing and Liu, Xiaohui and Parejo, Jose Antonio and Segura, Sergio and Yao, Xin},
  journal={ACM Transactions on Software Engineering and Methodology (TOSEM)},
  volume={29},
  number={1},
  pages={1--46},
  year={2020},
  publisher={ACM New York, NY, USA}
}

@article{park2023botied,
  title={BOtied: Multi-objective Bayesian optimization with tied multivariate ranks},
  author={Park, Ji Won and Tagasovska, Nata{\v{s}}a and Maser, Michael and Ra, Stephen and Cho, Kyunghyun},
  journal={arXiv preprint arXiv:2306.00344},
  year={2023}
}

@article{bhatija2025multi,
  title={Multi-Objective Causal Bayesian Optimization},
  author={Bhatija, Shriya and Zuercher, Paul-David and Thumm, Jakob and Bohn{\'e}, Thomas},
  journal={arXiv preprint arXiv:2502.14755},
  year={2025}
}

@article{ngo2025mobo,
  title={MOBO-OSD: Batch Multi-Objective {B}ayesian Optimization via Orthogonal Search Directions},
  author={Ngo, Lam and Ha, Huong and Chan, Jeffrey and Zhang, Hongyu},
  journal={arXiv preprint arXiv:2510.20872},
  year={2025}
}

@article{li2024constrained,
  title={Constrained Multi-objective {B}ayesian Optimization through Optimistic Constraints Estimation},
  author={Li, Diantong and Zhang, Fengxue and Liu, Chong and Chen, Yuxin},
  journal={arXiv preprint arXiv:2411.03641},
  year={2024}
}

@article{hoang2025high,
  title={High Dimensional {B}ayesian Optimization using Lasso Variable Selection},
  author={Hoang, Vu Viet and Tran, Hung The and Gupta, Sunil and Nguyen, Vu},
  journal={arXiv preprint arXiv:2504.01743},
  year={2025}
}

@InProceedings{renganathan2025qpots,
  title = 	 {$q\texttt{POTS}$: Efficient Batch Multiobjective {B}ayesian Optimization via Pareto Optimal {T}hompson Sampling},
  author =       {Renganathan, Ashwin and Carlson, Kade},
  booktitle = 	 {Proceedings of The 28th International Conference on Artificial Intelligence and Statistics},
  pages = 	 {4051--4059},
  year = 	 {2025},
  volume = 	 {258},
  series = 	 {Proceedings of Machine Learning Research},
  month = 	 {03--05 May},
  publisher =    {PMLR},
}

@article{das2025new,
  title={New methods to compute the generalized chi-square distribution},
  author={Das, Abhranil},
  journal={Journal of Statistical Computation and Simulation},
  pages={1--35},
  year={2025},
  publisher={Taylor \& Francis}
}

@article{imhof1961computing,
  title={Computing the distribution of quadratic forms in normal variables},
  author={Imhof, Jean-Pierre},
  journal={Biometrika},
  volume={48},
  number={3/4},
  pages={419--426},
  year={1961},
  publisher={JSTOR}
}

@article{ruben1962probability,
  title={Probability content of regions under spherical normal distributions, IV: The distribution of homogeneous and non-homogeneous quadratic functions of normal variables},
  author={Ruben, Harold},
  journal={The Annals of Mathematical Statistics},
  volume={33},
  number={2},
  pages={542--570},
  year={1962},
  publisher={JSTOR}
}

@article{yang2019efficient,
  title={Efficient computation of expected hypervolume improvement using box decomposition algorithms},
  author={Yang, Kaifeng and Emmerich, Michael and Deutz, Andr{\'e} and B{\"a}ck, Thomas},
  journal={Journal of Global Optimization},
  volume={75},
  number={1},
  pages={3--34},
  year={2019},
  publisher={Springer}
}

@inproceedings{zuluaga2013active,
  title={Active learning for multi-objective optimization},
  author={Zuluaga, Marcela and Sergent, Guillaume and Krause, Andreas and P{\"u}schel, Markus},
  booktitle={International conference on machine learning},
  pages={462--470},
  year={2013},
  organization={PMLR}
}

@phdthesis{parr2013improvement,
  title={Improvement criteria for constraint handling and multiobjective optimization},
  author={Parr, James},
  year={2013},
  school={University of Southampton}
}

@article{namura2017expected,
  title={Expected improvement of penalty-based boundary intersection for expensive multiobjective optimization},
  author={Namura, Nobuo and Shimoyama, Koji and Obayashi, Shigeru},
  journal={IEEE Transactions on Evolutionary Computation},
  volume={21},
  number={6},
  pages={898--913},
  year={2017},
  publisher={IEEE}
}

@inproceedings{jeong2005efficient,
  title={Efficient global optimization (EGO) for multi-objective problem and data mining},
  author={Jeong, Shinkyu and Obayashi, Shigeru},
  booktitle={2005 IEEE congress on evolutionary computation},
  volume={3},
  pages={2138--2145},
  year={2005},
  organization={IEEE}
}

@article{zhan2017expected,
  title={Expected improvement matrix-based infill criteria for expensive multiobjective optimization},
  author={Zhan, Dawei and Cheng, Yuansheng and Liu, Jun},
  journal={IEEE Transactions on Evolutionary Computation},
  volume={21},
  number={6},
  pages={956--975},
  year={2017},
  publisher={IEEE}
}

@article{svenson2016multiobjective,
  title={Multiobjective optimization of expensive-to-evaluate deterministic computer simulator models},
  author={Svenson, Joshua and Santner, Thomas},
  journal={Computational Statistics \& Data Analysis},
  volume={94},
  pages={250--264},
  year={2016},
  publisher={Elsevier}
}

@book{bautista2009sequential,
  title={A sequential design for approximating the {P}areto front using the expected {P}areto improvement function},
  author={Bautista, Dianne Carrol},
  year={2009},
  publisher={The Ohio State University}
}

@phdthesis{svenson2011computer,
  title={Computer experiments: Multiobjective optimization and sensitivity analysis},
  author={Svenson, Joshua},
  year={2011},
  school={The Ohio State University}
}

@article{keane2006statistical,
  title={Statistical improvement criteria for use in multiobjective design optimization},
  author={Keane, Andy J},
  journal={AIAA journal},
  volume={44},
  number={4},
  pages={879--891},
  year={2006}
}

@article{ip2025user,
  title={User Preference Meets {P}areto-Optimality in Multi-Objective {B}ayesian Optimization},
  author={Ip, Joshua Hang Sai and Chakrabarty, Ankush and Mesbah, Ali and Romeres, Diego},
  journal={Proceedings of the AAAI Conference on Artificial Intelligence},
  volume={39},
  number={19},
  pages={20246--20254},
  year={2025}
}

@inproceedings{astudillo2020multi,
  title={Multi-attribute {B}ayesian optimization with interactive preference learning},
  author={Astudillo, Raul and Frazier, Peter},
  booktitle={International Conference on Artificial Intelligence and Statistics},
  pages={4496--4507},
  year={2020},
  organization={PMLR}
}

@article{ozaki2024multi,
  title={Multi-objective {B}ayesian optimization with active preference learning},
  author={Ozaki, Ryota and Ishikawa, Kazuki and Kanzaki, Youhei and Takeno, Shion and Takeuchi, Ichiro and Karasuyama, Masayuki},
  journal={Proceedings of the AAAI conference on artificial intelligence},
  volume={38},
  number={13},
  pages={14490--14498},
  year={2024}
}

@book{garnett2023bayesian,
  title={Bayesian optimization},
  author={Garnett, Roman},
  year={2023},
  publisher={Cambridge University Press}
}

@article{gaudrie2020targeting,
  title={Targeting solutions in {B}ayesian multi-objective optimization: sequential and batch versions},
  author={Gaudrie, David and Le Riche, Rodolphe and Picheny, Victor and Enaux, Benoit and Herbert, Vincent},
  journal={Annals of Mathematics and Artificial Intelligence},
  volume={88},
  number={1},
  pages={187--212},
  year={2020},
  publisher={Springer}
}

@article{gaudrie2018budgeted,
  title={Budgeted Multi-Objective Optimization with a Focus on the Central Part of the {P}areto Front--Extended Version},
  author={Gaudrie, David and Riche, Rodolphe Le and Picheny, Victor and Enaux, Benoit and Herbert, Vincent},
  journal={arXiv preprint arXiv:1809.10482},
  year={2018}
}

@article{zuluaga2016pal,
  author  = {Marcela Zuluaga and Andreas Krause and Markus Püschel},
  title   = {$\epsilon$-PAL: An Active Learning Approach to the Multi-Objective Optimization Problem},
  journal = {Journal of Machine Learning Research},
  year    = {2016},
  volume  = {17},
  number  = {104},
  pages   = {1--32},
  url     = {http://jmlr.org/papers/v17/15-047.html}
}

@InProceedings{malkomes2021beyond,
  title = 	 {Beyond the {P}areto Efficient Frontier: Constraint Active Search for Multiobjective Experimental Design},
  author =       {Malkomes, Gustavo and Cheng, Bolong and Lee, Eric H and Mccourt, Mike},
  booktitle = 	 {Proceedings of the 38th International Conference on Machine Learning},
  pages = 	 {7423--7434},
  year = 	 {2021},
  volume = 	 {139},
  publisher =    {PMLR},
}

@inproceedings{papenmeier2023bounce,
 author = {Papenmeier, Leonard and Nardi, Luigi and Poloczek, Matthias},
 booktitle = {Advances in Neural Information Processing Systems},
 pages = {1764--1793},
 publisher = {Curran Associates, Inc.},
 title = {Bounce: Reliable High-Dimensional {B}ayesian Optimization for Combinatorial and Mixed Spaces},
 volume = {36},
 year = {2023}
}

@inproceedings{gonzalez2024survey,
 author = {Gonz\'{a}lez-Duque, Miguel and Michael, Richard and Bartels, Simon and Zainchkovskyy, Yevgen and Hauberg, S\o ren and Boomsma, Wouter},
 booktitle = {Advances in Neural Information Processing Systems},
 pages = {140478--140508},
 publisher = {Curran Associates, Inc.},
 title = {A survey and benchmark of high-dimensional {B}ayesian optimization of discrete sequences},
 volume = {37},
 year = {2024}
}

@article{emmerich2018tutorial,
  title={A tutorial on multiobjective optimization: fundamentals and evolutionary methods},
  author={Emmerich, Michael TM and Deutz, Andr{\'e} H},
  journal={Natural Computing},
  volume={17},
  pages={585--609},
  year={2018},
  publisher={Springer}
}

@article{cheng2017benchmark,
  title={A benchmark test suite for evolutionary many-objective optimization},
  author={Cheng, Ran and Li, Miqing and Tian, Ye and Zhang, Xingyi and Yang, Shengxiang and Jin, Yaochu and Yao, Xin},
  journal={Complex \& Intelligent Systems},
  volume={3},
  pages={67--81},
  year={2017},
  publisher={Springer}
}

@inproceedings{zheng2024boundary,
 author = {Zheng, Ruihao and Wang, Zhenkun},
 booktitle = {Advances in Neural Information Processing Systems},
 pages = {14349--14385},
 publisher = {Curran Associates, Inc.},
 title = {Boundary Decomposition for Nadir Objective Vector Estimation},
 volume = {37},
 year = {2024}
}

@article{kleywegt2002sample,
  title={The sample average approximation method for stochastic discrete optimization},
  author={Kleywegt, Anton J and Shapiro, Alexander and Homem-de-Mello, Tito},
  journal={SIAM Journal on Optimization},
  volume={12},
  number={2},
  pages={479--502},
  year={2002},
  publisher={SIAM}
}

@article{li2025expensive,
  title={Expensive multi-objective {B}ayesian optimization based on diffusion models},
  author={Li, Bingdong and Di, Zixiang and Lu, Yongfan and Qian, Hong and Wang, Feng and Yang, Peng and Tang, Ke and Zhou, Aimin},
  journal={Proceedings of the AAAI Conference on Artificial Intelligence}, 
  volume={39},
  number={25},
  pages={27063--27071},
  year={2025}
}

@book{miettinen1999nonlinear,
  title={Nonlinear Multiobjective Optimization},
  author={Miettinen, Kaisa},
  volume={12},
  year={1999},
  publisher={Springer Science \& Business Media}
}

@article{thompson1933likelihood,
  title={On the likelihood that one unknown probability exceeds another in view of the evidence of two samples},
  author={Thompson, William R},
  journal={Biometrika},
  volume={25},
  number={3/4},
  pages={285--294},
  year={1933},
  publisher={JSTOR}
}

@book{branke2008multiobjective,
  title     = {Multiobjective Optimization: Interactive and Evolutionary Approaches},
  editor    = {Branke, Jürgen and Deb, Kalyanmoy and Miettinen, Kaisa and Słowiński, Roman},
  volume    = {5252},
  year      = {2008},
  publisher = {Springer},
}

@article{emmerich2006single,
  title={Single-and multiobjective evolutionary optimization assisted by {G}aussian random field metamodels},
  author={Emmerich, Michael TM and Giannakoglou, Kyriakos C and Naujoks, Boris},
  journal={IEEE Transactions on Evolutionary Computation},
  volume={10},
  number={4},
  pages={421--439},
  year={2006},
  publisher={IEEE}
}

@article{shang2020survey,
  title={A survey on the hypervolume indicator in evolutionary multiobjective optimization},
  author={Shang, Ke and Ishibuchi, Hisao and He, Linjun and Pang, Lie Meng},
  journal={IEEE Transactions on Evolutionary Computation},
  volume={25},
  number={1},
  pages={1--20},
  year={2020},
  publisher={IEEE}
}

@article{peng2025bayesian,
  title={Bayesian optimization and explainable machine learning for High-dimensional multi-objective optimization of biodegradable magnesium alloys},
  author={Peng, Peng and Peng, Yi and Liu, Fuguo and Long, Shuai and Zhang, Cheng and Tang, Aitao and She, Jia and Zhang, Jianyue and Pan, Fusheng},
  journal={Journal of Materials Science \& Technology},
  year={2025},
  publisher={Elsevier}
}

@article{liang2021benchmarking,
  title={Benchmarking the performance of {B}ayesian optimization across multiple experimental materials science domains},
  author={Liang, Qiaohao and Gongora, Aldair E and Ren, Zekun and Tiihonen, Armi and Liu, Zhe and Sun, Shijing and Deneault, James R and Bash, Daniil and Mekki-Berrada, Flore and Khan, Saif A and others},
  journal={npj Computational Materials},
  volume={7},
  number={1},
  pages={188},
  year={2021},
  publisher={Nature Publishing Group UK London}
}

@article{low2024evolution,
  title={Evolution-guided {B}ayesian optimization for constrained multi-objective optimization in self-driving labs},
  author={Low, Andre KY and Mekki-Berrada, Flore and Gupta, Abhishek and Ostudin, Aleksandr and Xie, Jiaxun and Vissol-Gaudin, Eleonore and Lim, Yee-Fun and Li, Qianxiao and Ong, Yew Soon and Khan, Saif A and others},
  journal={npj Computational Materials},
  volume={10},
  number={1},
  pages={104},
  year={2024},
  publisher={Nature Publishing Group UK London}
}

@inproceedings{lin2022pareto,
 author = {Lin, Xi and Yang, Zhiyuan and Zhang, Xiaoyuan and Zhang, Qingfu},
 booktitle = {Advances in Neural Information Processing Systems},
 pages = {19231--19247},
 publisher = {Curran Associates, Inc.},
 title = {Pareto Set Learning for Expensive Multi-Objective Optimization},
 volume = {35},
 year = {2022}
}

@article{qing2023robust,
  title={A robust multi-objective {B}ayesian optimization framework considering input uncertainty},
  author={Qing, Jixiang and Couckuyt, Ivo and Dhaene, Tom},
  journal={Journal of Global Optimization},
  volume={86},
  number={3},
  pages={693--711},
  year={2023},
  publisher={Springer}
}

@article{liang2024survey,
  title={A survey of surrogate-assisted evolutionary algorithms for expensive optimization},
  author={Liang, Jing and Lou, Yahang and Yu, Mingyuan and Bi, Ying and Yu, Kunjie},
  journal={Journal of Membrane Computing},
  pages={1--20},
  year={2024},
  publisher={Springer}
}

@article{jin2011surrogate,
  title={Surrogate-assisted evolutionary computation: Recent advances and future challenges},
  author={Jin, Yaochu},
  journal={Swarm and Evolutionary Computation},
  volume={1},
  number={2},
  pages={61--70},
  year={2011},
  publisher={Elsevier}
}

@article{li2015many,
  title={Many-objective evolutionary algorithms: A survey},
  author={Li, Bingdong and Li, Jinlong and Tang, Ke and Yao, Xin},
  journal={ACM Computing Surveys (CSUR)},
  volume={48},
  number={1},
  pages={1--35},
  year={2015},
  publisher={Acm New York, NY, USA}
}

@Article{Zitzler1999,
  Title                    = {{Multiobjective evolutionary algorithms: A comparative case study and the strength {P}areto approach}},
  Author                   = {Zitzler, Eckart and Thiele, Lothar},
  Journal                  = {IEEE Transactions on Evolutionary Computation},
  Year                     = {1999},
  Number                   = {4},
  Pages                    = {257--271},
  Volume                   = {3}
}

@article{das1998normal,
  title={Normal-boundary intersection: A new method for generating the {P}areto surface in nonlinear multicriteria optimization problems},
  author={Das, Indraneel and Dennis, John E},
  journal={SIAM journal on optimization},
  volume={8},
  number={3},
  pages={631--657},
  year={1998},
  publisher={SIAM}
}

@article{youn2004reliability,
  title={Reliability-based design optimization for crashworthiness of vehicle side impact},
  author={Youn, Byeng D and Choi, KK and Yang, R-J and Gu, Lei},
  journal={Structural and Multidisciplinary Optimization},
  volume={26},
  pages={272--283},
  year={2004},
  publisher={Springer}
}

@inproceedings{nayebi2019framework,
  title={A framework for {B}ayesian optimization in embedded subspaces},
  author={Nayebi, Amin and Munteanu, Alexander and Poloczek, Matthias},
  booktitle={International Conference on Machine Learning},
  pages={4752--4761},
  year={2019},
  organization={PMLR}
}

@article{han2021high,
  title={High-dimensional {B}ayesian optimization via tree-structured additive models},
  author={Han, Eric and Arora, Ishank and Scarlett, Jonathan},
  journal={Proceedings of the AAAI Conference on Artificial Intelligence},
  volume={35},
  number={9},
  pages={7630--7638},
  year={2021}
}

@InProceedings{hvarfner2024vanilla,
  title = 	 {Vanilla {B}ayesian Optimization Performs Great in High Dimensions},
  author =       {Hvarfner, Carl and Hellsten, Erik Orm and Nardi, Luigi},
  booktitle = 	 {Proceedings of the 41st International Conference on Machine Learning},
  pages = 	 {20793--20817},
  year = 	 {2024},
  volume = 	 {235},
  month = 	 {21--27 Jul},
  publisher =    {PMLR},
}

@inproceedings{eriksson2019scalable,
 author = {Eriksson, David and Pearce, Michael and Gardner, Jacob and Turner, Ryan D and Poloczek, Matthias},
 booktitle = {Advances in Neural Information Processing Systems},
 title = {Scalable global optimization via local {B}ayesian optimization},
 volume = {32},
 pages = {5496--5507},
 year = {2019},
 publisher = {Curran Associates, Inc.},
}

@article{diouane2023trego,
  title={TREGO: a trust-region framework for efficient global optimization},
  author={Diouane, Youssef and Picheny, Victor and Riche, Rodolophe Le and Perrotolo, Alexandre Scotto Di},
  journal={Journal of Global Optimization},
  volume={86},
  number={1},
  pages={1--23},
  year={2023},
  publisher={Springer}
}

@inproceedings{antonov2022high,
  title={High dimensional {B}ayesian optimization with kernel principal component analysis},
  author={Antonov, Kirill and Raponi, Elena and Wang, Hao and Doerr, Carola},
  booktitle={International Conference on Parallel Problem Solving from Nature},
  pages={118--131},
  year={2022},
  organization={Springer}
}

@inproceedings{raponi2020high,
  title={High dimensional {B}ayesian optimization assisted by principal component analysis},
  author={Raponi, Elena and Wang, Hao and Bujny, Mariusz and Boria, Simonetta and Doerr, Carola},
  booktitle={Parallel Problem Solving from Nature--PPSN XVI: 16th International Conference, PPSN 2020, Leiden, The Netherlands, September 5-9, 2020, Proceedings, Part I 16},
  pages={169--183},
  year={2020},
  organization={Springer}
}

@inproceedings{letham2020re,
 author = {Letham, Ben and Calandra, Roberto and Rai, Akshara and Bakshy, Eytan},
 booktitle = {Advances in Neural Information Processing Systems},
 pages = {1546--1558},
 publisher = {Curran Associates, Inc.},
 title = {Re-Examining Linear Embeddings for High-Dimensional {B}ayesian Optimization},
 volume = {33},
 year = {2020}
}

@inproceedings{ziomek2023random,
  title={Are random decompositions all we need in high dimensional {B}ayesian optimisation?},
  author={Ziomek, Juliusz Krzysztof and Ammar, Haitham Bou},
  booktitle={International Conference on Machine Learning},
  pages={43347--43368},
  year={2023},
  organization={PMLR}
}

@inproceedings{delbridge2020randomly,
  title={Randomly projected additive {G}aussian processes for regression},
  author={Delbridge, Ian and Bindel, David and Wilson, Andrew Gordon},
  booktitle={International Conference on Machine Learning},
  pages={2453--2463},
  year={2020},
  organization={PMLR}
}

@inproceedings{eriksson2021high,
  title={High-dimensional {B}ayesian optimization with sparse axis-aligned subspaces},
  author={Eriksson, David and Jankowiak, Martin},
  booktitle={Uncertainty in Artificial Intelligence},
  pages={493--503},
  year={2021},
  organization={PMLR}
}

@article{wang2016bayesian,
  title={Bayesian optimization in a billion dimensions via random embeddings},
  author={Wang, Ziyu and Hutter, Frank and Zoghi, Masrour and Matheson, David and De Feitas, Nando},
  journal={Journal of Artificial Intelligence Research},
  volume={55},
  pages={361--387},
  year={2016}
}

@article{santoni2024comparison,
  title={Comparison of high-dimensional {B}ayesian optimization algorithms on {BBOB}},
  author={Santoni, Maria Laura and Raponi, Elena and De Leone, Renato and Doerr, Carola},
  journal={ACM Transactions on Evolutionary Learning and Optimisation},
  volume={4},
  number={3},
  pages={1--33},
  year={2024},
}

@article{binois2022survey,
  title={A survey on high-dimensional {G}aussian process modeling with application to {B}ayesian optimization},
  author={Binois, Mickael and Wycoff, Nathan},
  journal={ACM Transactions on Evolutionary Learning and Optimization},
  volume={2},
  number={2},
  pages={1--26},
  year={2022},
  publisher={ACM New York, NY}
}

@article{chen2024pg,
  title={{PG-LBO}: enhancing high-dimensional {B}ayesian optimization with pseudo-label and {G}aussian Process guidance},
  author={Chen, Taicai and Duan, Yue and Li, Dong and Qi, Lei and Shi, Yinghuan and Gao, Yang},
  journal={Proceedings of the AAAI Conference on Artificial Intelligence},
  volume={38},
  number={10},
  pages={11381--11389},
  year={2024}
}

@InProceedings{wang2018batched,
  title = 	 {Batched large-scale {B}ayesian optimization in high-dimensional spaces},
  author = 	 {Wang, Zi and Gehring, Clement and Kohli, Pushmeet and Jegelka, Stefanie},
  booktitle = 	 {Proceedings of the Twenty-First International Conference on Artificial Intelligence and Statistics},
  pages = 	 {745--754},
  year = 	 {2018},
  volume = 	 {84},
  publisher =    {PMLR},
}

@article{papenmeier2025understanding,
  title={Understanding High-Dimensional {B}ayesian Optimization},
  author={Papenmeier, Leonard and Poloczek, Matthias and Nardi, Luigi},
  journal={arXiv preprint arXiv:2502.09198},
  year={2025}
}

@InProceedings{wu2020practical,
  title = 	 {Practical Multi-fidelity {B}ayesian Optimization for Hyperparameter Tuning},
  author =       {Wu, Jian and Toscano-Palmerin, Saul and Frazier, Peter I. and Wilson, Andrew Gordon},
  booktitle = 	 {Proceedings of The 35th Uncertainty in Artificial Intelligence Conference},
  pages = 	 {788--798},
  year = 	 {2020},
  volume = 	 {115},
  series = 	 {Proceedings of Machine Learning Research},
  month = 	 {22--25 Jul},
  publisher =    {PMLR},
}

@inproceedings{zhang2017information,
  title={Information-based multi-fidelity {B}ayesian optimization},
  author={Zhang, Yehong and Hoang, Trong Nghia and Low, Bryan Kian Hsiang and Kankanhalli, Mohan},
  booktitle={NIPS workshop on {B}ayesian optimization},
  volume={49},
  year={2017},
  organization={Journal of Machine Learning Research JMLR. org Cambridge, MA}
}

@inproceedings{li2020multi,
 author = {Li, Shibo and Xing, Wei and Kirby, Robert and Zhe, Shandian},
 booktitle = {Advances in Neural Information Processing Systems},
 pages = {8521--8531},
 publisher = {Curran Associates, Inc.},
 title = {Multi-Fidelity {B}ayesian Optimization via Deep Neural Networks},
 volume = {33},
 year = {2020}
}

@InProceedings{takeno2020multi,
  title = 	 {Multi-fidelity {B}ayesian Optimization with Max-value Entropy Search and its Parallelization},
  author =       {Takeno, Shion and Fukuoka, Hitoshi and Tsukada, Yuhki and Koyama, Toshiyuki and Shiga, Motoki and Takeuchi, Ichiro and Karasuyama, Masayuki},
  booktitle = 	 {Proceedings of the 37th International Conference on Machine Learning},
  pages = 	 {9334--9345},
  year = 	 {2020},
  volume = 	 {119},
  series = 	 {Proceedings of Machine Learning Research},
  publisher =    {PMLR},
}

@InProceedings{song2019general,
  title = 	 {A General Framework for Multi-fidelity {B}ayesian Optimization with {G}aussian Processes},
  author =       {Song, Jialin and Chen, Yuxin and Yue, Yisong},
  booktitle = 	 {Proceedings of the Twenty-Second International Conference on Artificial Intelligence and Statistics},
  pages = 	 {3158--3167},
  year = 	 {2019},
  volume = 	 {89},
  series = 	 {Proceedings of Machine Learning Research},
  publisher =    {PMLR},
}

@article{moss2021gibbon,
  author  = {Henry B. Moss and David S. Leslie and Javier Gonzalez and Paul Rayson},
  title   = {GIBBON: General-purpose Information-Based {B}ayesian Optimisation},
  journal = {Journal of Machine Learning Research},
  year    = {2021},
  volume  = {22},
  number  = {235},
  pages   = {1--49},
}

@InProceedings{kandasamy2017multi,
  title = 	 {Multi-fidelity {B}ayesian Optimisation with Continuous Approximations},
  author =       {Kirthevasan Kandasamy and Gautam Dasarathy and Jeff Schneider and Barnab{\'a}s P{\'o}czos},
  booktitle = 	 {Proceedings of the 34th International Conference on Machine Learning},
  pages = 	 {1799--1808},
  year = 	 {2017},
  volume = 	 {70},
  series = 	 {Proceedings of Machine Learning Research},
  publisher =    {PMLR},
}

@inproceedings{lin2025few,
  title={Few for many: Tchebycheff set scalarization for many-objective optimization},
  author={Lin, Xi and Liu, Yilu and Zhang, Xiaoyuan and Liu, Fei and Wang, Zhenkun and Zhang, Qingfu},
  booktitle={The Thirteenth International Conference on Learning Representations},
  year={2025}
}

@inproceedings{xu2025standard,
  title={Standard {G}aussian process is all you need for high-dimensional {B}ayesian optimization},
  author={Xu, Zhitong and Wang, Haitao and Phillips, Jeff M and Zhe, Shandian},
  booktitle={The Thirteenth International Conference on Learning Representations},
  year={2025}
}

@article{cheaitou2019greening,
  title={Greening of maritime transportation: a multi-objective optimization approach},
  author={Cheaitou, Ali and Cariou, Pierre},
  journal={Annals of Operations Research},
  volume={273},
  number={1},
  pages={501--525},
  year={2019},
  publisher={Springer}
}

@article{deb2009reliability,
  title={Reliability-based optimization using evolutionary algorithms},
  author={Deb, Kalyanmoy and Gupta, Shubham and Daum, David and Branke, J{\"u}rgen and Mall, Abhishek Kumar and Padmanabhan, Dhanesh},
  journal={IEEE Transactions on Evolutionary Computation},
  volume={13},
  number={5},
  pages={1054--1074},
  year={2009},
  publisher={IEEE}
}

@article{tanabe2020easy,
  title={An easy-to-use real-world multi-objective optimization problem suite},
  author={Tanabe, Ryoji and Ishibuchi, Hisao},
  journal={Applied Soft Computing},
  volume={89},
  pages={106078},
  year={2020},
  publisher={Elsevier}
}

@article{belakaria2020multi,
  title={Multi-fidelity multi-objective {B}ayesian optimization: An output space entropy search approach},
  author={Belakaria, Syrine and Deshwal, Aryan and Doppa, Janardhan Rao},
  journal={Proceedings of the AAAI Conference on artificial intelligence},
  volume={34},
  number={06},
  pages={10035--10043},
  year={2020}
}

@article{belakaria2020uncertainty,
  title={Uncertainty-aware search framework for multi-objective {B}ayesian optimization},
  author={Belakaria, Syrine and Deshwal, Aryan and Jayakodi, Nitthilan Kannappan and Doppa, Janardhan Rao},
  journal={Proceedings of the AAAI Conference on Artificial Intelligence},
  volume={34},
  number={06},
  pages={10044--10052},
  year={2020}
}

@inproceedings{daulton2022robust,
  title={Robust multi-objective {B}ayesian optimization under input noise},
  author={Daulton, Samuel and Cakmak, Sait and Balandat, Maximilian and Osborne, Michael A and Zhou, Enlu and Bakshy, Eytan},
  booktitle={International Conference on Machine Learning},
  pages={4831--4866},
  year={2022},
  organization={PMLR}
}

@inproceedings{daulton2022multi,
  title={Multi-objective {B}ayesian optimization over high-dimensional search spaces},
  author={Daulton, Samuel and Eriksson, David and Balandat, Maximilian and Bakshy, Eytan},
  booktitle={Uncertainty in Artificial Intelligence},
  pages={507--517},
  year={2022},
  organization={PMLR}
}

@inproceedings{abdolshah2019multi,
 author = {Abdolshah, Majid and Shilton, Alistair and Rana, Santu and Gupta, Sunil and Venkatesh, Svetha},
 booktitle = {Advances in Neural Information Processing Systems},
 pages = {},
 publisher = {Curran Associates, Inc.},
 title = {Multi-objective {B}ayesian optimisation with preferences over objectives},
 volume = {32},
 year = {2019}
}

@article{dunlap2023continuous,
  title={Continuous flow synthesis of pyridinium salts accelerated by multi-objective {B}ayesian optimization with active learning},
  author={Dunlap, John H and Ethier, Jeffrey G and Putnam-Neeb, Amelia A and Iyer, Sanjay and Luo, Shao-Xiong Lennon and Feng, Haosheng and Torres, Jose Antonio Garrido and Doyle, Abigail G and Swager, Timothy M and Vaia, Richard A and others},
  journal={Chemical Science},
  volume={14},
  number={30},
  pages={8061--8069},
  year={2023},
  publisher={Royal Society of Chemistry}
}

@article{park2018multi,
  title={Multi-objective {B}ayesian optimization of chemical reactor design using computational fluid dynamics},
  author={Park, Seongeon and Na, Jonggeol and Kim, Minjun and Lee, Jong Min},
  journal={Computers \& Chemical Engineering},
  volume={119},
  pages={25--37},
  year={2018},
  publisher={Elsevier}
}

@article{shields2021bayesian,
  title={{B}ayesian reaction optimization as a tool for chemical synthesis},
  author={Shields, Benjamin J and Stevens, Jason and Li, Jun and Parasram, Marvin and Damani, Farhan and Alvarado, Jesus I Martinez and Janey, Jacob M and Adams, Ryan P and Doyle, Abigail G},
  journal={Nature},
  volume={590},
  number={7844},
  pages={89--96},
  year={2021},
  publisher={Nature Publishing Group UK London}
}

@article{ishibuchi2018specify,
    author = {Ishibuchi, Hisao and Imada, Ryo and Setoguchi, Yu and Nojima, Yusuke},
    title = {How to Specify a Reference Point in Hypervolume Calculation for Fair Performance Comparison},
    journal = {Evolutionary Computation},
    volume = {26},
    number = {3},
    pages = {411-440},
    year = {2018},
}

@inproceedings{jain2013improved,
  title={An improved adaptive approach for elitist nondominated sorting genetic algorithm for many-objective optimization},
  author={Jain, Himanshu and Deb, Kalyanmoy},
  booktitle={International Conference on Evolutionary Multi-Criterion Optimization},
  pages={307--321},
  year={2013},
  organization={Springer}
}

@incollection{deb2005scalable,
  title={Scalable test problems for evolutionary multiobjective optimization},
  author={Deb, Kalyanmoy and Thiele, Lothar and Laumanns, Marco and Zitzler, Eckart},
  booktitle={Evolutionary multiobjective optimization: theoretical advances and applications},
  pages={105--145},
  year={2005},
  publisher={Springer}
}

@article{picheny2019bayesian,
  title={A {B}ayesian optimization approach to find Nash equilibria},
  author={Picheny, Victor and Binois, Mickael and Habbal, Abderrahmane},
  journal={Journal of Global Optimization},
  volume={73},
  pages={171--192},
  year={2019},
  publisher={Springer}
}

@inproceedings{zhang2020random,
  title = 	 {Random Hypervolume Scalarizations for Provable Multi-Objective Black Box Optimization},
  author =       {Zhang, Richard and Golovin, Daniel},
  booktitle = 	 {Proceedings of the 37th International Conference on Machine Learning},
  pages = 	 {11096--11105},
  year = 	 {2020},
  volume = 	 {119},
  series = 	 {Proceedings of Machine Learning Research},
  publisher =    {PMLR},
}

@inproceedings{konakovic2020diversity,
 title = {Diversity-Guided Multi-Objective {B}ayesian Optimization With Batch Evaluations},
 author = {Konakovic Lukovic, Mina and Tian, Yunsheng and Matusik, Wojciech},
 booktitle = {Advances in Neural Information Processing Systems},
 pages = {17708--17720},
 publisher = {Curran Associates, Inc.},
 volume = {33},
 year = {2020}
}

@article{binois2020kalai,
  title   = {The Kalai-Smorodinsky solution for many-objective {B}ayesian optimization},
  author  = {Mickael Binois and Victor Picheny and Patrick Taillandier and Abderrahmane Habbal},
  journal = {Journal of Machine Learning Research},
  year    = {2020},
  volume  = {21},
  number  = {150},
  pages   = {1--42},
}

@inproceedings{zhao2023exact,
  title={Exact Formulas for the Computation of Expected Tchebycheff Improvement},
  author={Zhao, Liang and Zhang, Qingfu},
  booktitle={2023 IEEE Congress on Evolutionary Computation (CEC)},
  pages={1--8},
  year={2023},
  organization={IEEE}
}

@article{couckuyt2014fast,
  title={Fast calculation of multiobjective probability of improvement and expected improvement criteria for {P}areto optimization},
  author={Couckuyt, Ivo and Deschrijver, Dirk and Dhaene, Tom},
  journal={Journal of Global Optimization},
  volume={60},
  number={3},
  pages={575--594},
  year={2014},
  publisher={Springer}
}

@inproceedings{suzuki2020multi,
  title = 	 {Multi-objective {B}ayesian Optimization using {P}areto-frontier Entropy},
  author =       {Suzuki, Shinya and Takeno, Shion and Tamura, Tomoyuki and Shitara, Kazuki and Karasuyama, Masayuki},
  booktitle = 	 {Proceedings of the 37th International Conference on Machine Learning},
  pages = 	 {9279--9288},
  year = 	 {2020},
  volume = 	 {119},
  series = 	 {Proceedings of Machine Learning Research},
  publisher =    {PMLR},
}

@article{belakaria2021output,
  title={Output space entropy search framework for multi-objective {B}ayesian optimization},
  author={Belakaria, Syrine and Deshwal, Aryan and Doppa, Janardhan Rao},
  journal={Journal of artificial intelligence research},
  volume={72},
  pages={667--715},
  year={2021}
}

@inproceedings{belakaria2019max,
 title = {Max-value Entropy Search for Multi-Objective {B}ayesian Optimization},
 author = {Belakaria, Syrine and Deshwal, Aryan and Doppa, Janardhan Rao},
 booktitle = {Advances in Neural Information Processing Systems},
 pages = {},
 publisher = {Curran Associates, Inc.},
 volume = {32},
 year = {2019}
}

@article{garrido2023parallel,
  title={Parallel predictive entropy search for multi-objective {B}ayesian optimization with constraints applied to the tuning of machine learning algorithms},
  author={Garrido-Merch{\'a}n, Eduardo C and Fern{\'a}ndez-S{\'a}nchez, Daniel and Hern{\'a}ndez-Lobato, Daniel},
  journal={Expert Systems with Applications},
  volume={215},
  pages={119328},
  year={2023},
  publisher={Elsevier}
}

@article{garrido2019predictive,
    title = {Predictive Entropy Search for Multi-objective {B}ayesian Optimization with Constraints},
    journal = {Neurocomputing},
    volume = {361},
    pages = {50-68},
    year = {2019},
    author = {Eduardo C. Garrido-Merchán and Daniel Hernández-Lobato},
}

@inproceedings{hernandez2016predictive,
  title = 	 {Predictive Entropy Search for Multi-objective {B}ayesian Optimization},
  author = 	 {Hernandez-Lobato, Daniel and Hernandez-Lobato, Jose and Shah, Amar and Adams, Ryan},
  booktitle = 	 {Proceedings of The 33rd International Conference on Machine Learning},
  pages = 	 {1492--1501},
  year = 	 {2016},
  volume = 	 {48},
  series = 	 {Proceedings of Machine Learning Research},
  address = 	 {New York, USA},
  publisher =    {PMLR},
}

@inproceedings{daulton2023hypervolume,
  title = 	 {Hypervolume Knowledge Gradient: A Lookahead Approach for Multi-Objective {B}ayesian Optimization with Partial Information},
  author =       {Daulton, Sam and Balandat, Maximilian and Bakshy, Eytan},
  booktitle = 	 {Proceedings of the 40th International Conference on Machine Learning},
  pages = 	 {7167--7204},
  year = 	 {2023},
  volume = 	 {202},
  series = 	 {Proceedings of Machine Learning Research},
  publisher =    {PMLR},
}

@inproceedings{paszke2019pytorch,
  title     = {PyTorch: An Imperative Style, High-Performance Deep Learning Library},
  author    = {Adam Paszke and Sam Gross and Francisco Massa and Adam Lerer and James Bradbury and Gregory Chanan and Trevor Killeen and Zeming Lin and Natalia Gimelshein and Luca Antiga and Alban Desmaison and Andreas Kopf and Edward Yang and Zachary DeVito and Martin Raison and Alykhan Tejani and Sasank Chilamkurthy and Benoit Steiner and Lu Fang and Junjie Bai and Soumith Chintala},
  booktitle = {Advances in Neural Information Processing Systems},
  year      = {2019},
  volume    = {32},
  publisher = {Curran Associates, Inc.},
}

@inproceedings{gardner2018gpytorch,
 author = {Gardner, Jacob and Pleiss, Geoff and Weinberger, Kilian Q and Bindel, David and Wilson, Andrew G},
 booktitle = {Advances in Neural Information Processing Systems},
 pages = {},
 publisher = {Curran Associates, Inc.},
 title = {GPyTorch: Blackbox Matrix-Matrix {G}aussian Process Inference with {GPU} Acceleration},
 volume = {31},
 year = {2018}
}

@article{li2014evolutionary,
  title={An evolutionary many-objective optimization algorithm based on dominance and decomposition},
  author={Li, Ke and Deb, Kalyanmoy and Zhang, Qingfu and Kwong, Sam},
  journal={IEEE transactions on Evolutionary Computation},
  volume={19},
  number={5},
  pages={694--716},
  year={2014},
  publisher={IEEE}
}

@article{jain2013evolutionary,
  title={An evolutionary many-objective optimization algorithm using reference-point based nondominated sorting approach, part {II}: Handling constraints and extending to an adaptive approach},
  author={Jain, Himanshu and Deb, Kalyanmoy},
  journal={IEEE Transactions on Evolutionary Computation},
  volume={18},
  number={4},
  pages={602--622},
  year={2013},
  publisher={IEEE}
}

@article{deb2013evolutionary,
  title={An evolutionary many-objective optimization algorithm using reference-point-based nondominated sorting approach, part {I}: solving problems with box constraints},
  author={Deb, Kalyanmoy and Jain, Himanshu},
  journal={IEEE transactions on Evolutionary Computation},
  volume={18},
  number={4},
  pages={577--601},
  year={2013},
  publisher={IEEE}
}

@article{sobol1967distribution,
  title={The distribution of points in a cube and the approximate evaluation of integrals},
  author={Sobol, Ilya M},
  journal={USSR Computational Mathematics and Mathematical Physics},
  volume={7},
  number={4},
  pages={86--112},
  year={1967}
}

@article{yang2019multi,
  title={Multi-objective {B}ayesian global optimization using expected hypervolume improvement gradient},
  author={Yang, Kaifeng and Emmerich, Michael and Deutz, Andr{\'e} and B{\"a}ck, Thomas},
  journal={Swarm and Evolutionary Computation},
  volume={44},
  pages={945--956},
  year={2019},
  publisher={Elsevier}
}

@article{deng2025expected,
  title={Expected Hypervolume Improvement Is a Particular Hypervolume Improvement},
  author={Deng, Jingda and Sun, Jianyong and Zhang, Qingfu and Li, Hui},
  journal={Proceedings of the AAAI Conference on Artificial Intelligence},
  volume={39},
  number={15},
  pages={16217--16225},
  year={2025}
}

@article{jiang2025trading,
  title={Trading off quality and uncertainty through multi-objective optimisation in batch {B}ayesian optimisation},
  author={Jiang, Chao and Li, Miqing},
  journal={Proceedings of the AAAI Conference on Artificial Intelligence},
  volume={39},
  number={25}, 
  year={2025},
  pages={27027-27035}
}

@article{jiang2025multi,
  title={Multi-objectivising acquisition functions in {B}ayesian optimisation},
  author={Jiang, Chao and Li, Miqing},
  journal={ACM Transactions on Evolutionary Learning and Optimization},
  volume = {5},
  number = {2},
  year={2025},
  publisher={ACM New York, NY},
}

@inproceedings{tu2022joint,
  title={Joint entropy search for multi-objective {B}ayesian optimization},
  author={Tu, Ben and Gandy, Axel and Kantas, Nikolas and Shafei, Behrang},
  booktitle={Advances in Neural Information Processing Systems},
  volume={35},
  pages={9922--9938},
  publisher = {Curran Associates, Inc.},
  year={2022}
}

@inproceedings{ament2023unexpected,
    author = {Ament, Sebastian and Daulton, Samuel and Eriksson, David and Balandat, Maximilian and Bakshy, Eytan},
    booktitle = {Advances in Neural Information Processing Systems},
    pages = {20577--20612},
    publisher = {Curran Associates, Inc.},
    title = {Unexpected Improvements to Expected Improvement for {B}ayesian Optimization},
    volume = {36},
    year = {2023}
}

@article{ahmadianshalchi2024pareto,
  title={Pareto Front-Diverse Batch Multi-Objective {B}ayesian Optimization},
  author={Ahmadianshalchi, Alaleh and Belakaria, Syrine and Doppa, Janardhan Rao},
  journal={Proceedings of the AAAI Conference on Artificial Intelligence},
  volume={38},
  number={10},
  pages={10784--10794},
  year={2024}
}

@inproceedings{chugh2020scalarizing,
  title={Scalarizing functions in {B}ayesian multiobjective optimization},
  author={Chugh, Tinkle},
  booktitle={2020 IEEE Congress on Evolutionary Computation (CEC)},
  pages={1--8},
  year={2020},
  organization={IEEE}
}

@inproceedings{ponweiser2008multiobjective,
  title={Multiobjective Optimization on a Limited Budget of Evaluations Using Model-Assisted $\mathcal{S}$-Metric Selection},
  author={Ponweiser, Wolfgang and Wagner, Tobias and Biermann, Dirk and Vincze, Markus},
  booktitle={Parallel Problem Solving from Nature -- PPSN X},
  pages={784--794},
  year={2008},
  organization={Springer}
}

@article{lai1985asymptotically,
	title={Asymptotically efficient adaptive allocation rules},
	author={Lai, Tze Leung and Robbins, Herbert},
	journal={Advances in Applied Mathematics},
	volume={6},
	number={1},
	pages={4--22},
	year={1985},
	publisher={Academic Press}
}

@article{jones1998efficient,
  title={Efficient global optimization of expensive black-box functions},
  author={Jones, Donald R and Schonlau, Matthias and Welch, William J},
  journal={Journal of Global Optimization},
  volume={13},
  number={4},
  pages={455--492},
  year={1998},
  publisher={Springer}
}

@inproceedings{balandat2020botorch,
 author = {Balandat, Maximilian and Karrer, Brian and Jiang, Daniel and Daulton, Samuel and Letham, Ben and Wilson, Andrew G and Bakshy, Eytan},
 booktitle = {Advances in Neural Information Processing Systems},
 pages = {21524--21538},
 publisher = {Curran Associates, Inc.},
 title = {BoTorch: A Framework for Efficient Monte-Carlo {B}ayesian Optimization},
 volume = {33},
 year = {2020}
}

@inproceedings{paria2020flexible,
  title = 	 {A Flexible Framework for Multi-Objective {B}ayesian Optimization using Random Scalarizations},
  author =       {Paria, Biswajit and Kandasamy, Kirthevasan and P{\'{o}}czos, Barnab{\'{a}}s},
  booktitle = 	 {Proceedings of The 35th Uncertainty in Artificial Intelligence Conference},
  pages = 	 {766--776},
  year = 	 {2020},
  editor = 	 {Adams, Ryan P. and Gogate, Vibhav},
  volume = 	 {115},
  series = 	 {Proceedings of Machine Learning Research},
  publisher =    {PMLR},
}

@article{zhang2009expensive,
  author={Zhang, Qingfu and Liu, Wudong and Tsang, Edward and Virginas, Botond},
  journal={IEEE Transactions on Evolutionary Computation}, 
  title={Expensive multiobjective optimization by {MOEA/D} with {G}aussian process model}, 
  year={2010},
  volume={14},
  number={3},
  pages={456-474},
}

@inproceedings{daulton2021parallel,
 author = {Daulton, Samuel and Balandat, Maximilian and Bakshy, Eytan},
 booktitle = {Advances in Neural Information Processing Systems},
 publisher = {Curran Associates, Inc.},
 title = {Parallel {B}ayesian Optimization of Multiple Noisy Objectives with Expected Hypervolume Improvement},
 volume = {34},
 year = {2021}
}

@inproceedings{daulton2020differentiable,
 author = {Daulton, Samuel and Balandat, Maximilian and Bakshy, Eytan},
 booktitle = {Advances in Neural Information Processing Systems},
 pages = {9851--9864},
 publisher = {Curran Associates, Inc.},
 title = {Differentiable Expected Hypervolume Improvement for Parallel Multi-Objective {B}ayesian Optimization},
 volume = {33},
 year = {2020}
}

@article{knowles2006parego,
  title={ParEGO: A hybrid algorithm with on-line landscape approximation for expensive multiobjective optimization problems},
  author={Knowles, Joshua},
  journal={IEEE transactions on Evolutionary Computation},
  volume={10},
  number={1},
  pages={50--66},
  year={2006},
  publisher={IEEE}
}

@article{holm1979simple,
  title={A simple sequentially rejective multiple test procedure},
  author={Holm, Sture},
  journal={Scandinavian Journal of Statistics},
  pages={65--70},
  year={1979},
  publisher={JSTOR}
}

@incollection{wilcoxon1992individual,
  title={Individual comparisons by ranking methods},
  author={Wilcoxon, Frank},
  booktitle={Breakthroughs in Statistics: Methodology and Distribution},
  pages={196--202},
  year={1992},
  publisher={Springer}
}

\clearpage

\appendix
\onecolumn

\begin{center}
\hrule height 4pt
\vskip 0.25in
\vskip -\parskip
    {\LARGE\bf  Appendix to:\\[2ex] \papertitle}
\vskip 0.29in
\vskip -\parskip
\hrule height 1pt
\vskip 0.2in%
\end{center}

\section{Multi-objective Bayesian Optimisation (MOBO)}\label{appdx:sec:BO}
MOBO consists of two main steps, i.e., training $m$ Gaussian process models based on the observed solutions and optimising an acquisition function $\alpha(\bm x): \mathcal{X} \to \mathbb{R}$ to select a solution for evaluation. 
In this work, we model each objective with an independent Gaussian process $f_i \sim \mathcal{GP}(m_i(\bm x),k_i(\bm x,\bm x'))$, where $m_i(\bm x): \mathcal{X} \to \mathbb{R}$ is the $i$th mean function, and $k_i(\cdot,\cdot): \mathcal{X} \times \mathcal{X} \to \mathbb{R}$ is the $i$th covariance function. 
We use the notation $K(\mathrm{A}, \mathrm{B})$ to represent the covariance matrix at all pairs of solutions in set $\mathrm{A}$ and in set $\mathrm{B}$. 
Given $n$ observed solutions $\mathcal{D}^n= \{(\bm x^t,\bm y^t)\}_{t=1}^n$ where $\bm y^t = \bm{f}(\bm x^t) + \bm\zeta^t$ and the noise $\bm \zeta^t \sim \mathcal{N}(0, \text{diag} (\bm \sigma_\zeta^2))$, the posterior distribution of $i$th objective at a new location $\bm x$ is a Gaussian distribution: 
\begin{equation}
	p(f_i(\bm x)|\mathcal{D}^n) \sim \mathcal{N}(\mu_i(\bm x),\sigma_i^2(\bm x))   
\end{equation}
\begin{equation}
	\mu_i(\bm x)=K(\bm x,X^n)(K(X^n,X^n)+\sigma_{\zeta_i}^2\mathbf{I})^{-1}Y_i
\end{equation}
\begin{equation}
	\sigma_i^2(\bm x)=K(\bm x,\bm x)-K(\bm x,X^n)((K(X^n,X^n)+\sigma_{\zeta_i}^2\mathbf{I})^{-1}K(X^n,\bm x) 
\end{equation}
where $\mu_i(\bm x)$ and $\sigma_i^2(\bm x)$ are the mean and variance at $\bm x$, respectively; 
$X^n=(\bm x^1,\dots,\bm x^n)\in \mathbb{R}^{n\times d}$ and $Y_i^n=(y_i^1,\dots,y_i^n)\in \mathbb{R}^{n}$ are the matrix of evaluated solutions and the corresponding vector of $y$ values, respectively; 
$\sigma_{\zeta_i}^2$ is the variance of the observation noise $\zeta_i \sim \mathcal{N}(0, \sigma_{\zeta_i}^2)$, and corresponds to the $i$-th diagonal entry of the noise covariance matrix $\text{diag}(\bm{\sigma}_\zeta^2)$; 


\section{Illustrative Example of ESPI}\label{appdx:sec:ESPI}

To help understand the proposed ESPI, an illustration of ESPI in a bi-objective case is shown in Figure~\ref{fig:espi}. 
In this example, the utopian point is denoted by the black star, while the red dots represent the nondominated solutions from the current dataset. 
A new candidate point, whose objective values are yet to be observed, is shown as the blue dot. 
The red dotted line indicates the shortest Euclidean distance $g^*$ from the utopian point to the existing nondominated set. 
In contrast, the blue dotted line represents the distance from the new point to the utopian point. 
When a new point lies within the shaded blue region and is closer to the utopian point, the improvement $I_{SP}(\cdot)$ is higher.

\begin{figure}[!ht]
    \centering
    \begin{minipage}{0.35\linewidth}{
    \includegraphics[width=\linewidth]{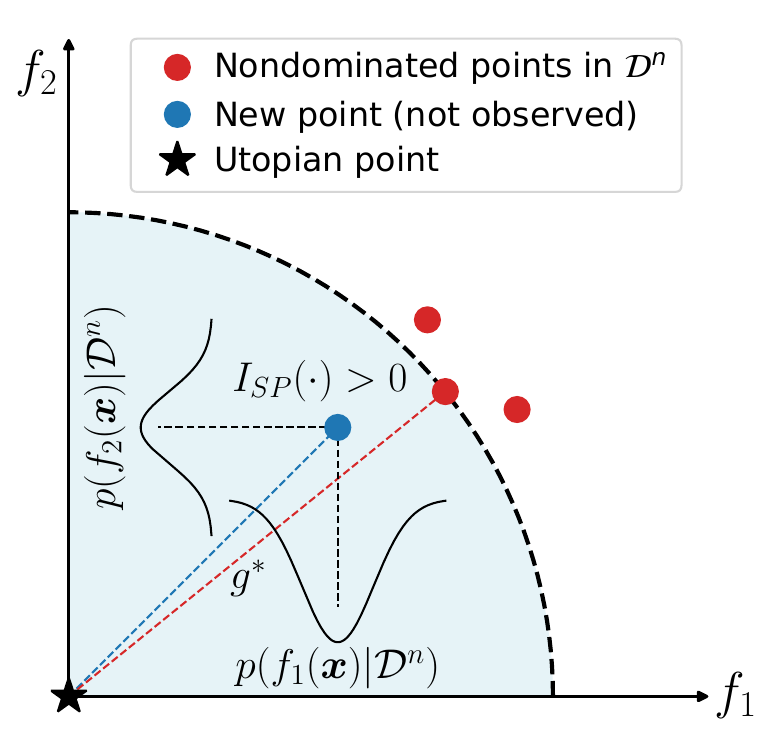}	
    }\end{minipage}
    \begin{minipage}{0.58\linewidth}{
    \caption{Illustration of the proposed expected single-point improvement (ESPI) in a bi-objective space. 
    The figure shows the utopian point (black star), nondominated points in the current dataset $\mathcal{D}^n$ (red dots), and a new candidate point whose true objective values are not yet observed (blue dot). 
    The dashed arc represents the current best Euclidean distance from the utopian point to the nondominated set, denoted as $g^*$ (red dotted line). 
    The blue dotted line represents the distance from the new candidate point to the utopian point $\bm{z}^{*}$, 
    which may improve upon the current best distance $g^*$. 
    When a new point lies within the shaded blue region and is closer to the utopian point, the improvement $I_{SP}(\cdot)$ is higher.  
    }
    \label{fig:espi}
    }\end{minipage}
\end{figure}

\newpage
\section{Extended Related Work}~\label{Appendix:related_work}


Over the last two decades, a variety of BO methods have been proposed to tackle expensive multi-objective optimisation 
problems~\citep{jeong2005efficient,bautista2009sequential,svenson2011computer,svenson2016multiobjective,zuluaga2016pal,zhan2017expected,picheny2019bayesian,belakaria2020uncertainty,malkomes2021beyond,park2023botied,ngo2025mobo}. Most of them aim to identify a good approximation of the entire Pareto front~\citep{keane2006statistical,svenson2011computer,parr2013improvement,zuluaga2013active,ahmadianshalchi2024pareto}. To do so, some studies convert a multi-objective problem into multiple single-objective problems by a scalarisation function (e.g., random augmented Tchebycheff scalarisation). They then optimise acquisition functions from single-objective BO to determine the next evaluation point(s)~\citep{knowles2006parego,zhao2023exact}. In these studies, different acquisition functions are employed, such as expected improvement (EI)~\citep{jones1998efficient} in~\cite{knowles2006parego,zhang2009expensive,namura2017expected,chugh2020scalarizing}, Thompson sampling (TS)~\citep{thompson1933likelihood} in~\cite{paria2020flexible,zhang2020random}, and upper confidence bound (UCB)~\citep{lai1985asymptotically} in~\cite{paria2020flexible,zhang2020random,li2024constrained}.

The remaining studies directly optimise multi-objective problems by considering the definition of optimality in multi-objective optimisation, i.e., the Pareto dominance relation.
A representative approach is to use HV since maximising the HV value is equivalent to finding the entire Pareto front~\citep{ponweiser2008multiobjective,couckuyt2014fast,daulton2020differentiable,daulton2021parallel,renganathan2025qpots,bhatija2025multi}.
Along this line, expected hypervolume improvement (EHVI) is widely considered~\citep{emmerich2006single,daulton2023hypervolume,qing2023robust,yang2019multi,yang2019efficient,deng2025expected} as it is a natural
extension of the EI for multi-objective optimisation. 
Another idea is to leverage information theory to guide exploration toward regions likely contributing to the Pareto front. 
Such methods focus on improving the posterior of optimal inputs (i.e., the approximated Pareto set)~\citep{garrido2023parallel,garrido2019predictive,hernandez2016predictive}, optimal outputs (i.e., the approximated Pareto front)~\citep{belakaria2019max,belakaria2021output,suzuki2020multi}, or both of them~\citep{tu2022joint}. 

That said, there do exist a few studies that do not aim to identify the entire Pareto front. Among them, some attempt to use decision-maker preferences to guide the search towards specific region(s)~\citep{abdolshah2019multi,astudillo2020multi,ozaki2024multi,ip2025user}. Such methods typically adjust the target region by eliciting or updating decision-maker preferences during 
optimisation. Another attempt is to directly target certain region(s) of the Pareto front, without an assumption that decision-maker preferences can be available or elicited~\citep{gaudrie2018budgeted,gaudrie2020targeting,binois2020kalai}. For example, \cite{gaudrie2018budgeted,gaudrie2020targeting} propose the Centred Expected Hypervolume Improvement (C-EHVI) by dynamically adjusting the reference point using the Kalai-Smorodinsky equilibrium (also used in \cite{binois2020kalai}) to approach the central part of the Pareto front. 
Such work is more relevant to our study, and we have thus included C-EHVI in our experimental comparison.

\section{Theoretical Results}\label{appdx:sec:convergence}


\subsection{Proof of Theorem~\ref{thm:SAA_ESPI}.}\label{appdx:sec:proof1}

We consider the setting from~\citet[Section D.5]{balandat2020botorch}. 
Let $\epsilon^t ~\sim \mathcal{N}(0, I_{m})$.\footnote{Theorem~\ref{thm:SAA_ESPI} can be extended to handle non-iid base samples from a family of quasi-Monte Carlo methods
as in~\citet{balandat2020botorch}.} 
Using the reparameterisation trick, we can write the posterior at $\bm x$ as
\begin{equation*}
    \bm{\bar{f}_t}(\bm x, \epsilon^t) = \bar{\mu}(\bm x) + \bar{L}(\bm x)\epsilon^t
\end{equation*}
where $\bar{\mu}(\bm x):\mathbb{R}^d \rightarrow \mathbb{R}^m$ is the multi-output GP's posterior mean; $\bar{L}(\bm x) \in \mathbb{R}^{m \times m}$ is a root decomposition (often a Cholesky decomposition) of the multi-output GP’s posterior covariance $\bar{K}\in \mathbb{R}^{m \times m}$; and $\epsilon^t\in \mathbb{R}^m$. 
Let
\begin{align*}
    \bar{A}(\bm x, \epsilon^t) &= \max\Big( 0,\, g^* - \| \bm{\bar{f}_t}(\bm x, \epsilon^t)-\bm z^*\|\Big) 
\end{align*}
where 
$g^*= \min_{\bm x\in \bar{X}^n} g(\bm f(\bm x),\bm z^{*})$ and $\bm z^{*}=(z_1^*, z_2^*, \dots, z_m^*)$ is the utopian point. Following \citet[Theorem 3]{balandat2020botorch}, we need to show that there exists an integrable function $\ell:\mathbb{R}^{m}\mapsto \mathbb{R}$ such that for almost every $\epsilon^t$ and all $\bm x, \bm y \in \mathcal X \subset \mathbb{R}^d$, 
\begin{align}
    |\bar{A}(\bm x, \epsilon^t) - \bar{A}(\bm y, \epsilon^t)| \leq \ell(\epsilon^t) \|\bm x - \bm y\|.
\end{align}

Let 
\begin{align*}
    \bar{A}(\bm x, \epsilon^t) &= \max\Big( 0,\, g^* - \| \bm{\bar{f}_t}(\bm x, \epsilon^t)-\bm z^*\|\Big) \\
    &= \frac{1}{2} \Big(g^*-\| \bm{\bar{f}_t}(\bm x, \epsilon^t)-\bm z^*\| + \big|g^* - \| \bm{\bar{f}_t}(\bm x, \epsilon^t)-\bm z^*\|\big|\Big).
\end{align*}

Hence we obtain, 
\label{appdx:eq:Convergence:SplitProduct}
\begin{align*}
    &\quad \ \bigl|\bar{A}(\bm x, \epsilon^t) - \bar{A}(\bm y, \epsilon^t)\bigr| \\
    &= \Big|\frac{1}{2}\Big(\| \bm{\bar{f}_t}(\bm y, \epsilon^t)-\bm z^*\|-\| \bm{\bar{f}_t}(\bm x, \epsilon^t)-\bm z^*\|\Big) + 
       \frac{1}{2}\Big(\big|g^* - \| \bm{\bar{f}_t}(\bm x, \epsilon^t)-\bm z^*\|\big|-\big|g^* - \| \bm{\bar{f}_t}(\bm y, \epsilon^t)-\bm z^*\|\big|\Big)\Big|.
\end{align*}

Let $I_1:=\| \bm{\bar{f}_t}(\bm y, \epsilon^t)-\bm z^*\|-\| \bm{\bar{f}_t}(\bm x, \epsilon^t)-\bm z^*\|$ and $I_2:=\big|g^* - \| \bm{\bar{f}_t}(\bm x, \epsilon^t)-\bm z^*\|\big|-\big|g^* - \| \bm{\bar{f}_t}(\bm y, \epsilon^t)-\bm z^*\|\big|$. Hence
\begin{align*}
    \quad \ \bigl|\bar{A}(\bm x, \epsilon^t) - \bar{A}(\bm y, \epsilon^t)\bigr| \leq \frac{1}{2}|I_1| + \frac{1}{2}|I_2|.
\end{align*}
    
We observe that 
\begin{align*}
    |I_1| 
    &= \big| \| \bm{\bar{f}_t}(\bm y, \epsilon^t)-\bm z^*\|-\| \bm{\bar{f}_t}(\bm x, \epsilon^t)-\bm z^*\| \big| \\
    & \leq \|\bm{\bar{f}_t}(\bm y, \epsilon^t)-\bm{\bar{f}_t}(\bm x, \epsilon^t)\| \\
    &= \|\bar{\mu}(\bm y) + \bar{L}(\bm y)\epsilon^t - (\bar{\mu}(\bm x) + \bar{L}(\bm x)\epsilon^t) \| \\
    & \leq \|\bar{\mu}(\bm y) -\bar{\mu}(\bm x) \| + \|(\bar{L}(\bm y) - \bar{L}(\bm x))\epsilon^t\|.
\end{align*}
Since $\mathcal{X}$ is compact, and $\bar{\mu}$ and $\bar{L}$ have uniformly bounded gradients, they are Lipschitz.  There exist $C_{\mu_1}, C_{L_1} < \infty$ such that 
\begin{align*}
    |I_1| 
    & \leq \|\bar{\mu}(\bm y) -\bar{\mu}(\bm x) \| + \|(\bar{L}(\bm y) - \bar{L}(\bm x))\epsilon^t\|  \\
    & \leq  \ell_{I_1}(\epsilon^t)\|\bm x- \bm y\|
\end{align*}
where $\ell_{I_1}(\epsilon^t):=C_{\mu_1}+C_{L_1}\|\epsilon^t\|$. 
Furthermore, 
\begin{align*}
    |I_2| 
    &=\Big|\big|g^* - \| \bm{\bar{f}_t}(\bm x, \epsilon^t)-\bm z^*\|\big|-\big|g^* - \| \bm{\bar{f}_t}(\bm y, \epsilon^t)-\bm z^*\|\big|\Big| \\
    &\leq \big|\| \bm{\bar{f}_t}(\bm x, \epsilon^t)-\bm z^*\| - \| \bm{\bar{f}_t}(\bm y, \epsilon^t)-\bm z^*\|\big| \\
    &\leq \|\bm{\bar{f}_t}(\bm x, \epsilon^t)-\bm{\bar{f}_t}(\bm y, \epsilon^t)\| \\
    &= \|\bar{\mu}(\bm x) + \bar{L}(\bm x)\epsilon^t - (\bar{\mu}(\bm y) + \bar{L}(\bm y)\epsilon^t) \| \\
    & \leq \|\bar{\mu}(\bm x) -\bar{\mu}(\bm y) \| + \|(\bar{L}(\bm x) - \bar{L}(\bm y))\epsilon^t\|.
\end{align*}
Since $\mathcal{X}$ is compact, and $\bar{\mu}$ and $\bar{L}$ have uniformly bounded gradients, they are Lipschitz.  There exist $C_{\mu_2}, C_{L_2} < \infty$ such that 
\begin{align*}
    |I_2| 
    & \leq  \ell_{I_2}(\epsilon^t)\|\bm x- \bm y\|
\end{align*}
where $\ell_{I_2}(\epsilon^t):=C_{\mu_2}+C_{L_2}\|\epsilon^t\|$. Hence
\begin{align*}
    \quad \ \bigl|\bar{A}(\bm x, \epsilon^t) - \bar{A}(\bm y, \epsilon^t)\bigr| 
    & \leq \frac{1}{2}|I_1| + \frac{1}{2}|I_2| \\
    & \leq \frac{1}{2} \ell_{I_1}(\epsilon^t)\|\bm x- \bm y\| + \frac{1}{2} \ell_{I_2}(\epsilon^t)\|\bm x- \bm y\| \\
    & = \frac{1}{2} (\ell_{I_1}(\epsilon^t)+\ell_{I_2}(\epsilon^t)) \|\bm x- \bm y\|.
\end{align*}
Hence
\begin{align*}
    \quad \ \bigl|\bar{A}(\bm x, \epsilon^t) - \bar{A}(\bm y, \epsilon^t)\bigr| 
    & \leq \ell(\epsilon^t)\|\bm x- \bm y\|
\end{align*}
where $\ell(\epsilon^t) := (C_{\mu_1}+C_{\mu_2})+(C_{L_1}+C_{L_2})\|\epsilon^t\|$. Note that $\ell(\epsilon^t)$ is integrable because all absolute moments exist for the Gaussian distribution. Since this satisfies the criteria for Theorem 3 in~\cite{balandat2020botorch}, the theorem holds for ESPI.

\subsection{Theorem~\ref{thm:SAA_NESPI} and Its Proof}\label{appdx:sec:proof2}

\begin{theorem}
\label{thm:SAA_NESPI}
    Suppose that $\mathcal X$ is compact and that $\bm f$ has a multi-output GP prior with continuously differentiable mean and covariance functions. 
    Let $X^n=\{\bm x^t\}_{t=1}^n$ be the set of observed decision vectors in $\mathcal{D}^n$, 
    $\aNESPI^* := \max_{\bm x \in \mathcal X} \aNESPI(\bm x)$ denote the maximum of NESPI, $S^* := \argmax_{\bm x \in \mathcal X} \aNESPI(\bm x)$ denote the set of maximisers of $\aNESPI$, 
    $\hataNESPI^N(\bm x)$ denote the deterministic acquisition function via the base samples $\{\epsilon^t\}_{t=1}^N \sim \mathcal N(0,I_{(n+1)m})$. 
    Suppose that $\hat{\bm x}^*_N \in \argmax_{\bm x \in \mathcal X} \hataNESPI^N(\bm x)$, then
    
    \begin{enumerate}[label={(\arabic*)}]
        \item $\hataNESPI^N(\hat{\bm x}^*_N) \rightarrow \aNESPI^*$ a.s.,
        \item $d(\hat{\bm x}_{\!N}^*, S^*) \rightarrow 0$ a.s., where $d(\hat{\bm x}_{\!N}^*,S^*):=\inf_{\bm x\in S^*}\lVert\hat{\bm x}_{\!N}^*-\bm x\rVert$. 
    \end{enumerate}
\end{theorem}

\textit{Proof of Theorem~\ref{thm:SAA_NESPI}.}
Let $X^n:=[(\bm{x}^1)^T,\dots,(\bm x^n)^T]^T \in \mathbb{R}^{nd}$ be the $n$ observed points; 
$\bm x^{n+1} \in \mathcal{X}\subset \mathbb{R}^{d}$ be a candidate point; 
and $\epsilon^t \in \mathbb{R}^{(n+1)m}$ with $\epsilon^t \in \mathcal{N}(0,I_{(n+1)m})$. 
Let $\bm{\tilde{f}}_t(X^n,\bm x^{n+1}):=[\bm{\tilde{f}}_t(\bm x^1),\dots,\bm{\tilde{f}}_t(\bm x^{n}),\bm{\tilde{f}}_t(\bm x^{n+1})]$ denote the $t^{\text{th}}$ sample of the corresponding objectives, so that we can write the posterior via the reparameterisation trick as:
\begin{equation*}
    \bm f_t(X^n, \bm x^{n+1}, \epsilon^t) = \mu(X^n,\bm x^{n+1}) + L(X^n,\bm x^{n+1})\epsilon^t
\end{equation*}
where $\mu(X^n,\bm x^{n+1}):\mathbb{R}^{(n+1)d} \rightarrow \mathbb{R}^{(n+1)m}$ is the multi-output GP's posterior mean; $L(X^n,\bm x^{n+1}) \in \mathbb{R}^{(n+1)m \times (n+1)m}$ is a root decomposition of the multi-output GP’s posterior covariance $K\in \mathbb{R}^{(n+1)m \times (n+1)m}$. 
Let $\bm{f}^{(m)}(\bm x_i,\epsilon^t)
:= S_i
\Bigl(\mu(X^n,\bm x^{n+1}) + L(X^n,\bm x^{n+1})\epsilon^t\Bigr)$,
where $\bm{f}^{(m)}(\bm x_i,\epsilon^t) 
\in \mathbb{R}^{m}$ represents the posterior at point $\bm x_i$ and 
$S_i\in \mathbb{R}^{m \times (n+1)m}, i=1,\dots,n+1$ is the selector matrix used to extract the corresponding element for the $\bm x_i$. 
Let
\begin{align*}
    A(\bm x^{n+1}, \epsilon^t;X^n) &= \max\Big( 0,\, \hat{g^*_t}(X^n) - \| \bm{f}^{(m)}(\bm x^{n+1}, \epsilon^t)-\bm z^*\|\Big).
\end{align*}

Let $\hat{g^*_t}(X^n):= \min_{\bm x_{obs} \in \{\bm x_i\}_{i=1}^n}\|{\bm f_t^{(m)}}(\bm x_{obs}, \epsilon^t) - \bm z^*\|$, hence we obtain:
\begin{align*}
    A(\bm x^{n+1}, \epsilon^t;X^n) &= \max\Big( 0,\, \min_{\bm x_{obs} \in \{\bm x_i\}_{i=1}^n}\|\bm{f}^{(m)}(\bm x_{obs}, \epsilon^t) - \bm z^*\| - \| \bm{f}^{(m)}(\bm x^{n+1}, \epsilon^t)-\bm z^*\|\Big).
\end{align*}

Following~\citet[Theorem 3]{balandat2020botorch}, we need to show that there exists an integrable function $\ell:\mathbb{R}^{m}\mapsto \mathbb{R}$ such that for almost every $\epsilon^t$ and all $\bm x^{n+1}, \bm y^{n+1} \in \mathcal X \subset \mathbb{R}^d$, 
\begin{align}
    |A(\bm x^{n+1}, \epsilon^t;X^n) - A(\bm y^{n+1}, \epsilon^t;X^n)| \leq \ell(\epsilon^t) \|\bm x^{n+1} - \bm y^{n+1}\|.
\end{align}

Let
\begin{align*}
    \quad \  A(\bm x^{n+1}, \epsilon^t; X^n) 
    &= \max \left(0,\hat{g_t^*}(X^n)  - \| \bm{f}^{(m)}(\bm x^{n+1}, \epsilon^t)-\bm z^*\|\right) \\
    &= \frac{1}{2} \Big(\hat{g_t^*}(X^n) -\| \bm{f}^{(m)}(\bm x^{n+1}, \epsilon^t)-\bm z^*\| + \big|\hat{g_t^*}(X^n) - \| \bm{f}^{(m)}(\bm x^{n+1}, \epsilon^t)-\bm z^*\|\big|\Big).
\end{align*}

Hence we obtain
\begin{align*}
    &\quad \ \bigl|A(\bm x^{n+1}, \epsilon^t; X^n) - A(\bm y^{n+1}, \epsilon^t; X^n)\bigr| \\
    &= \Big|\frac{1}{2}\Big(\| \bm{f}^{(m)}(\bm y^{n+1}, \epsilon^t)-\bm z^*\|-\| \bm{f}^{(m)}(\bm x^{n+1}, \epsilon^t)-\bm z^*\|\Big) \\
    &+ 
       \frac{1}{2}\Big(\big|\hat{g_t^*}(X^n) - \| \bm{f}^{(m)}(\bm x^{n+1}, \epsilon^t)-\bm z^*\|\big|-\big|\hat{g_t^*}(X^n) - \| \bm{f}^{(m)}(\bm y^{n+1}, \epsilon^t)-\bm z^*\|\big|\Big)\Big|.
\end{align*}

Let $I_1':=\| \bm{f}^{(m)}(\bm y^{n+1}, \epsilon^t)-\bm z^*\|-\| \bm{f}^{(m)}(\bm x^{n+1}, \epsilon^t)-\bm z^*\|$ and $I_2':=\big|\hat{g_t^*}(X^n) - \| \bm{f}^{(m)}(\bm x^{n+1}, \epsilon^t)-\bm z^*\|\big|-\big|\hat{g_t^*}(X^n) - \| \bm{f}^{(m)}(\bm y^{n+1}, \epsilon^t)-\bm z^*\|\big|$. Hence
\begin{align*}
    \quad \ \bigl|A(\bm x^{n+1}, \epsilon^t;X^n) - A(\bm y^{n+1}, \epsilon^t;X^n)\bigr| \leq \frac{1}{2}|I_1'| + \frac{1}{2}|I_2'|.
\end{align*}
    
We observe that 
\begin{align*}
    |I_1'| 
    &= \big| \| \bm{f}^{(m)}(\bm y^{n+1}, \epsilon^t)-\bm z^*\|-\| \bm{f}^{(m)}(\bm x^{n+1}, \epsilon^t)-\bm z^*\| \big| \\
    & \leq \|\bm{f}^{(m)}(\bm y^{n+1}, \epsilon^t)-\bm{f}^{(m)}(\bm x^{n+1}, \epsilon^t)\| \\
    &= \|\mu^{(m)}(\bm y^{n+1}) + L^{(m)}(\bm y^{n+1})\epsilon^t - (\mu^{(m)}(\bm x^{n+1}) + L^{(m)}(\bm x^{n+1})\epsilon^t) \| \\
    & \leq \|\mu^{(m)}(\bm y^{n+1}) -\mu^{(m)}(\bm x^{n+1}) \| + \|(L^{(m)}(\bm y^{n+1}) - L^{(m)}(\bm x^{n+1}))\epsilon^t\| 
\end{align*}
Since $\mathcal{X}$ is compact, and $\mu^{(m)}$ and $L^{(m)}$ have uniformly bounded gradients, they are Lipschitz.  There exist $C_{\mu_1}', C_{L_1}' < \infty$ such that 
\begin{align*}
    |I_1'| 
    & \leq \|\mu^{(m)}(\bm y^{n+1}) -\mu^{(m)}(\bm x^{n+1}) \| + \|(L^{(m)}(\bm y^{n+1}) - L^{(m)}(\bm x^{n+1}))\epsilon^t\|  \\
    & \leq  \ell_{I_1'}(\epsilon^t)\|\bm x^{n+1}- \bm y^{n+1}\|.
\end{align*}
where $\ell_{I_1'}(\epsilon^t):=C_{\mu_1}'+C_{L_1}'\|\epsilon^t\|$. 
Furthermore, 
\begin{align*}
    |I_2'| 
    &=\Big|\big|\hat{g_t^*}(X^n) - \| \bm f^{(m)}(\bm x^{n+1}, \epsilon^t)-\bm z^*\|\big|-\big|\hat{g_t^*}(X^n) - \| \bm f^{(m)}(\bm y^{n+1}, \epsilon^t)-\bm z^*\|\big|\Big| \\
    &\leq \big|\| \bm{f}^{(m)}(\bm x^{n+1}, \epsilon^t)-\bm z^*\| - \| \bm{f}^{(m)}(\bm y^{n+1}, \epsilon^t)-\bm z^*\|\big| \\
    &\leq \|\bm{f}^{(m)}(\bm x^{n+1}, \epsilon^t)-\bm{f}^{(m)}(\bm y^{n+1}, \epsilon^t)\| \\
    &= \|\mu^{(m)}(\bm x^{n+1}) + L^{(m)}(\bm x^{n+1})\epsilon^t - (\mu^{(m)}(\bm y^{n+1}) + L^{(m)}(\bm y^{n+1})\epsilon^t) \| \\
    & \leq \|\mu^{(m)}(\bm x^{n+1}) -\mu^{(m)}(\bm y^{n+1}) \| + \|(L^{(m)}(\bm x^{n+1}) - L^{(m)}(\bm y^{n+1}))\epsilon^t\| 
\end{align*}
Since $\mathcal{X}$ is compact, and $\mu^{(m)}$ and $L^{(m)}$ have uniformly bounded gradients, they are Lipschitz.  There exist $C_{\mu_2}', C_{L_2}' < \infty$ such that 
\begin{align*}
    |I_2'| 
    & \leq  \ell_{I_2'}(\epsilon^t)\|\bm x^{n+1}- \bm y^{n+1}\|
\end{align*}
where $\ell_{I_2'}(\epsilon^t):=C_{\mu_2}'+C_{L_2}'\|\epsilon^t\|$. Hence
\begin{align*}
    \quad \ \bigl|A(\bm x^{n+1}, \epsilon^t) - A(\bm y^{n+1}, \epsilon^t)\bigr| 
    & \leq \frac{1}{2}|I_1'| + \frac{1}{2}|I_2'| \\
    & \leq \frac{1}{2} \ell_{I_1'}(\epsilon^t)\|\bm x^{n+1}- \bm y^{n+1}\| + \frac{1}{2} \ell_{I_2'}(\epsilon^t)\|\bm x^{n+1}- \bm y^{n+1}\| \\
    & = \frac{1}{2} (\ell_{I_1'}(\epsilon^t)+\ell_{I_2'}(\epsilon^t)) \|\bm x^{n+1}- \bm y^{n+1}\|
\end{align*}
Hence
\begin{align*}
    \quad \ \bigl|A(\bm x^{n+1}, \epsilon^t) - A(\bm y^{n+1}, \epsilon^t)\bigr| 
    & \leq \ell(\epsilon^t)\|\bm x^{n+1}- \bm y^{n+1}\|
\end{align*}
where $\ell(\epsilon^t) := (C_{\mu_1}'+C_{\mu_2}')+(C_{L_1}'+C_{L_2}')\|\epsilon^t\|$. Note that $\ell(\epsilon^t)$ is integrable because all absolute moments exist for the Gaussian distribution. Since this satisfies the criteria for Theorem 3 in~\cite{balandat2020botorch}, the theorem holds for NESPI. 
It is worth noting that Theorems~\ref{thm:SAA_ESPI} and~\ref{thm:SAA_NESPI} readily extend to the batch setting. 

It can be shown that the gradient of $\hataESPI(\bm x)$ or $\hataNESPI(\bm x)$ is an unbiased estimator of the true gradient of $\aESPI$ or $\aNESPI$, though it is not necessary for the SAA approach~\citep{daulton2021parallel}.

\section{Experiment Settings}\label{appdx:sec:setting}

\subsection{Implementation details}\label{appdx:sec:implementation}

All experiments were conducted using Python 3.12, with all methods developed on the open-source Python framework BoTorch~\citep{balandat2020botorch}, which builds on GPyTorch~\citep{gardner2018gpytorch} for Gaussian process modelling and PyTorch~\citep{paszke2019pytorch} for automatic differentiation. The computational studies were performed on a Red Hat Enterprise Linux 8.8 system, operating on a 64-bit x86 CPU architecture. The computing cluster utilised Intel Xeon Platinum 8360Y processors running at 2.40 GHz. 

The code for ParEGO, NParEGO, TS-TCH, EHVI, NEHVI, and JES is available at~\url{https://github.com/pytorch/botorch}. 
Specifically, we implement ParEGO and NParEGO based on the implementation provided by BoTorch.\footnote{\href{https://botorch.readthedocs.io/en/latest/_modules/botorch/acquisition/monte_carlo.html\#qExpectedImprovement}{EI}, 
\href{https://botorch.readthedocs.io/en/latest/_modules/botorch/acquisition/monte_carlo.html\#qNoisyExpectedImprovement}{NEI}, 
and the \href{https://botorch.org/docs/tutorials/multi_objective_bo/}{BoTorch multi-objective tutorial}.} 
We implement TS-TCH based on the implementation provided by BoTorch.\footnote{\href{https://botorch.readthedocs.io/en/latest/_modules/botorch/generation/sampling.html\#MaxPosteriorSampling}{TS} and 
\href{https://botorch.org/docs/tutorials/multi_objective_bo/}{BoTorch multi-objective tutorial}}
We implement EHVI and NEHVI based on the implementation provided by BoTorch.\footnote{\href{https://botorch.readthedocs.io/en/latest/acquisition.html\#botorch.acquisition.multi_objective.logei.qLogExpectedHypervolumeImprovement}{EHVI},
\href{https://botorch.readthedocs.io/en/latest/acquisition.html\#botorch.acquisition.multi_objective.logei.qLogNoisyExpectedHypervolumeImprovement}{NEHVI}, and
\href{https://botorch.org/docs/tutorials/multi_objective_bo/}{BoTorch multi-objective tutorial}}
We implement JES based on the implementation provided by BoTorch.\footnote{\href{https://botorch.org/docs/tutorials/information_theoretic_acquisition_functions/}{Botorch JES implementation}} 

The data and code are available at an anonymised repository for reproducibility: \url{https://anonymous.4open.science/r/SPMO-B9DE}.



\subsection{Method Details}\label{appdx:sec:method}

For all the methods, we set the same $2(d + 1)$ points from a scrambled Sobol sequence and allow a maximum of 200 evaluations, following the practice in~\cite{daulton2020differentiable,daulton2021parallel,konakovic2020diversity}. 
All the methods use $N = 128$ Monte Carlo samples. 

For ParEGO~\citep{knowles2006parego} and its noisy variant NParEGO~\citep{daulton2021parallel}, we employ random scalarisations, whereby a weight vector $\bm{w} \in \mathbb{R}^{m}$ is generated from the unit simplex. 
The augmented Tchebycheff scalarisation function is applied, defined as $g(\bm{y}) = \max_{i} (w_i y_i) + \alpha \sum_{i} (w_i y_i)$.
Log expected improvement and noisy log expected improvement are used as the acquisition functions in ParEGO and NParEGO, respectively, as recommended in~\cite{ament2023unexpected}. 
In the batch setting, $q$ distinct weight vectors are sampled, and the acquisition function is optimised sequentially for each.

\begin{table*}[!ht]\scriptsize
\centering
    \caption{Reference points of the six benchmark problems used for hypervolume computation in EHVI and NEHVI, as well as for performance evaluation. 
    Note that the referents points of the two real-world problems, i.e., car side impact design and car cab design, are set to $(1.1, ..., 1.1) \in \mathbb R^m$ in the normalised objective space following the practice in~\cite{tanabe2020easy} the utopian and nadir points are available at~\url{https://github.com/ryojitanabe/reproblems/tree/master/ideal_nadir_points}). }
    \begin{tabular}{ccc}
    \hline
    \toprule
    \textbf{Problem} & \textbf{Reference point} & \textbf{Suggested by} \\ 
    \midrule

    DTLZ1  & $(400.0, ..., 400.0) \in \mathbb R^m$ & ~\cite{balandat2020botorch,chugh2020scalarizing}\\
    DTLZ2 & $(1.1, \dots, 1.1) \in \mathbb R^m$ & ~\cite{balandat2020botorch,daulton2020differentiable,ishibuchi2018specify}\\
    Inverted DTLZ1 & $(400.0,\dots, 400.0) \in \mathbb R^m$ &~\cite{chugh2020scalarizing,ishibuchi2018specify} \\
    Inverted DTLZ2 & $(1.1, \dots, 1.1) \in \mathbb R^m$ &~\cite{ishibuchi2018specify}  \\
    Convex DTLZ2 & $(1.1,\dots, 1.1) \in \mathbb R^m$  &~\cite{ishibuchi2018specify} \\
    Scaled DTLZ2 & $(1.1*2^0,\dots, 1.1*2^{m}) \in \mathbb R^m$  &~\cite{ishibuchi2018specify}  \\
    DTLZ3 & $(10000.0,\dots, 10000.0) \in \mathbb R^m$  &~\cite{balandat2020botorch}  \\
    DTLZ4 & $(1.1,\dots, 1.1) \in \mathbb R^m$  &~\cite{balandat2020botorch}  \\

    DTLZ5 & $(10.0,\dots, 10.0) \in \mathbb R^m$  &~\cite{balandat2020botorch}  \\

    DTLZ6 & $(10.0,\dots, 10.0) \in \mathbb R^m$  &~\cite{balandat2020botorch}  \\

    DTLZ7 & $(15.0,\dots, 15.0) \in \mathbb R^m$  &~\cite{balandat2020botorch}  \\

    
    \bottomrule
    \end{tabular}
    \label{tbl:HV}
\end{table*}

For TS-TCH~\citep{paria2020flexible}, similarly to ParEGO, we employ random scalarisations by using the augmented Tchebycheff scalarisation function. 
After converting to single-objective optimisation problem, Thompson sampling is used as the acquisition function. 
We draw a sample from the joint posterior over a discrete set of 1000$d$ points sampled from a scrambled Sobol sequence, suggested by~\cite{daulton2020differentiable}. 
In the batch setting, $q$ distinct weight vectors are sampled, and the acquisition function is optimised sequentially for each.

For EHVI~\citep{daulton2020differentiable} and its noisy variant NEHVI~\citep{daulton2021parallel}, the reference point is predefined, following the practice in~\cite{daulton2020differentiable,yang2019multi}. 
Table~\ref{tbl:HV} lists the reference points used for each problem in the experimental evaluation. 
The logarithmic variants of both acquisition functions are employed, as recommended in~\cite{ament2023unexpected}. 
In the batch setting, the sequential greedy optimisation strategy is adopted. 
It is worth noting that in this study, EHVI and NEHVI are evaluated only on problems with 3 and 5 objectives, as the acquisition optimisation wall time becomes prohibitively high when the number of objectives increases to 10 (see Table~\ref{tbl:wall_time}). 

For C-EHVI~\citep{gaudrie2018budgeted,gaudrie2020targeting}, the Kalai-Smorodinsky equilibrium is used to determine the reference point of HV. 
The logarithmic variant of the acquisition functions is employed, as recommended in~\cite{ament2023unexpected}.

For joint entropy search (JES), we use $S = 10$ Monte Carlo samples and $p = 10$ number of Pareto optimal points, according to~\cite{tu2022joint}. 
In the batch setting, the sequential greedy optimisation strategy is adopted.

For all the problems, we normalise the input variables and standardise the objective values before training Gaussian processes. 
We assume an independent surrogate model for each objective, using a constant mean function and a Matérn 5/2 ARD kernel. 
We optimise all acquisition functions by using the L-BFGS-B, with up to 200 iterations. 


\subsection{Problem Details}\label{appendix:subsec:problems}

The details of the benchmark problems and real-world problems are given in the following. 


\paragraph{DTLZ1.} DTLZ1 is a scalable benchmark problem from~\cite{deb2005scalable}, which is defined as:

\begin{align*}
\quad f_1(\bm{x}) &= \frac{1}{2} x_1 x_2 \cdots x_{m-1} (1 + g(\bm{x}_m)), \\
\quad f_2(\bm{x}) &= \frac{1}{2} x_1 x_2 \cdots (1 - x_{m-1}) (1 + g(\bm{x}_m)), \\
&\vdots \\
\quad f_{m-1}(\bm{x}) &= \frac{1}{2} x_1 (1 - x_2)(1 + g(\bm{x}_m)), \\
\quad f_m(\bm{x}) &= \frac{1}{2} (1 - x_1)(1 + g(\bm{x}_m)), \\
\text{s.t.} \quad &0 \leq x_i \leq 1, \quad \text{for } i = 1, 2, \dots, d.
\end{align*}

where $g(\bm{x}_m) = 100 \left( |\bm{x}_m| + \sum_{x_i \in \bm{x}_m} (x_i - 0.5)^2 - \cos(20\pi(x_i - 0.5)) \right)$ and $\bm{x}_m$ is the last $d - m + 1$ variables.

\paragraph{DTLZ2.} DTLZ2 is a scalable benchmark problem from~\cite{deb2005scalable}, which is defined as:

\begin{align*}
 \quad f_1(\bm{x}) &= (1 + g(\bm{x}_m)) \cos(x_1 \pi/2) \cdots \cos(x_{m-2} \pi/2) \cos(x_{m-1} \pi/2), \\
 \quad f_2(\bm{x}) &= (1 + g(\bm{x}_m)) \cos(x_1 \pi/2) \cdots \cos(x_{m-2} \pi/2) \sin(x_{m-1} \pi/2), \\
 \quad f_3(\bm{x}) &= (1 + g(\bm{x}_m)) \cos(x_1 \pi/2) \cdots \sin(x_{m-2} \pi/2), \\
&\vdots \\
 \quad f_m(\bm{x}) &= (1 + g(\bm{x}_m)) \sin(x_1 \pi/2), \\
 s.t. \quad & 0 \leq x_i \leq 1, \quad \text{for } i = 1, 2, \dots, d.
\end{align*}

where $g(\bm{x}_m) = \sum_{x_i \in \bm{x}_m} (x_i - 0.5)^2$ and $\bm{x}_m$ is the last $d - m + 1$ variables. 

\paragraph{Inverted DTLZ1.} Inverted DTLZ1 is a variant of DTLZ1~\citep{jain2013evolutionary}, which is defined as:
\begin{align*}
    f_i(\bm x) = 0.5 \cdot (1 + g(\bm x_m)) - f_i^{\text{DTLZ1}}(\bm x), \quad  i=1,\dots,m
\end{align*}
where $g(\bm x_m)$ is the same function as used in DTLZ1, $f_i^{\text{DTLZ1}}(\bm x)$ denotes the $i$th objective of the original DTLZ1 formulation. 

\paragraph{Inverted DTLZ2.} 
Inverted DTLZ2 is a variant of DTLZ2~\citep{jain2013improved}, which is defined as: 
\begin{align*}
    f_i(\bm x) = 1 + g(\bm x_m) - f_i^{\text{DTLZ2}}(\bm x), \quad i=1,\dots,m
\end{align*}
where $g(\bm x_m)$ is the same function as used in DTLZ2, $f_i^{\text{DTLZ2}}(\bm x)$ denotes the $i$th objective of the original DTLZ2 formulation. 

\paragraph{Convex DTLZ2.} Convex DTLZ2 is a variant of DTLZ2~\citep{deb2013evolutionary}, which is defined as: 
\begin{align*}
    f_i(\bm x) &= (f_i^{\text{DTLZ2}}(\bm x))^4, \quad i=1,\dots,m-1 \\
    f_m(\bm x) &= (f_m^{\text{DTLZ2}}(\bm x))^2
\end{align*}
where $f_i^{\text{DTLZ2}}(\bm x)$ denotes the $i$th objective of the original DTLZ2 formulation. 
This problem convert the original concave problem to convex problem.

\paragraph{Scaled DTLZ2.} Scaled DTLZ2 is a variant of DTLZ2, which is defined as: 
\begin{align*}
    f_i(\bm x) = 2^{i-1}\cdot f_i^{\text{DTLZ2}}(\bm x) , \quad i=1,\dots,m
\end{align*}
where $f_i^{\text{DTLZ2}}(\bm x)$ denotes the $i$th objective of the original DTLZ2 formulation. This benchmark problem is used to see whether an algorithm can deal with problems with different scales of different objectives. 

\paragraph{DTLZ3--DTLZ7.} DTLZ3--DTLZ7 are scalable multi-objective benchmark problems. Their mathematical formulations are provided in~\cite{deb2005scalable}.

According to the original paper~\citep{deb2005scalable}, the dimensionality $d$ of DTLZ1 and its variant (i.e., inverted DTLZ1) is $m+4$, the dimensionality $d$ of DTLZ2-6 and their variants (i.e., inverted DTLZ2, convex DTLZ2, and scaled DTLZ2) is $m+9$, and the dimensionality $d$ of DTLZ7 is $m+19$.

\paragraph{Car Side Impact Design.} 

The car side-impact problem aims to minimise vehicle weight while satisfying safety constraints related to occupant injury and structural response~\citep{jain2013evolutionary}. 
It involves $m=4$ objectives with $d=7$ variables, which are based on a surrogate model that is fit to data collected from a simulator. 
The mathematical formulations are given as follows:

\begin{align*}
f_1(\bm{x}) &= 1.98 + 4.9x_1 + 6.67x_2 + 6.98x_3 + 4.01x_4 + 1.78x_5 + 10^{-5}x_6 + 2.73x_7 \\
f_2(\bm{x}) &= 4.72 - 0.5x_4 - 0.19x_2x_3 \\
f_3(\bm{x}) &= 0.5 \left( V_{\text{MBP}}(\bm{x}) + V_{\text{FD}}(\bm{x}) \right) \\
f_4(\bm{x}) &= -\sum_{i=1}^{10} \max\left(g_i(\bm{x}), 0\right)
\end{align*}

\noindent
where the constraint functions \( g_i(\bm{x}) \) are defined as:
\begin{align*}
g_1(\bm{x}) &= 1 - 1.16 + 0.3717x_2x_4 + 0.0092928x_3 \\
g_2(\bm{x}) &= 0.32 - 0.261 + 0.0159x_1x_2 + 0.06486x_1 + 0.019x_2x_7 - 0.0144x_3x_5 - 0.0154464x_6 \\
g_3(\bm{x}) &= 0.32 - 0.214 - 0.00817x_5 + 0.045195x_1 + 0.0135168x_1 - 0.03099x_2x_6 \\
&\quad + 0.018x_2x_7 - 0.007176x_3 - 0.023223x_3 + 0.00364x_5x_6 + 0.018x_2^2 \\
g_4(\bm{x}) &= 0.32 - 0.74 + 0.61x_2 + 0.031296x_3 + 0.031872x_7 - 0.227x_2^2 \\
g_5(\bm{x}) &= 32 - 28.98 - 3.818x_3 + 4.2x_1x_2 - 1.27296x_6 + 2.68065x_7 \\
g_6(\bm{x}) &= 32 - 33.86 - 2.95x_3 + 5.057x_1x_2 + 3.795x_2 + 3.4431x_7 - 1.45728 \\
g_7(\bm{x}) &= 32 - 46.36 + 9.9x_2 + 4.4505x_1 \\
g_8(\bm{x}) &= 4 - f_2(\bm{x}) \\
g_9(\bm{x}) &= 9.9 - V_{\text{MBP}}(\bm{x}) \\
g_{10}(\bm{x}) &= 15.7 - V_{\text{FD}}(\bm{x})
\end{align*}

\noindent
with volume terms defined as:
\begin{align*}
V_{\text{MBP}}(\bm{x}) &= 10.58 - 0.674x_1x_2 - 0.67275x_2 \\
V_{\text{FD}}(\bm{x}) &= 16.45 - 0.489x_3x_7 - 0.8435x_6x_7.
\end{align*}

\noindent
The search space is defined as:
\[
\begin{aligned}
x_1 &\in [0.5, 1.5], \\
x_2 &\in [0.45, 1.35], \\
x_3, x_4 &\in [0.5, 1.5], \\
x_5 &\in [0.875, 2.625], \\
x_6, x_7 &\in [0.4, 1.2].
\end{aligned}
\]


\paragraph{Car Cab Design.} This vehicle performance optimisation problem involves $m=9$ objectives with $d=7$ variables, relating to aspects such as car roominess, fuel economy, acceleration time, and road noise at various speeds~\citep{deb2013evolutionary}. 
\noindent
The problem includes 7 decision variables and 4 stochastic variables, which are based on a surrogate model that is fit to data collected from a simulator, defined as:

\begin{align*}
f_1(\bm{x}) &= 1.98 + 4.9 x_1 + 6.67 x_2 + 6.98 x_3 + 4.01 x_4 + 1.75 x_5 + 10^{-5} x_6 + 2.73 x_7 \\[1ex]
f_2(\bm{x}) &= \left[ {1.16 - 0.3717 x_2 x_4 - 0.00931 x_2 x_{10} - 0.484 x_3 x_9 + 0.01343 x_6 x_{10}} \right]_+ \\[1ex]
f_3(\bm{x}) &= 
\left[
\frac{1}{0.32} \left(
\begin{aligned}
&0.261 - 0.0159 x_1 x_2 - 0.188 x_1 x_8 - 0.019 x_2 x_7 + 0.0144 x_3 x_5 + 0.8757 x_5 x_{10} \\
&+ 0.08045 x_6 x_9 + 0.00139 x_8 x_{11} + 0.00001575 x_{10} x_{11}
\end{aligned}
 \right) \right]_+ \\[1ex]
f_4(\bm{x}) &= 
\left[ \frac{1}{0.32} \left(
\begin{aligned}
&0.214 + 0.00817 x_5 - 0.131 x_1 x_8 - 0.0704 x_1 x_9 + 0.03099 x_2 x_6 - 0.018 x_2 x_7 \\
&+ 0.0208 x_3 x_8 + 0.121 x_3 x_9 - 0.00364 x_5 x_6 + 0.0007715 x_5 x_{10} \\
&- 0.0005354 x_6 x_{10} + 0.00121 x_8 x_{11} + 0.00184 x_9 x_{10} - 0.018 x_2^2
\end{aligned}
 \right) \right]_+ \\[1ex]
f_5(\bm{x}) &= \left[ \frac{0.74 - 0.61 x_2 - 0.163 x_3 x_8 + 0.001232 x_3 x_{10} - 0.166 x_7 x_9 + 0.227 x_2^2}{0.32} \right]_+ \\[1ex]
f_6(\bm{x}) &= \left[ \frac{1}{32} \cdot \frac{1}{3} \left(
\begin{aligned}
&28.98 + 3.818 x_3 - 4.2 x_1 x_2 + 0.0207 x_5 x_{10} + 6.63 x_6 x_9 - 7.77 x_7 x_8 + 0.32 x_9 x_{10} \\
&+ 33.86 + 2.95 x_3 + 0.1792 x_{10} - 5.057 x_1 x_2 - 11 x_2 x_8 - 0.0215 x_5 x_{10} - 9.98 x_7 x_8 \\
& + 22 x_8 x_9 + 46.36 - 9.9 x_2 - 12.9 x_1 x_8 + 0.1107 x_3 x_{10}
\end{aligned}
\right) \right]_+ \\[1ex]
f_7(\bm{x}) &= \left[ \frac{4.72 - 0.5 x_4 - 0.19 x_2 x_3 - 0.0122 x_4 x_{10} + 0.009325 x_6 x_{10} + 0.000191 x_{11}^2}{4.0} \right]_+ \\[1ex]
f_8(\bm{x}) &= \left[ \frac{10.58 - 0.674 x_1 x_2 - 1.95 x_2 x_8 + 0.02054 x_3 x_{10} - 0.0198 x_4 x_{10} + 0.028 x_6 x_{10}}{9.9} \right]_+ \\[1ex]
f_9(\bm{x}) &= \left[ \frac{16.45 - 0.489 x_3 x_7 - 0.843 x_5 x_6 + 0.0432 x_9 x_{10} - 0.0556 x_9 x_{11} - 0.000786 x_{11}^2}{15.7} \right]_+
\end{align*}

where $[\cdot]_+$ denotes $\max (0, \cdot)$ and the search space is defined as: 
\[
\begin{aligned}
x_1 &\in [0.5,\ 1.5], \\
x_2 &\in [0.45,\ 1.35], \\
x_3,\ x_4 &\in [0.5,\ 1.5], \\
x_5 &\in [0.875,\ 2.625], \\
x_6,\ x_7 &\in [0.4,\ 1.2].
\end{aligned}
\]

The four stochastic variables are defined as:
\[
\begin{aligned}
x_8 &\sim \mathcal{N}(0.345,\ 0.006^2), \\
x_9 &\sim \mathcal{N}(0.192,\ 0.006^2), \\
x_{10},\ x_{11} &\sim \mathcal{N}(0,\ 10^2).
\end{aligned}
\]


\clearpage
\newpage
\section{Additional Experimental Results}

\subsection{Noiseless Cases}\label{appendix:sec:noiseless}

In this section, we present the results on the six noiseless problems (i.e., DTLZ1 and DTLZ2 along with their four variants) with 3 and 10 objectives. 
Tables~\ref{tbl:Dist_M3},~\ref{tbl:HV_SP_M3} and~\ref{tbl:HV_M3} show the distance-based metric (log distance), the HV of the best solution (in terms of its HV value) and the HV of all evaluated solutions obtained by the SPMO and the peer methods on the six noiseless problems, respectively. 
Figures~\ref{fig:violin_dist_M3},~\ref{fig:violin_single_HV_M3} and~\ref{fig:violin_all_HV_M3}  present the violin plots, illustrating the distributions of the corresponding results reported in Tables~\ref{tbl:Dist_M3},~\ref{tbl:HV_SP_M3}, and~\ref{tbl:HV_M3}, respectively. 
In addition, Figure~\ref{fig:convergence_distance_M3} presents the trajectories of the distance metric obtained by each method on the noiseless problems with 3 and 10 objectives. 
We also give the results of the problems DTLZ3--DLTZ7 with 5 objectives, shown in Tables~\ref{tbl:Dist_M5_dtlz3_7},~\ref{tbl:HV_SP_M5_dtlz3_7} and~\ref{tbl:HV_M5_dtlz3_7}.

\begin{table*}[!ht]
    \centering
    \caption{Results of the distance-based metric (log distance) obtained by the SPMO and the peer methods on the noiseless problems with 3 objectives (\textbf{top}) and 10 objectives (\textbf{bottom}) on 30 independent runs. The method with the best mean is highlighted in bold. The symbols ``$+$'', ``$\sim$'' and ``$-$'' indicate that the method is statistically worse than, equivalent to and better than our SPMO, respectively.}
    \resizebox{\textwidth}{!}{%
    \begin{tabular}{lllllllc}
    \toprule
    \bfseries Method
    & \multicolumn{1}{c}{\bfseries DTLZ1 (3)} 
    & \multicolumn{1}{c}{\bfseries DTLZ2 (3)} 
    & \multicolumn{1}{c}{\bfseries Inverted DTLZ1 (3)} 
    & \multicolumn{1}{c}{\bfseries Inverted DTLZ2 (3)} 
    & \multicolumn{1}{c}{\bfseries Convex DTLZ2 (3)} 
    & \multicolumn{1}{c}{\bfseries Scaled DTLZ2 (3)} 
    & {\bfseries Sum up} \\ 
    \bfseries 
    & \multicolumn{1}{c}{Mean (Std)}
    & \multicolumn{1}{c}{Mean (Std)}
    & \multicolumn{1}{c}{Mean (Std)}
    & \multicolumn{1}{c}{Mean (Std)}
    & \multicolumn{1}{c}{Mean (Std)}
    & \multicolumn{1}{c}{Mean (Std)}
    & \multicolumn{1}{c}{$+$/$\sim$/$-$}  \\ \midrule
    \bfseries Sobol & 3.8e+0 (3.3e--1)$^+$ & 2.3e--1 (5.1e--2)$^+$ & 4.4e+0 (4.0e--1)$^+$ & 5.7e--2 (5.3e--2)$^+$ & -7.7e--2 (1.7e--1)$^+$ & 2.2e--1 (4.8e--2)$^+$ & \bfseries 6/ 0/ 0\\
    \bfseries ParEGO & 3.7e+0 (2.0e--1)$^+$ & 4.7e--3 (3.1e--3)$^+$ & 3.5e+0 (5.7e--1)$^+$ & -3.0e--1 (7.7e--3)$^+$ & -9.3e--1 (1.2e--1)$^+$ & 5.2e--3 (4.7e--3)$^+$ & \bfseries 6/ 0/ 0\\
    \bfseries TS-TCH & 3.9e+0 (2.6e--1)$^+$ & 1.4e--1 (7.1e--2)$^+$ & 4.3e+0 (3.6e--1)$^+$ & -1.1e--1 (2.4e--2)$^+$ & -3.4e--1 (1.3e--1)$^+$ & 1.5e--1 (3.6e--2)$^+$ & \bfseries 6/ 0/ 0\\
    \bfseries EHVI & 3.6e+0 (1.4e--1)$^+$ & 4.6e--3 (2.5e--3)$^+$ & 4.0e+0 (2.1e--1)$^+$ & -3.0e--1 (4.3e--3)$^+$ & -9.1e--1 (1.1e--1)$^+$ & 1.8e--2 (6.0e--2)$^+$ & \bfseries 6/ 0/ 0\\
    \bfseries C-EHVI & 3.6e+0 (1.2e--1)$^+$ & 1.6e--3 (1.7e--3)$^+$ & 4.0e+0 (2.6e--1)$^+$ & -2.8e--1 (2.9e--2)$^+$ & -4.5e--1 (2.4e--1)$^+$ & 3.2e--3 (5.5e--3)$^+$ & \bfseries 6/ 0/ 0\\
    \bfseries JES & 3.5e+0 (1.1e--1)$^+$ & 5.5e--3 (4.2e--3)$^+$ & 4.1e+0 (1.7e--1)$^+$ & -3.0e--1 (4.8e--3)$^+$ & -9.1e--1 (9.0e--2)$^+$ & 9.6e--3 (8.7e--3)$^+$ & \bfseries 6/ 0/ 0\\
    \bfseries SPMO & \best {3.3e+0} (\best {2.9e--1})  & \best {2.2e--4} (\best {1.2e--4})  & \best {2.8e+0} (\best {6.6e--1})  & \best {-3.1e--1} (\best {1.3e--5})  & \best {-1.2e+0} (\best {1.4e--2})  & \best {3.8e--5} (\best {1.6e--5})  & \\
    
    \toprule
    \bfseries Method
    & \multicolumn{1}{c}{\bfseries DTLZ1 (10)} 
    & \multicolumn{1}{c}{\bfseries DTLZ2 (10)} 
    & \multicolumn{1}{c}{\bfseries Inverted DTLZ1 (10)} 
    & \multicolumn{1}{c}{\bfseries Inverted DTLZ2 (10)} 
    & \multicolumn{1}{c}{\bfseries Convex DTLZ2 (10)} 
    & \multicolumn{1}{c}{\bfseries Scaled DTLZ2 (10)} 
    & {\bfseries Sum up} \\ 
    & \multicolumn{1}{c}{Mean (Std)}
    & \multicolumn{1}{c}{Mean (Std)}
    & \multicolumn{1}{c}{Mean (Std)}
    & \multicolumn{1}{c}{Mean (Std)}
    & \multicolumn{1}{c}{Mean (Std)}
    & \multicolumn{1}{c}{Mean (Std)}
    & \multicolumn{1}{c}{$+$/$\sim$/$-$}  \\ \midrule
    \bfseries Sobol & 3.8e+0 (2.4e--1)$^+$ & 2.4e--1 (4.1e--2)$^+$ & 5.2e+0 (2.5e--1)$^+$ & 1.2e+0 (5.4e--2)$^+$ & -4.5e--1 (2.2e--1)$^+$ & 2.4e--1 (5.1e--2)$^+$ & \bfseries 6/ 0/ 0\\
    \bfseries ParEGO & 3.7e+0 (2.6e--1)$^+$ & 1.2e--1 (8.3e--2)$^+$ & 3.5e+0 (4.7e--1)$^\sim$ & 8.6e--1 (1.9e--2)$^+$ & -1.8e+0 (2.3e--1)$^+$ & 1.1e--1 (6.3e--2)$^+$ & \bfseries 5/ 1/ 0\\
    \bfseries TS-TCH & 3.8e+0 (3.9e--1)$^+$ & 2.1e--1 (2.3e--2)$^+$ & 5.3e+0 (3.7e--1)$^+$ & 1.0e+0 (1.4e--2)$^+$ & -6.7e--1 (2.5e--1)$^+$ & 2.1e--1 (3.6e--2)$^+$ & \bfseries 6/ 0/ 0\\
    \bfseries C-EHVI & 3.7e+0 (1.5e--1)$^+$ & 6.4e--3 (5.1e--3)$^+$ & 3.9e+0 (5.1e--1)$^+$ & 8.4e--1 (1.6e--2)$^+$ & -5.5e--1 (3.2e--1)$^+$ & 5.5e--2 (7.5e--2)$^+$ & \bfseries 6/ 0/ 0\\

    \bfseries JES & 3.4e+0 (2.1e--1)$^+$ & 1.5e--1 (6.6e--2)$^+$ & 4.5e+0 (5.3e--1)$^+$ & 8.6e--1 (2.0e--2)$^+$ & -1.7e+0 (3.0e--1)$^+$ & 1.2e--1 (8.5e--2)$^+$ & \bfseries 6/ 0/ 0\\
    \bfseries SPMO & \best {2.8e+0} (\best {5.5e--1})  & \best {1.3e--3} (\best {2.6e--3})  & \best {3.4e+0} (\best {5.1e--1})  & \best {7.8e--1} (\best {2.4e--2})  & \best {-3.2e+0} (\best {1.5e--1})  & \best {1.7e--2} (\best {3.7e--2})  & \\
    \bottomrule
    
    \end{tabular}
    }
    \label{tbl:Dist_M3}
\end{table*}
\FloatBarrier 


\begin{table*}[!ht]
    \centering
    \caption{The HV of the best solution (in terms of its HV value) obtained by the proposed SPMO and the peer methods on the noiseless problems with 3 objectives (\textbf{top}) and 10 objectives  (\textbf{bottom}) on 30 independent runs. 
    The method with the best mean is highlighted in bold. The symbols ``$+$'', ``$\sim$'' and ``$-$'' indicate that the method is statistically worse than, equivalent to and better than our SPMO, respectively.
    } 
    \resizebox{\textwidth}{!}{%
    \begin{tabular}{lllllllc}
    \toprule
    \bfseries Method
    & \multicolumn{1}{c}{\bfseries DTLZ1 (3)} 
    & \multicolumn{1}{c}{\bfseries DTLZ2 (3)} 
    & \multicolumn{1}{c}{\bfseries Inverted DTLZ1 (3)} 
    & \multicolumn{1}{c}{\bfseries Inverted DTLZ2 (3)} 
    & \multicolumn{1}{c}{\bfseries Convex DTLZ2 (3)} 
    & \multicolumn{1}{c}{\bfseries Scaled DTLZ2 (3)} 
    & {\bfseries Sum up} \\ 
    & \multicolumn{1}{c}{Mean (Std)}
    & \multicolumn{1}{c}{Mean (Std)}
    & \multicolumn{1}{c}{Mean (Std)}
    & \multicolumn{1}{c}{Mean (Std)}
    & \multicolumn{1}{c}{Mean (Std)}
    & \multicolumn{1}{c}{Mean (Std)}
    & \multicolumn{1}{c}{$+$/$\sim$/$-$}  \\ \midrule
    \bfseries Sobol & 5.3e+7 (3.3e+6)$^+$ & 3.3e--2 (2.1e--2)$^+$ & 4.4e+7 (5.7e+6)$^+$ & 1.1e--1 (2.5e--2)$^+$ & 1.9e--1 (1.1e--1)$^+$ & 3.7e--2 (1.8e--2)$^+$ & \bfseries 6/ 0/ 0\\
    \bfseries ParEGO & 5.6e+7 (1.5e+6)$^+$ & 1.6e--1 (7.0e--3)$^+$ & 5.5e+7 (4.2e+6)$^+$ & 3.1e--1 (3.7e--3)$^+$ & 7.2e--1 (5.1e--2)$^+$ & \best {1.6e--1} (\best {9.5e--3})$^-$ & \bfseries 5/ 0/ 1\\
    \bfseries TS-TCH & 5.3e+7 (3.1e+6)$^+$ & 5.0e--2 (3.1e--2)$^+$ & 4.6e+7 (5.2e+6)$^+$ & 2.0e--1 (1.2e--2)$^+$ & 3.7e--1 (8.4e--2)$^+$ & 4.8e--2 (2.5e--2)$^+$ & \bfseries 6/ 0/ 0\\
    \bfseries EHVI & 5.7e+7 (7.7e+5)$^+$ & 1.6e--1 (4.6e--3)$^+$ & 5.1e+7 (2.1e+6)$^+$ & 3.1e--1 (2.2e--3)$^+$ & 6.8e--1 (4.6e--2)$^+$ & 1.3e--1 (3.4e--2)$^-$ & \bfseries 5/ 0/ 1\\
    \bfseries C-EHVI & 5.7e+7 (8.3e+5)$^+$ & 1.3e--1 (1.3e--2)$^+$ & 5.1e+7 (2.4e+6)$^+$ & 3.0e--1 (1.4e--2)$^+$ & 4.5e--1 (1.4e--1)$^+$ & 1.4e--1 (1.7e--2)$^-$ & \bfseries 5/ 0/ 1\\
    \bfseries JES & 5.6e+7 (8.2e+5)$^+$ & 1.5e--1 (8.3e--3)$^+$ & 5.0e+7 (2.0e+6)$^+$ & 3.0e--1 (2.4e--3)$^+$ & 7.1e--1 (3.6e--2)$^+$ & 1.5e--1 (8.1e--3)$^-$ & \bfseries 5/ 0/ 1\\
    \bfseries SPMO & \best {5.8e+7} (\best {1.7e+6})  & \best {1.7e--1} (\best {5.3e--3})  & \best {5.9e+7} (\best {3.5e+6})  & \best {3.1e--1} (\best {6.4e--6})  & \best {7.9e--1} (\best {5.1e--3})  & 1.2e--1 (7.5e--3)  & \\

    \toprule
    \bfseries Method
    & \multicolumn{1}{c}{\bfseries DTLZ1 (10)} 
    & \multicolumn{1}{c}{\bfseries DTLZ2 (10)} 
    & \multicolumn{1}{c}{\bfseries Inverted DTLZ1 (10)} 
    & \multicolumn{1}{c}{\bfseries Inverted DTLZ2 (10)} 
    & \multicolumn{1}{c}{\bfseries Convex DTLZ2 (10)} 
    & \multicolumn{1}{c}{\bfseries Scaled DTLZ2 (10)} 
    & {\bfseries Sum up} \\ 
    & \multicolumn{1}{c }{Mean (Std)}
    & \multicolumn{1}{c }{Mean (Std)}
    & \multicolumn{1}{c }{Mean (Std)}
    & \multicolumn{1}{c }{Mean (Std)}
    & \multicolumn{1}{c }{Mean (Std)}
    & \multicolumn{1}{c }{Mean (Std)}
    & \multicolumn{1}{c}{$+$/$\sim$/$-$}  \\ \midrule

    \bfseries Sobol & 8.7e+25 (3.8e+24)$^+$ & 4.7e--2 (2.5e--2)$^+$ & 2.4e+25 (9.4e+24)$^+$ & 5.9e-11 (3.2e-10)$^+$ & 7.7e--1 (2.4e--1)$^+$ & 5.0e--2 (2.1e--2)$^+$ & \bfseries 6/ 0/ 0\\
    \bfseries ParEGO & 9.2e+25 (3.3e+24)$^+$ & 1.2e--1 (9.6e--2)$^+$ & 7.8e+25 (1.1e+25)$^\sim$ & 2.2e--5 (1.4e--5)$^+$ & 1.9e+0 (1.0e--1)$^+$ & 1.3e--1 (8.3e--2)$^+$ & \bfseries 5/ 1/ 0\\
    \bfseries TS-TCH & 8.8e+25 (5.1e+24)$^+$ & 4.3e--2 (4.0e--2)$^+$ & 2.2e+25 (1.3e+25)$^+$ & 1.6e--8 (2.6e--8)$^+$ & 1.0e+0 (2.5e--1)$^+$ & 4.1e--2 (2.5e--2)$^+$ & \bfseries 6/ 0/ 0\\
    \bfseries C-EHVI & 9.2e+25 (1.0e+24)$^+$ & 2.6e--1 (3.2e--2)$^+$ & 7.0e+25 (1.3e+25)$^+$ & 3.4e--5 (1.4e--5)$^+$ & 9.1e--1 (3.3e--1)$^+$ & 2.0e--1 (9.0e--2)$^+$ & \bfseries 6/ 0/ 0\\
    \bfseries JES & 9.3e+25 (1.9e+24)$^+$ & 9.9e--2 (7.2e--2)$^+$ & 4.9e+25 (1.8e+25)$^+$ & 2.1e--5 (1.1e--5)$^+$ & 1.8e+0 (1.3e--1)$^+$ & 1.3e--1 (9.7e--2)$^+$ & \bfseries 6/ 0/ 0\\
    \bfseries SPMO & \best {9.7e+25} (\best {3.6e+24})  & \best {2.9e--1} (\best {3.4e--2})  & \best {8.2e+25} (\best {9.3e+24})  & \best {1.3e--4} (\best {5.1e--5})  & \best {2.3e+0} (\best {3.5e--2})  & \best {2.6e--1} (\best {4.6e--2})  & \\
    \bottomrule

    \end{tabular}
    }
    \label{tbl:HV_SP_M3}
\end{table*}
\FloatBarrier 

\begin{table*}[!ht]
    \centering
    \caption{The HV of all the solutions obtained by the proposed SPMO and the peer methods on the noiseless problems with 3 objectives (\textbf{top}) and 10 objectives (\textbf{bottom}) on 30 independent runs, respectively. 
    The method with the best mean is highlighted in bold. The symbols ``$+$'', ``$\sim$'', and ``$-$'' indicate that a method is statistically worse than, equivalent to, and better than SPMO, respectively.}
    \resizebox{\textwidth}{!}{%
    \begin{tabular}{lllllllc}
    \toprule
    \bfseries Method
    & \multicolumn{1}{c}{\bfseries DTLZ1 (3)} 
    & \multicolumn{1}{c}{\bfseries DTLZ2 (3)} 
    & \multicolumn{1}{c}{\bfseries Inverted DTLZ1 (3)} 
    & \multicolumn{1}{c}{\bfseries Inverted DTLZ2 (3)} 
    & \multicolumn{1}{c}{\bfseries Convex DTLZ2 (3)} 
    & \multicolumn{1}{c}{\bfseries Scaled DTLZ2 (3)} 
    & {\bfseries Sum up} \\ 
    & \multicolumn{1}{c }{Mean (Std)}
    & \multicolumn{1}{c }{Mean (Std)}
    & \multicolumn{1}{c }{Mean (Std)}
    & \multicolumn{1}{c }{Mean (Std)}
    & \multicolumn{1}{c }{Mean (Std)}
    & \multicolumn{1}{c }{Mean (Std)}
    & \multicolumn{1}{c}{$+$/$\sim$/$-$}  \\ \midrule
    \bfseries Sobol & 6.3e+7 (2.8e+5)$^\sim$ & 4.9e--2 (2.7e--2)$^+$ & 5.2e+7 (2.9e+6)$^+$ & 2.1e--1 (2.6e--2)$^+$ & 2.5e--1 (1.3e--1)$^+$ & 6.2e--2 (2.5e--2)$^+$ & \bfseries 5/ 1/ 0\\
    \bfseries ParEGO & 6.3e+7 (1.1e+6)$^-$ & 5.4e--1 (5.2e--2)$^-$ & 5.8e+7 (2.6e+6)$^+$ & 6.7e--1 (1.1e--2)$^-$ & 1.0e+0 (9.0e--2)$^\sim$ & \best {5.4e--1} (\best {7.5e--2})$^-$ & \bfseries 1/ 1/ 4\\
    \bfseries TS-TCH & 6.2e+7 (7.8e+5)$^\sim$ & 7.5e--2 (4.2e--2)$^+$ & 5.2e+7 (3.3e+6)$^+$ & 4.1e--1 (1.2e--2)$^-$ & 5.4e--1 (1.1e--1)$^+$ & 8.8e--2 (4.6e--2)$^+$ & \bfseries 4/ 1/ 1\\
    \bfseries EHVI & 6.4e+7 (1.0e+5)$^-$ & \best {6.4e--1} (\best {2.2e--2})$^-$ & 5.8e+7 (2.4e+6)$^+$ & \best {7.0e--1} (\best {3.9e--3})$^-$ & \best {1.0e+0} (\best {6.9e--2})$^\sim$ & 2.6e--1 (7.7e--2)$^-$ & \bfseries 1/ 1/ 4\\
    \bfseries C-EHVI & 6.3e+7 (6.8e+5)$^-$ & 3.3e--1 (5.4e--2)$^+$ & 5.7e+7 (3.2e+6)$^+$ & 3.2e--1 (3.5e--2)$^+$ & 5.7e--1 (2.0e--1)$^+$ & 3.3e--1 (6.5e--2)$^-$ & \bfseries 4/ 0/ 2\\
    \bfseries JES & \best {6.4e+7} (\best {5.0e+4})$^-$ & 5.6e--1 (5.3e--2)$^-$ & 6.1e+7 (1.6e+6)$^\sim$ & 6.6e--1 (1.1e--2)$^-$ & 1.0e+0 (6.3e--2)$^\sim$ & 5.0e--1 (8.5e--2)$^-$ & \bfseries 0/ 2/ 4\\
    \bfseries SPMO & 6.2e+7 (1.0e+6)  & 4.6e--1 (6.3e--2)  & \best {6.1e+7} (\best {1.5e+6})  & 3.7e--1 (1.5e--2)  & 1.0e+0 (4.2e--2)  & 1.6e--1 (5.2e--2)  & \\
    \toprule
    \bfseries Method
    & \multicolumn{1}{c}{\bfseries DTLZ1 (10)} 
    & \multicolumn{1}{c}{\bfseries DTLZ2 (10)} 
    & \multicolumn{1}{c}{\bfseries Inverted DTLZ1 (10)} 
    & \multicolumn{1}{c}{\bfseries Inverted DTLZ2 (10)} 
    & \multicolumn{1}{c}{\bfseries Convex DTLZ2 (10)} 
    & \multicolumn{1}{c}{\bfseries Scaled DTLZ2 (10)} 
    & {\bfseries Sum up} \\ 
    & \multicolumn{1}{c }{Mean (Std)}
    & \multicolumn{1}{c }{Mean (Std)}
    & \multicolumn{1}{c }{Mean (Std)}
    & \multicolumn{1}{c }{Mean (Std)}
    & \multicolumn{1}{c }{Mean (Std)}
    & \multicolumn{1}{c }{Mean (Std)}
    & \multicolumn{1}{c}{$+$/$\sim$/$-$}  \\ \midrule

    \bfseries Sobol & 2.4e+28 (1.3e+28)$^-$ & 3.6e--1 (1.7e--1)$^+$ & 5.3e+25 (5.9e+25)$^+$ & 5.9e-11 (3.2e-10)$^+$ & 5.2e+0 (3.2e+0)$^+$ & 3.6e--1 (1.6e--1)$^+$ & \bfseries 5/ 0/ 1\\
    \bfseries ParEGO & 9.0e+26 (9.8e+26)$^+$ & 1.7e+0 (3.5e+0)$^+$ & 3.6e+27 (1.3e+28)$^+$ & 7.2e--4 (4.9e--4)$^+$ & 2.7e+2 (2.2e+2)$^+$ & 6.8e--1 (7.6e--1)$^+$ & \bfseries 6/ 0/ 0\\
    \bfseries TS-TCH & \best {2.9e+28} (\best {2.5e+28})$^-$ & 1.8e--1 (2.0e--1)$^+$ & 4.7e+25 (4.4e+25)$^+$ & 1.9e--8 (3.3e--8)$^+$ & 2.8e+1 (2.4e+1)$^+$ & 1.5e--1 (1.0e--1)$^+$ & \bfseries 5/ 0/ 1\\
    \bfseries C-EHVI & 1.0e+26 (1.6e+24)$^+$ & 5.2e--1 (1.2e--1)$^+$ & 7.1e+25 (1.3e+25)$^+$ & 6.6e--5 (2.6e--5)$^+$ & 1.3e+0 (4.0e--1)$^+$ & 3.8e--1 (2.2e--1)$^+$ & \bfseries 6/ 0/ 0\\
    \bfseries JES & 8.8e+26 (1.3e+27)$^+$ & 6.3e--1 (1.1e+0)$^+$ & 5.7e+27 (1.9e+28)$^+$ & 7.5e--4 (3.9e--4)$^+$ & 2.8e+2 (1.9e+2)$^+$ & 1.5e+0 (2.2e+0)$^+$ & \bfseries 6/ 0/ 0\\
    \bfseries SPMO & 1.2e+28 (1.5e+28)  & \best {5.2e+1} (\best {3.3e+1})  & \best {2.3e+28} (\best {5.3e+28})  & \best {1.4e--1} (\best {1.5e--1})  & \best {6.3e+3} (\best {8.0e+3})  & \best {5.2e+1} (\best {6.5e+1})  & \\
    \bottomrule
    
    \end{tabular}
    }
    \label{tbl:HV_M3}
\end{table*}

\FloatBarrier 

\begin{figure*}[!ht]
    \centering
    
    \begin{subfigure}[b]{0.6\textwidth}
    
    \begin{minipage}{\textwidth}
        \centering
        \includegraphics[width=\textwidth]{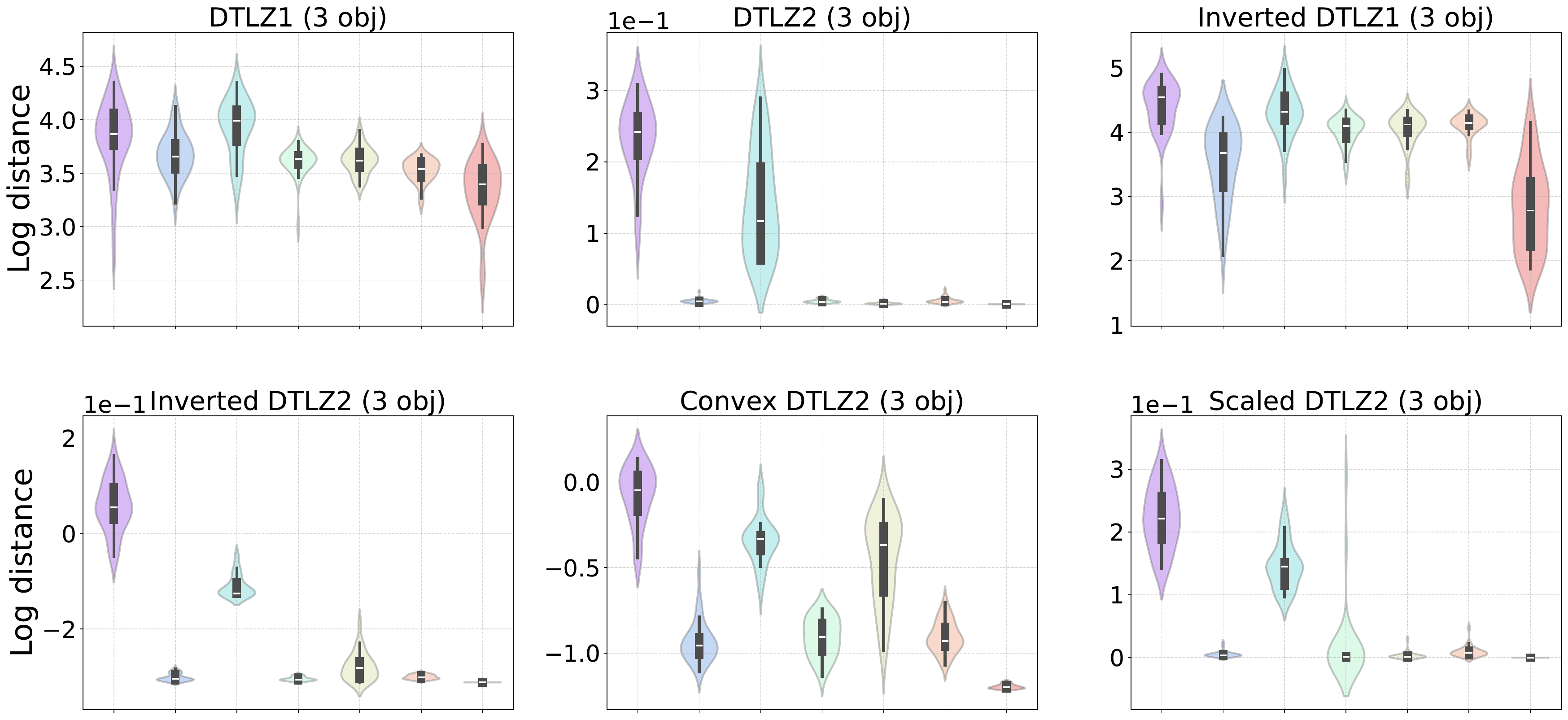}
    \end{minipage}

    \begin{minipage}{\textwidth}
        \centering
        \includegraphics[width=0.8\textwidth]{figures/distance/legend.pdf}
    \end{minipage}
    \caption{Violin plots of the distance-based metric (log distance) obtained by each method on problems with 3 objectives. }
    \end{subfigure}

    \begin{subfigure}[b]{0.6\textwidth}
    
    \begin{minipage}{\textwidth}
        \centering
        \includegraphics[width=\textwidth]{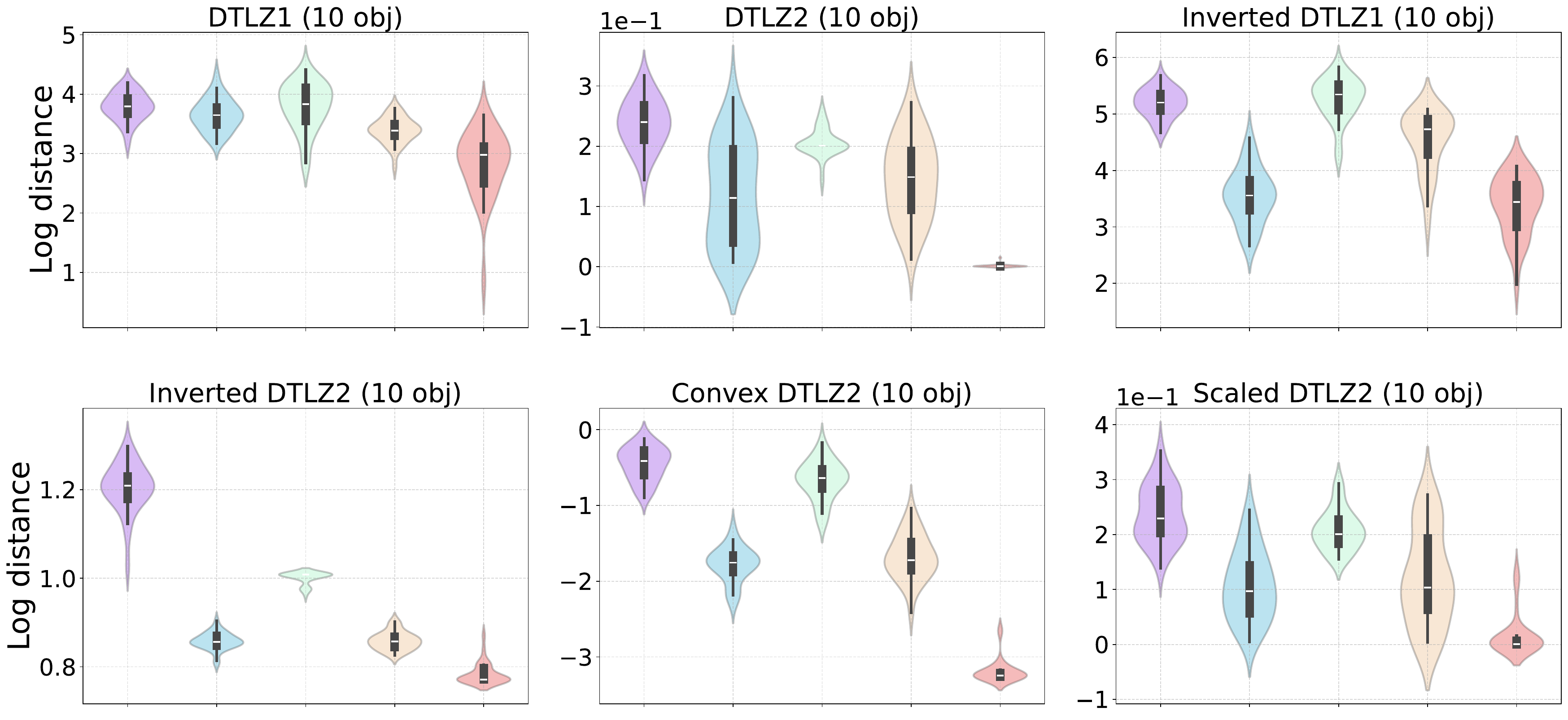}
    \end{minipage}

    \begin{minipage}{\textwidth}
        \centering
        \includegraphics[width=0.6\textwidth]{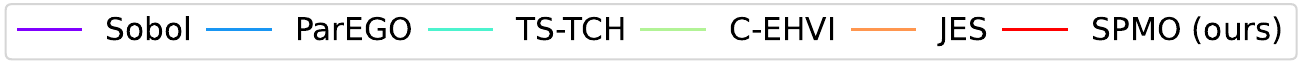}
    \end{minipage}
    \caption{Violin plots of the distance-based metric (log distance) obtained by each method on problems with 10 objectives. }
    \end{subfigure}
    
    \caption{Violin plots of the distance-based metric (log distance) obtained by the proposed SPMO and the peer methods on the noiseless problems with 3 and 10 objectives. 
    Each violin represents the distribution of the distance-based metric obtained by a method over 30 independent runs. 
    \label{fig:violin_dist_M3}
    }
\end{figure*}
\FloatBarrier

\begin{figure*}[!ht]
    \centering
    
    \begin{subfigure}[b]{\textwidth}
    
    \begin{minipage}{\textwidth}
        \centering
        \includegraphics[width=\textwidth]{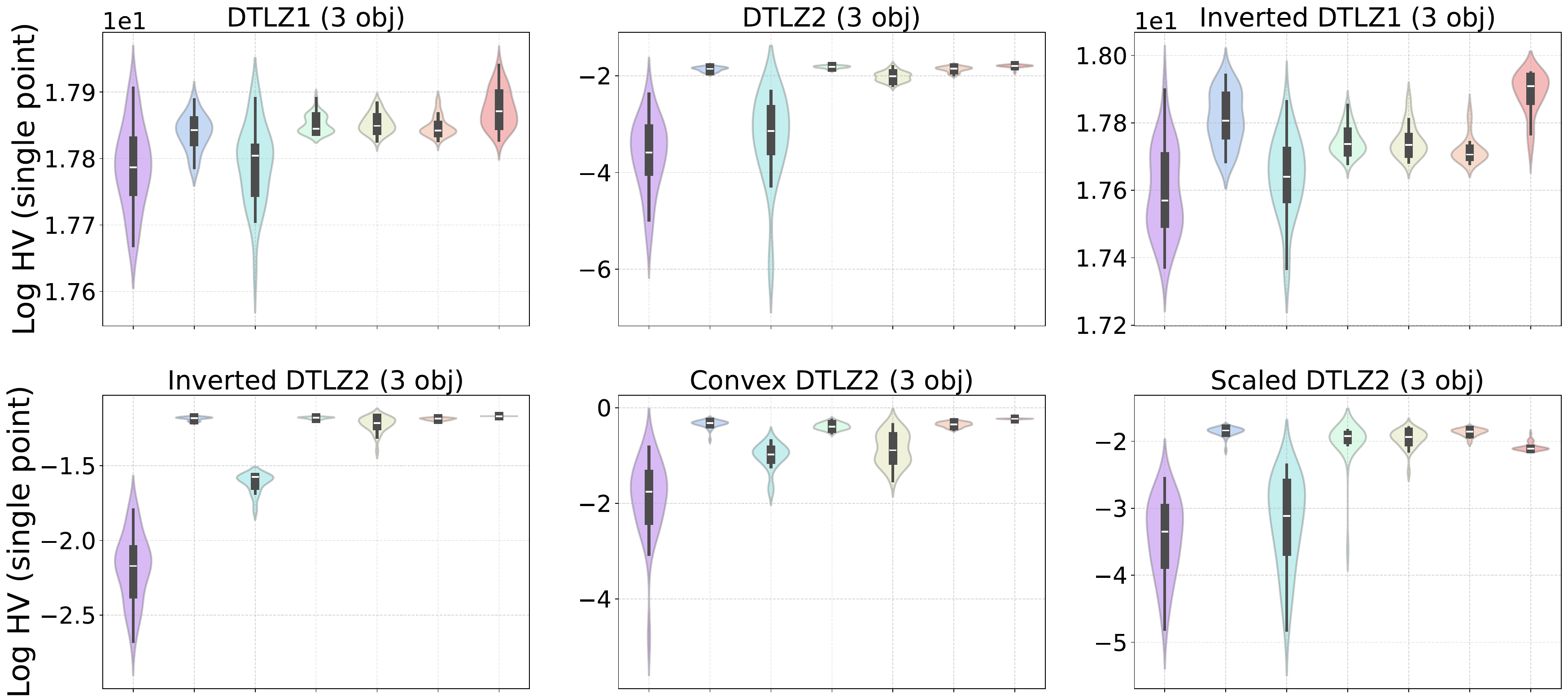}
    \end{minipage}

    \begin{minipage}{\textwidth}
        \centering
        \includegraphics[width=0.75\textwidth]{figures/distance/legend.pdf}
    \end{minipage}
    \caption{Violin plots of the HV of the best solution obtained by each method on problems with 3 objectives. }
    \end{subfigure}

    \begin{subfigure}[b]{\textwidth}
    
    \begin{minipage}{\textwidth}
        \centering
        \includegraphics[width=\textwidth]{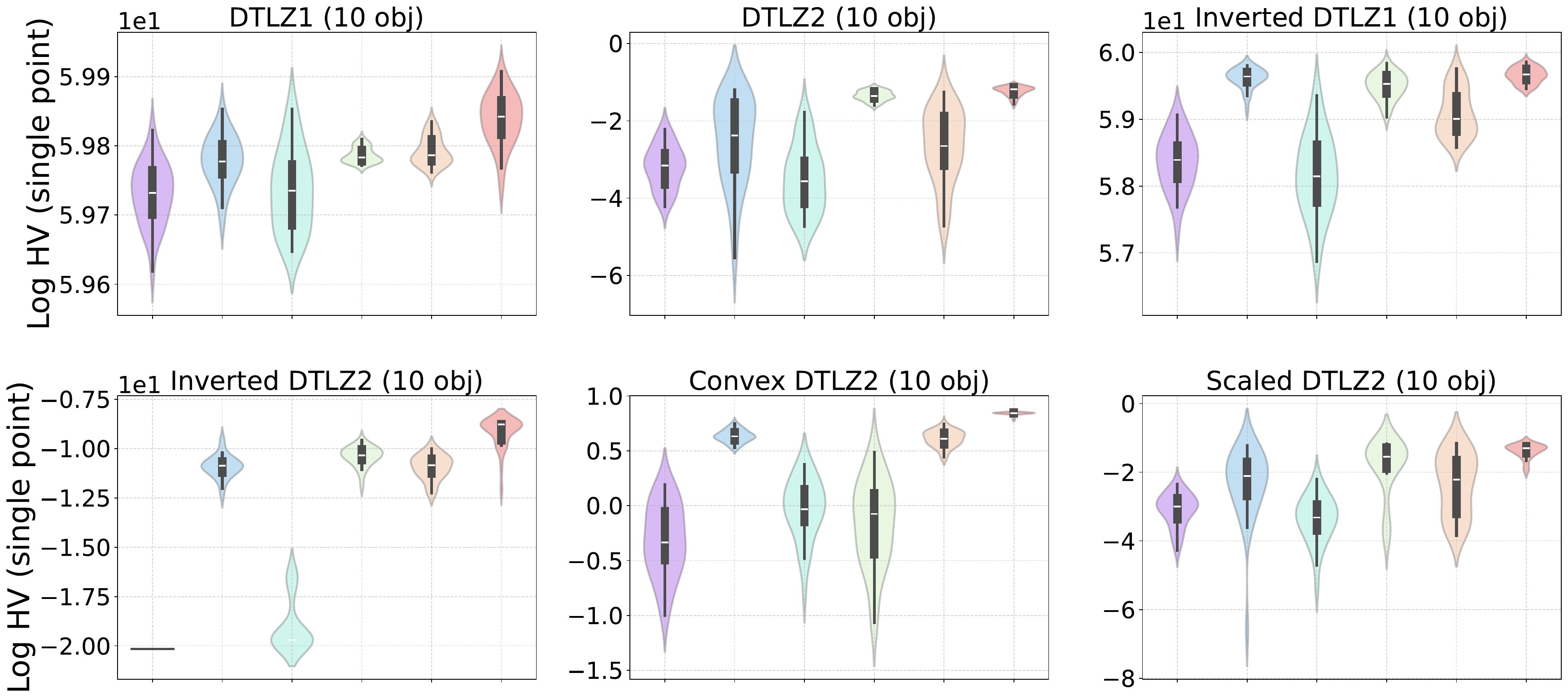}
    \end{minipage}

    \begin{minipage}{\textwidth}
        \centering
        \includegraphics[width=0.7\textwidth]{figures/distance/Legend_M10_False.pdf}
    \end{minipage}
    \caption{Violin plots of the HV of the best solution obtained by each method on problems with 10 objectives. }
    \end{subfigure}
    
    \caption{Violin plots of the HV of the best solution (in terms of its HV value) obtained by the proposed SPMO and the peer methods on the noiseless problems with 3 and 10 objectives. 
    Each violin represents the distribution of HV values obtained by a method over 30 independent runs. 
    \label{fig:violin_single_HV_M3}
    }
\end{figure*}
\FloatBarrier

\begin{figure*}[!ht]
    \centering

    \begin{subfigure}[b]{\textwidth}
        \begin{minipage}{\textwidth}
            \centering
            \includegraphics[width=\textwidth]{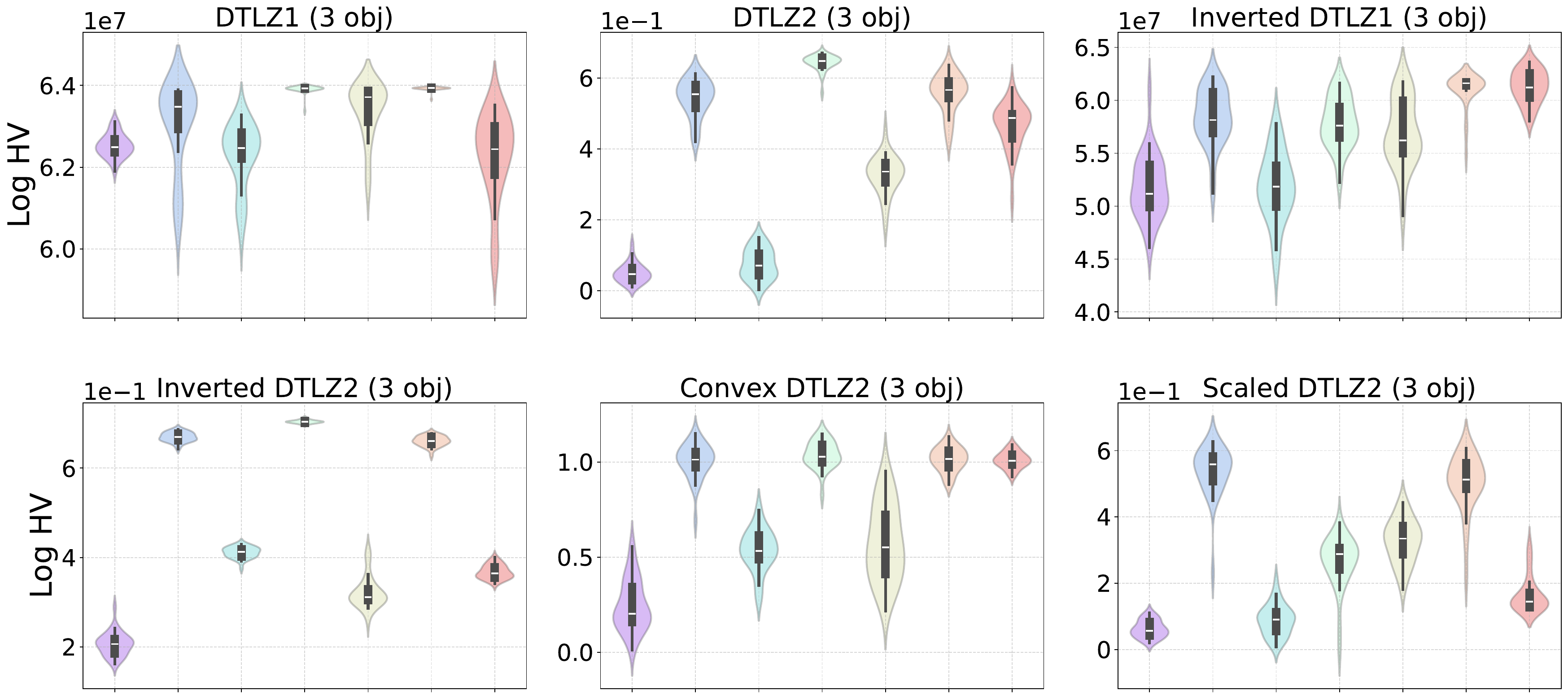}
        \end{minipage}

        \begin{minipage}{\textwidth}
            \centering
            \includegraphics[width=0.8\textwidth]{figures/distance/legend.pdf}
        \end{minipage}
        \caption{Violin plots of the HV of all evaluated solutions obtained by each method on problems with 3 objectives.}
    \end{subfigure}

    \begin{subfigure}[b]{\textwidth}
        \begin{minipage}{\textwidth}
            \centering
            \includegraphics[width=\textwidth]{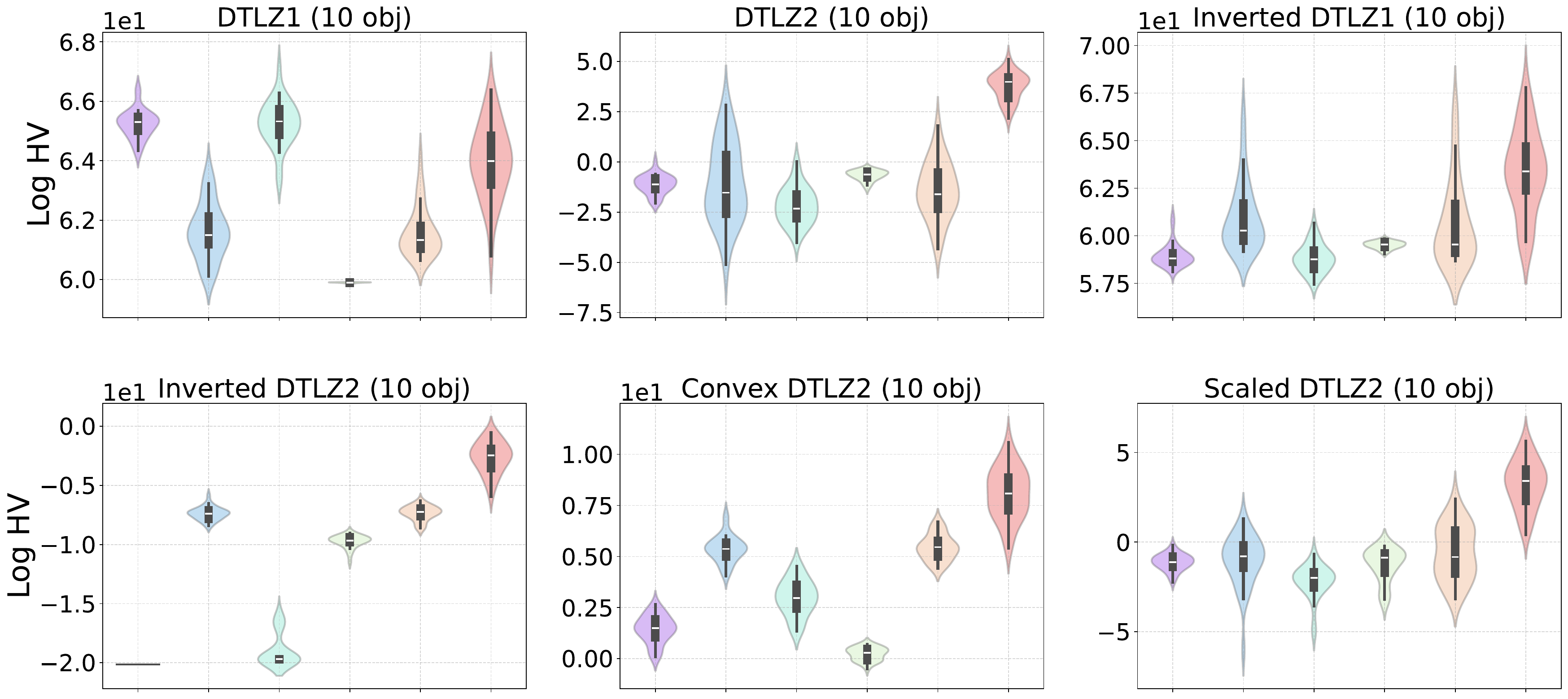}
        \end{minipage}

        \begin{minipage}{\textwidth}
            \centering
            \includegraphics[width=0.7\textwidth]{figures/distance/Legend_M10_False.pdf}
        \end{minipage}
        \caption{Violin plots of the HV of all evaluated solutions obtained by each method on problems with 10 objectives.}
    \end{subfigure}
    
    \caption{Violin plots of the HV of all evaluated solutions obtained by the proposed SPMO and the peer methods on the noiseless problems with 3 objectives (\textbf{top}) and 10 objectives (\textbf{bottom}), respectively. 
    Each violin represents the distribution of HV values obtained by a method over 30 independent runs. 
    \label{fig:violin_all_HV_M3}
    }
\end{figure*}
\FloatBarrier

\begin{figure*}[!ht]
    \centering
    
    \begin{subfigure}[b]{\textwidth}
    
    \begin{minipage}{\textwidth}
        \centering
        \includegraphics[width=\textwidth]{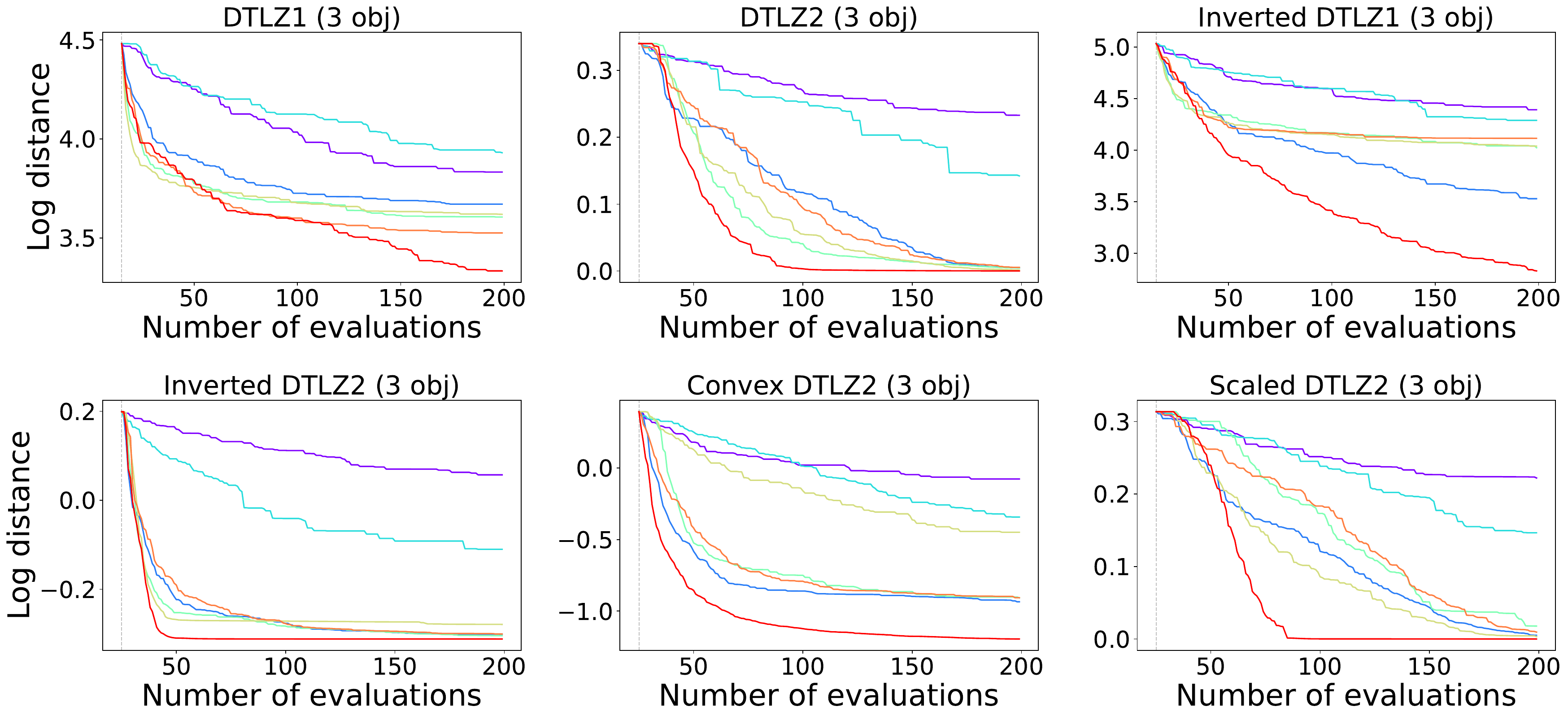}
    \end{minipage}
    \begin{minipage}{\textwidth}
        \centering
        \includegraphics[width=0.8\textwidth]{figures/distance/legend.pdf}
    \end{minipage}
    \caption{Trajectories of the distance metric on the problems with 3 objectives.}
    \end{subfigure}

    \begin{subfigure}[b]{\textwidth}
    
    \begin{minipage}{\textwidth}
        \centering
        \includegraphics[width=\textwidth]{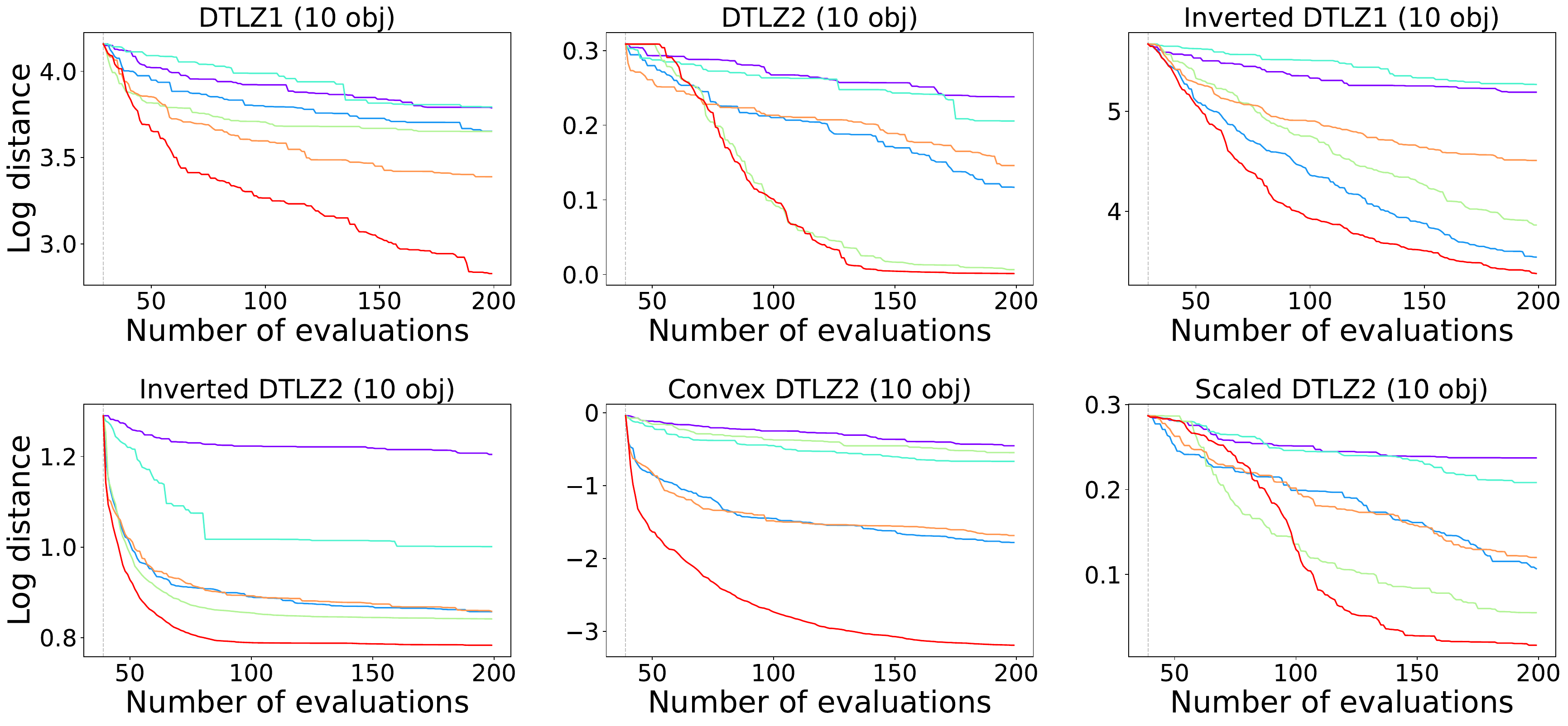}
    \end{minipage}
    
    \begin{minipage}{\textwidth}
        \centering
        \includegraphics[width=0.7\textwidth]{figures/distance/Legend_M10_False.pdf}
    \end{minipage}
    \caption{Trajectories of the distance metric on the problems with 10 objectives.}
    \end{subfigure}

    \caption{Trajectories of the distance metric obtained by the SPMO and the peer methods on the noiseless problems with 3 and 10 objectives. Each coloured line represents the mean distance of the closest solution to the utopian point on 30 independent runs (after the initial Sobol samples, represented by the dashed grey line). 
    }

    \label{fig:convergence_distance_M3}
\end{figure*}

\FloatBarrier

\newpage
\begin{table*}[!ht]
    \centering
    \caption{Results of the distance-based metric (log distance) obtained by obtained by the SPMO and the peer methods on DTLZ3--DTLZ7 with 5 objectives on 30 independent runs. The method with the best mean is highlighted in bold. The symbols ``$+$'', ``$\sim$'' and ``$-$'' indicate that the method is statistically worse than, equivalent to and better than our SPMO, respectively.}
    \resizebox{\textwidth}{!}{%
    \begin{tabular}{llllllc}
    \toprule
    \bfseries Method
    & \multicolumn{1}{c}{\bfseries DTLZ3} 
    & \multicolumn{1}{c}{\bfseries DTLZ4} 
    & \multicolumn{1}{c}{\bfseries DTLZ5} 
    & \multicolumn{1}{c}{\bfseries DTLZ6} 
    & \multicolumn{1}{c}{\bfseries DTLZ7} 
    & {\bfseries Sum up} \\ 
    & \multicolumn{1}{c}{Mean (Std)}
    & \multicolumn{1}{c}{Mean (Std)}
    & \multicolumn{1}{c}{Mean (Std)}
    & \multicolumn{1}{c}{Mean (Std)}
    & \multicolumn{1}{c}{Mean (Std)}
    & \multicolumn{1}{c}{$+$/$\sim$/$-$}  \\ \midrule
    \bfseries Sobol & 6.3e+0 (1.6e--1)$^+$ & 2.7e--1 (6.0e--2)$^+$ & 2.5e--1 (4.1e--2)$^+$ & 2.2e+0 (2.1e--2)$^+$ & 3.1e+0 (5.0e--2)$^+$ & \bfseries 5/ 0/ 0\\
    \bfseries ParEGO & 5.5e+0 (1.1e--1)$^\sim$ & 2.3e--1 (7.1e--2)$^+$ & 4.0e--2 (3.1e--2)$^+$ & 4.3e--2 (1.6e--1)$^+$ & 2.2e+0 (8.0e--2)$^+$ & \bfseries 4/ 1/ 0\\
    \bfseries TS-TCH & 6.3e+0 (1.6e--1)$^+$ & 2.6e--1 (5.0e--2)$^+$ & 2.3e--1 (5.9e--2)$^+$ & 2.2e+0 (1.6e--2)$^+$ & 3.1e+0 (4.8e--2)$^+$ & \bfseries 5/ 0/ 0\\
    \bfseries EHVI & \best {5.4e+0} (\best {4.6e--2})$^-$ & 8.7e--2 (9.4e--2)$^+$ & 3.0e--1 (6.0e--2)$^+$ & -9.4e--7 (8.4e--7)$^+$ & 1.7e+0 (8.5e--2)$^+$ & \bfseries 4/ 0/ 1\\
    \bfseries JES & 5.5e+0 (5.0e--2)$^\sim$ & 2.3e--1 (8.1e--2)$^+$ & 5.2e--2 (5.8e--2)$^+$ & 2.1e--2 (1.1e--1)$^+$ & 2.2e+0 (5.8e--2)$^+$ & \bfseries 4/ 1/ 0\\
    \bfseries SPMO & 5.5e+0 (1.8e--1)  & \best {1.2e--2} (\best {1.9e--2})  & \best {3.2e--3} (\best {5.0e--3})  & \best {-1.9e--6} (\best {5.6e--7})  & \best {1.7e+0} (\best {1.1e--1})  & \\
    \bottomrule
    \end{tabular}
    }
    \vspace*{0.1mm}
    \label{tbl:Dist_M5_dtlz3_7}
\end{table*}

\FloatBarrier

    
    


\begin{table*}[!ht]
    \centering
    \caption{The HV of the best solution (in terms of its HV value) obtained by SPMO and the peer methods on the DTLZ3--DTLZ7 with 5 objectives on 30 independent runs. 
    The method with the best mean is highlighted in bold. The symbols ``$+$'', ``$\sim$'' and ``$-$'' indicate that the method is statistically worse than, equivalent to and better than our SPMO, respectively.
    } 
    \resizebox{\textwidth}{!}{%
    \begin{tabular}{llllllc}
    \toprule
    \bfseries Method
    & \multicolumn{1}{c}{\bfseries DTLZ3} 
    & \multicolumn{1}{c}{\bfseries DTLZ4} 
    & \multicolumn{1}{c}{\bfseries DTLZ5} 
    & \multicolumn{1}{c}{\bfseries DTLZ6} 
    & \multicolumn{1}{c}{\bfseries DTLZ7} 
    & {\bfseries Sum up} \\ 
    & \multicolumn{1}{c }{Mean (Std)}
    & \multicolumn{1}{c }{Mean (Std)}
    & \multicolumn{1}{c }{Mean (Std)}
    & \multicolumn{1}{c }{Mean (Std)}
    & \multicolumn{1}{c }{Mean (Std)}
    & \multicolumn{1}{c}{$+$/$\sim$/$-$}  \\ \midrule
    \bfseries Sobol & 9.2e+19 (1.1e+18)$^+$ & 2.2e--2 (1.4e--2)$^+$ & 8.4e+4 (9.5e+2)$^\sim$ & 8.7e+3 (8.2e+2)$^+$ & -0.0e+0 (0.0e+0)$^+$ & \bfseries 4/ 1/ 0\\
    \bfseries ParEGO & 9.7e+19 (5.2e+17)$^\sim$ & 2.5e--2 (3.0e--2)$^+$ & \best {8.9e+4} (\best {1.4e+3})$^-$ & 8.8e+4 (3.8e+3)$^-$ & 2.8e+5 (2.7e+4)$^+$ & \bfseries 2/ 1/ 2\\
    \bfseries TS-TCH & 9.3e+19 (1.0e+18)$^+$ & 1.3e--2 (1.1e--2)$^+$ & 8.6e+4 (1.1e+3)$^-$ & 9.3e+3 (9.9e+2)$^+$ & -0.0e+0 (0.0e+0)$^+$ & \bfseries 4/ 0/ 1\\
    \bfseries EHVI & \best {9.8e+19} (\best {1.1e+17})$^-$ & 1.3e--1 (5.6e--2)$^+$ & 8.3e+4 (1.6e+3)$^+$ & \best {9.0e+4} (\best {1.3e--2})$^-$ & 3.8e+5 (2.2e+4)$^\sim$ & \bfseries 2/ 1/ 2\\
    \bfseries JES & 9.8e+19 (1.4e+17)$^-$ & 3.6e--2 (3.5e--2)$^+$ & 8.9e+4 (1.5e+3)$^-$ & 8.8e+4 (3.5e+3)$^-$ & 2.8e+5 (2.0e+4)$^+$ & \bfseries 2/ 0/ 3\\
    \bfseries SPMO & 9.7e+19 (1.0e+18)  & \best {1.6e--1} (\best {2.5e--2})  & 8.5e+4 (2.6e+3)  & 8.3e+4 (2.9e+3)  & \best {3.8e+5} (\best {2.8e+4})  & \\
    \bottomrule
    \end{tabular}
    }
    \vspace*{0.1mm}
    \label{tbl:HV_SP_M5_dtlz3_7}
\end{table*}

\FloatBarrier 

\begin{table*}[!ht]
    \centering
    \caption{The HV of all evaluated solutions obtained by SPMO and the peer methods on DTLZ3--DTLZ7 with 5 objectives on 30 independent runs. 
    The method with the best mean is highlighted in bold. The symbols ``$+$'', ``$\sim$'' and ``$-$'' indicate that the method is statistically worse than, equivalent to and better than our SPMO, respectively.
    } 
    \resizebox{\textwidth}{!}{%
    \begin{tabular}{llllllllc}
    \toprule
    \bfseries Method
    & \multicolumn{1}{c}{\bfseries DTLZ3} 
    & \multicolumn{1}{c}{\bfseries DTLZ4} 
    & \multicolumn{1}{c}{\bfseries DTLZ5} 
    & \multicolumn{1}{c}{\bfseries DTLZ6} 
    & \multicolumn{1}{c}{\bfseries DTLZ7} 
    & {\bfseries Sum up} \\ 
    & \multicolumn{1}{c }{Mean (Std)}
    & \multicolumn{1}{c }{Mean (Std)}
    & \multicolumn{1}{c }{Mean (Std)}
    & \multicolumn{1}{c }{Mean (Std)}
    & \multicolumn{1}{c }{Mean (Std)}
    & \multicolumn{1}{c}{$+$/$\sim$/$-$}  \\ \midrule

    \bfseries Sobol & 3.9e+20 (1.7e+20)$^\sim$ & 4.9e--2 (2.4e--2)$^+$ & 2.1e+5 (1.1e+5)$^\sim$ & 1.4e+5 (1.7e+4)$^+$ & -0.0e+0 (0.0e+0)$^+$ & \bfseries 3/ 2/ 0\\

    \bfseries ParEGO & 2.1e+20 (2.1e+20)$^\sim$ & 5.3e--2 (7.4e--2)$^+$ & 2.4e+5 (8.7e+4)$^\sim$ & 3.9e+5 (1.6e+5)$^-$ & 9.1e+5 (4.1e+5)$^+$ & \bfseries 2/ 2/ 1\\
    \bfseries TS-TCH & \best {5.0e+20} (\best {3.0e+20})$^-$ & 2.1e--2 (1.8e--2)$^+$ & 1.8e+5 (8.2e+4)$^+$ & 1.8e+5 (4.2e+4)$^+$ & -0.0e+0 (0.0e+0)$^+$ & \bfseries 4/ 0/ 1\\
    \bfseries EHVI & 1.7e+20 (1.5e+20)$^\sim$ & 9.3e--1 (7.1e--1)$^\sim$ & 3.5e+5 (1.8e+5)$^-$ & \best {4.2e+5} (\best {1.2e+5})$^-$ & 1.9e+6 (5.3e+5)$^\sim$ & \bfseries 0/ 3/ 2\\
    \bfseries JES & 2.7e+20 (2.7e+20)$^\sim$ & 8.8e--2 (1.3e--1)$^+$ & \best {3.6e+5} (\best {3.5e+5})$^\sim$ & 4.0e+5 (1.6e+5)$^-$ & 1.0e+6 (4.3e+5)$^+$ & \bfseries 2/ 2/ 1\\
    \bfseries SPMO & 4.2e+20 (7.0e+20)  & \best {1.1e+0} (\best {9.3e--1})  & 2.5e+5 (1.4e+5)  & 3.1e+5 (1.9e+5)  & \best {2.3e+6} (\best {1.3e+6})  & \\
    \bottomrule
    \end{tabular}
    }
    \vspace*{0.1mm}
    \label{tbl:HV_M5_dtlz3_7}
\end{table*}

    

    

\newpage
\subsection{Noisy Cases}\label{appendix:sec:noisy}

In this section, we present the results on the noisy problems.  
Tables~\ref{tbl:Dist_M5_noisy},~\ref{tbl:HV_SP_M5_noisy} and~\ref{tbl:HV_M5_noisy} show the distance-based metric (log distance), the HV of the best solution (in terms of its HV value) and the HV of all evaluated solutions obtained by the SPMO and the peer methods, respectively. 
Figures~\ref{fig:violin_dist_M5_noisy},~\ref{fig:violin_single_HV_M5_noisy} and~\ref{fig:violin_all_HV_M5_noisy} present the violin plots, illustrating the distributions of the corresponding results reported in Tables~\ref{tbl:Dist_M5_noisy},~\ref{tbl:HV_SP_M5_noisy} and~\ref{tbl:HV_M5_noisy}, respectively.

\begin{table*}[!ht]
    \centering
    \caption{Results of the distance-based metric (log distance) obtained by the SPMO and the peer methods on the noisy problems with 5 objectives on 30 independent runs. 
    The method with the best mean is highlighted in bold. The symbols ``$+$'', ``$\sim$'' and ``$-$'' indicate that the method is statistically worse than, equivalent to and better than our SPMO, respectively.}
    \resizebox{\textwidth}{!}{%
    \begin{tabular}{lllllllc}
    \toprule
    \bfseries Method
    & \multicolumn{1}{c}{\bfseries DTLZ1} 
    & \multicolumn{1}{c}{\bfseries DTLZ2} 
    & \multicolumn{1}{c}{\bfseries Inverted DTLZ1} 
    & \multicolumn{1}{c}{\bfseries Inverted DTLZ2} 
    & \multicolumn{1}{c}{\bfseries Convex DTLZ2} 
    & \multicolumn{1}{c}{\bfseries Scaled DTLZ2} 
    & {\bfseries Sum up} \\ 
    & \multicolumn{1}{c}{Mean (Std)}
    & \multicolumn{1}{c}{Mean (Std)}
    & \multicolumn{1}{c}{Mean (Std)}
    & \multicolumn{1}{c}{Mean (Std)}
    & \multicolumn{1}{c}{Mean (Std)}
    & \multicolumn{1}{c}{Mean (Std)}
    & \multicolumn{1}{c}{$+$/$\sim$/$-$}  \\ \midrule

    \bfseries Sobol & 3.7e+0 (4.2e--1)$^+$ & 1.9e--1 (5.6e--2)$^+$ & 4.7e+0 (3.4e--1)$^+$ & 6.1e--1 (6.1e--2)$^+$ & -2.9e--1 (2.2e--1)$^+$ & 2.2e--1 (4.4e--2)$^+$ & \bfseries 6/ 0/ 0\\
    \bfseries NParEGO & 3.6e+0 (3.1e--1)$^+$ & 1.0e--1 (8.2e--2)$^+$ & 3.1e+0 (3.8e--1)$^+$ & 1.8e--1 (5.0e--2)$^+$ & -1.3e+0 (2.7e--1)$^+$ & 5.2e--2 (5.9e--2)$^+$ & \bfseries 6/ 0/ 0\\
    \bfseries TS-TCH & 3.8e+0 (3.5e--1)$^+$ & 2.0e--1 (4.2e--2)$^+$ & 4.9e+0 (3.2e--1)$^+$ & 4.6e--1 (5.5e--2)$^+$ & -4.7e--1 (2.3e--1)$^+$ & 2.1e--1 (5.0e--2)$^+$ & \bfseries 6/ 0/ 0\\
    \bfseries NEHVI & 3.5e+0 (1.4e--1)$^+$ & -1.1e--1 (8.1e--2)$^+$ & 4.0e+0 (5.8e--1)$^+$ & 1.5e--1 (3.3e--2)$^+$ & -1.3e+0 (3.0e--1)$^+$ & 2.9e--1 (9.5e--2)$^+$ & \bfseries 6/ 0/ 0\\
    \bfseries JES & 3.4e+0 (1.2e--1)$^+$ & 1.3e--1 (1.2e--1)$^+$ & 4.5e+0 (1.1e--1)$^+$ & 1.9e--1 (6.5e--2)$^+$ & -6.9e--1 (3.5e--1)$^+$ & 1.0e--1 (8.7e--2)$^+$ & \bfseries 6/ 0/ 0\\
    \bfseries SPMO & \best {3.0e+0} (\best {6.0e--1})  & \best {-1.8e--1} (\best {5.8e--2})  & \best {2.9e+0} (\best {4.8e--1})  & \best {5.8e--2} (\best {6.8e--2})  & \best {-2.1e+0} (\best {3.3e--1})  & \best {-1.9e--1} (\best {1.8e--1})  & \\
    \bottomrule
    \end{tabular}
    }
    \vspace*{0.1mm}
    \label{tbl:Dist_M5_noisy}
\end{table*}

    
    



\begin{table*}[!ht]
    \centering
    \caption{The HV of the best solution (in terms of its HV value) obtained by SPMO and the peer methods on the noisy problems with 5 objectives on 30 independent runs. 
    The method with the best mean is highlighted in bold. The symbols ``$+$'', ``$\sim$'' and ``$-$'' indicate that the method is statistically worse than, equivalent to and better than our SPMO, respectively. 
    } 
    \resizebox{\textwidth}{!}{%
    \begin{tabular}{lllllllc}
    \toprule
    \bfseries Method
    & \multicolumn{1}{c}{\bfseries DTLZ1} 
    & \multicolumn{1}{c}{\bfseries DTLZ2} 
    & \multicolumn{1}{c}{\bfseries Inverted DTLZ1} 
    & \multicolumn{1}{c}{\bfseries Inverted DTLZ2} 
    & \multicolumn{1}{c}{\bfseries Convex DTLZ2} 
    & \multicolumn{1}{c}{\bfseries Scaled DTLZ2} 
    & {\bfseries Sum up} \\ 
    & \multicolumn{1}{c}{Mean (Std)}
    & \multicolumn{1}{c}{Mean (Std)}
    & \multicolumn{1}{c}{Mean (Std)}
    & \multicolumn{1}{c}{Mean (Std)}
    & \multicolumn{1}{c}{Mean (Std)}
    & \multicolumn{1}{c}{Mean (Std)}
    & \multicolumn{1}{c}{$+$/$\sim$/$-$}  \\ \midrule
    \bfseries Sobol & 8.6e+12 (5.2e+11)$^+$ & 5.1e--2 (3.0e--2)$^+$ & 5.4e+12 (1.2e+12)$^+$ & 8.2e--4 (1.5e--3)$^+$ & 4.0e--1 (1.8e--1)$^+$ & 3.8e--2 (1.9e--2)$^+$ & \bfseries 6/ 0/ 0\\
    \bfseries NParEGO & 9.1e+12 (2.5e+11)$^+$ & 1.0e--1 (4.9e--2)$^+$ & 9.0e+12 (4.1e+11)$^+$ & 5.8e--2 (1.6e--2)$^+$ & 1.2e+0 (2.0e--1)$^+$ & 1.3e--1 (4.0e--2)$^+$ & \bfseries 6/ 0/ 0\\
    \bfseries TS-TCH & 8.5e+12 (5.1e+11)$^+$ & 5.0e--2 (1.4e--2)$^+$ & 4.8e+12 (1.2e+12)$^+$ & 8.1e--3 (5.4e--3)$^+$ & 5.3e--1 (1.7e--1)$^+$ & 3.1e--2 (1.6e--2)$^+$ & \bfseries 6/ 0/ 0\\
    \bfseries NEHVI & 9.1e+12 (1.5e+11)$^+$ & 2.8e--1 (7.6e--2)$^+$ & 7.4e+12 (1.2e+12)$^+$ & 6.7e--2 (1.0e--2)$^+$ & 1.2e+0 (2.1e--1)$^+$ & 1.2e--2 (1.0e--2)$^+$ & \bfseries 6/ 0/ 0\\
    \bfseries JES & 9.0e+12 (1.3e+11)$^+$ & 8.7e--2 (7.2e--2)$^+$ & 6.1e+12 (3.2e+11)$^+$ & 5.8e--2 (2.0e--2)$^+$ & 7.8e--1 (3.1e--1)$^+$ & 1.0e--1 (5.7e--2)$^+$ & \bfseries 6/ 0/ 0\\
    \bfseries SPMO & \best {9.4e+12} (\best {4.4e+11})  & \best {3.2e--1} (\best {4.9e--2})  & \best {9.2e+12} (\best {4.2e+11})  & \best {1.0e--1} (\best {2.7e--2})  & \best {1.8e+0} (\best {2.2e--1})  & \best {3.8e--1} (\best {1.6e--1})  & \\
    \bottomrule
    \end{tabular}
    }
    \vspace*{0.1mm}
    \label{tbl:HV_SP_M5_noisy}
\end{table*}
\FloatBarrier

\begin{table*}[!ht]
    \centering
    \caption{The HV of all the solutions obtained by the proposed SPMO and the peer methods on the noisy problems with five objectives on 30 independent runs. 
    The method with the best mean is highlighted in bold. The symbols ``$+$'', ``$\sim$'', and ``$-$'' indicate that a method is statistically worse than, equivalent to, and better than SPMO, respectively.}
    \resizebox{\textwidth}{!}{%
    \begin{tabular}{lllllllc}
    \toprule
    \bfseries Method
    & \multicolumn{1}{c}{\bfseries DTLZ1} 
    & \multicolumn{1}{c}{\bfseries DTLZ2} 
    & \multicolumn{1}{c}{\bfseries Inverted DTLZ1} 
    & \multicolumn{1}{c}{\bfseries Inverted DTLZ2} 
    & \multicolumn{1}{c}{\bfseries Convex DTLZ2} 
    & \multicolumn{1}{c}{\bfseries Scaled DTLZ2} 
    & {\bfseries Sum up} \\ 
    & \multicolumn{1}{c }{Mean (Std)}
    & \multicolumn{1}{c }{Mean (Std)}
    & \multicolumn{1}{c }{Mean (Std)}
    & \multicolumn{1}{c }{Mean (Std)}
    & \multicolumn{1}{c }{Mean (Std)}
    & \multicolumn{1}{c }{Mean (Std)}
    & \multicolumn{1}{c}{$+$/$\sim$/$-$}  \\ \midrule
    \bfseries Sobol & 4.5e+13 (1.4e+13)$^+$ & 2.1e--1 (8.4e--2)$^+$ & 8.2e+12 (4.7e+12)$^+$ & 8.5e--4 (1.5e--3)$^+$ & 9.3e--1 (5.1e--1)$^+$ & 1.6e--1 (5.7e--2)$^+$ & \bfseries 6/ 0/ 0\\
    \bfseries NParEGO & 5.8e+13 (2.3e+13)$^\sim$ & 4.3e--1 (3.9e--1)$^+$ & 1.7e+13 (9.6e+12)$^+$ & 5.1e--1 (1.3e--1)$^+$ & 1.1e+1 (4.8e+0)$^+$ & 7.7e--1 (6.5e--1)$^+$ & \bfseries 5/ 1/ 0\\
    \bfseries TS-TCH & 5.1e+13 (1.3e+13)$^\sim$ & 1.9e--1 (8.3e--2)$^+$ & 6.9e+12 (2.6e+12)$^+$ & 3.2e--2 (1.8e--2)$^+$ & 2.2e+0 (9.8e--1)$^+$ & 9.0e--2 (6.5e--2)$^+$ & \bfseries 5/ 1/ 0\\
    \bfseries NEHVI & 7.2e+13 (2.1e+13)$^\sim$ & \best {3.1e+0} (\best {8.9e--1})$^-$ & 1.5e+13 (1.7e+13)$^+$ & 7.7e--1 (2.0e--1)$^+$ & 9.5e+0 (3.8e+0)$^+$ & 1.5e--2 (1.3e--2)$^+$ & \bfseries 4/ 1/ 1\\
    \bfseries JES & \best {7.5e+13} (\best {2.3e+13})$^\sim$ & 4.3e--1 (5.0e--1)$^+$ & 3.5e+13 (2.0e+13)$^\sim$ & 5.7e--1 (1.5e--1)$^+$ & 4.6e+0 (2.8e+0)$^+$ & 4.3e--1 (3.3e--1)$^+$ & \bfseries 4/ 2/ 0\\
    \bfseries SPMO & 6.2e+13 (2.9e+13)  & 2.0e+0 (5.6e--1)  & \best {4.4e+13} (\best {3.6e+13})  & \best {9.6e--1} (\best {3.1e--1})  & \best {1.4e+1} (\best {4.4e+0})  & \best {2.9e+0} (\best {1.8e+0})  & \\
    \bottomrule
    \end{tabular}
    }
    \vspace*{0.1mm}
    \label{tbl:HV_M5_noisy}
\end{table*}
\FloatBarrier

\begin{figure*}[!ht]
    \centering

    \begin{subfigure}[b]{0.65\textwidth}
        \begin{minipage}{\textwidth}
            \centering
            \includegraphics[width=\textwidth]{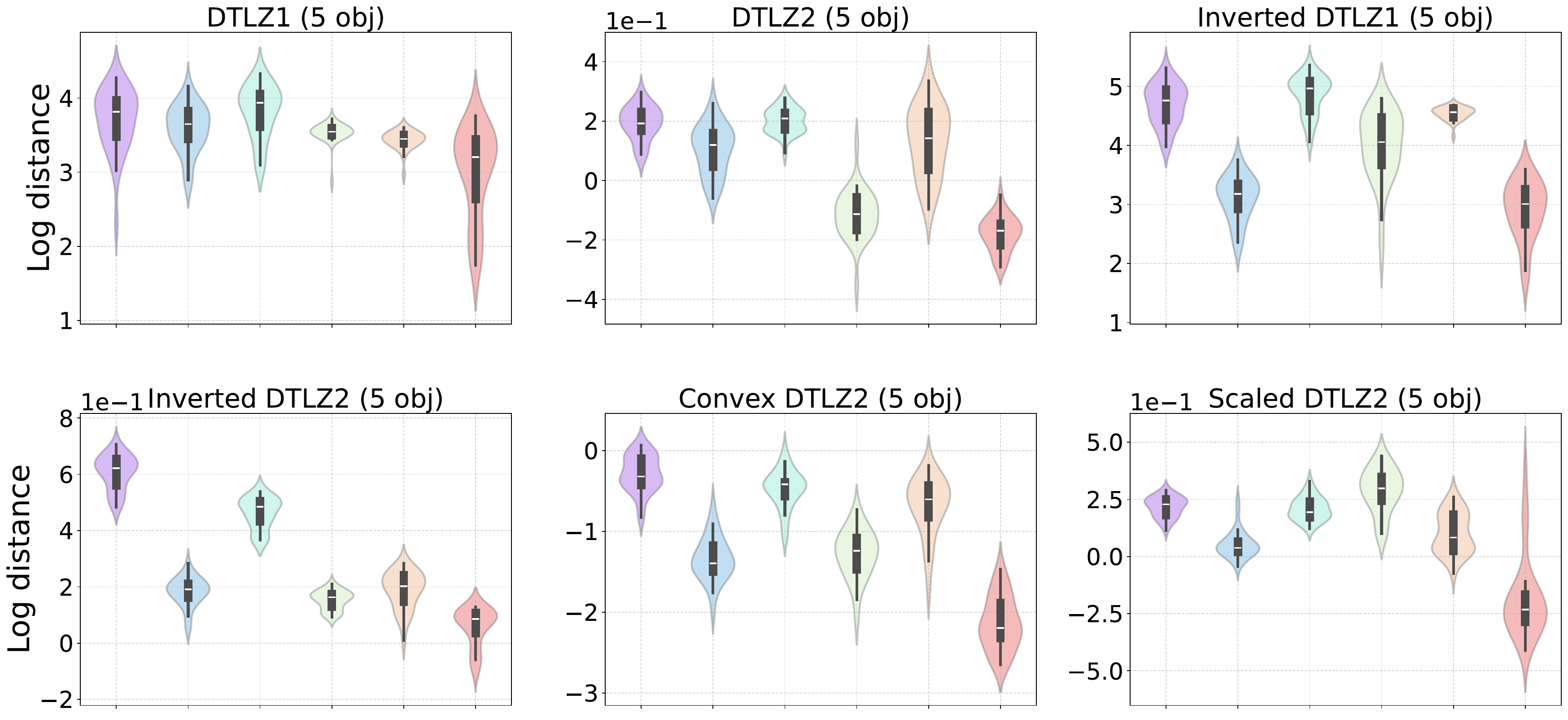}
        \end{minipage}

        \begin{minipage}{\textwidth}
            \centering
            \includegraphics[width=0.7\textwidth]{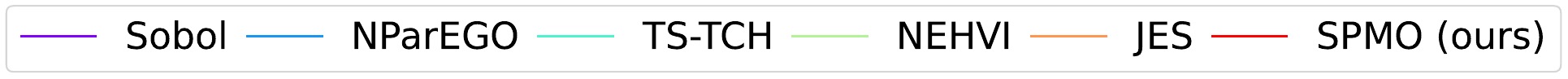}
        \end{minipage}
    \end{subfigure}

    \caption{Violin plots of the distance-based metric (log distance) obtained by the six methods on the noisy problems with five objectives. 
    Each violin represents the distribution of the distance-based metric obtained by a method over 30 independent runs. 
    }
    \label{fig:violin_dist_M5_noisy}
\end{figure*}
\FloatBarrier

\begin{figure*}[!ht]
    \centering

    \begin{subfigure}[b]{0.65\textwidth}
        \begin{minipage}{\textwidth}
            \centering
            \includegraphics[width=\textwidth]{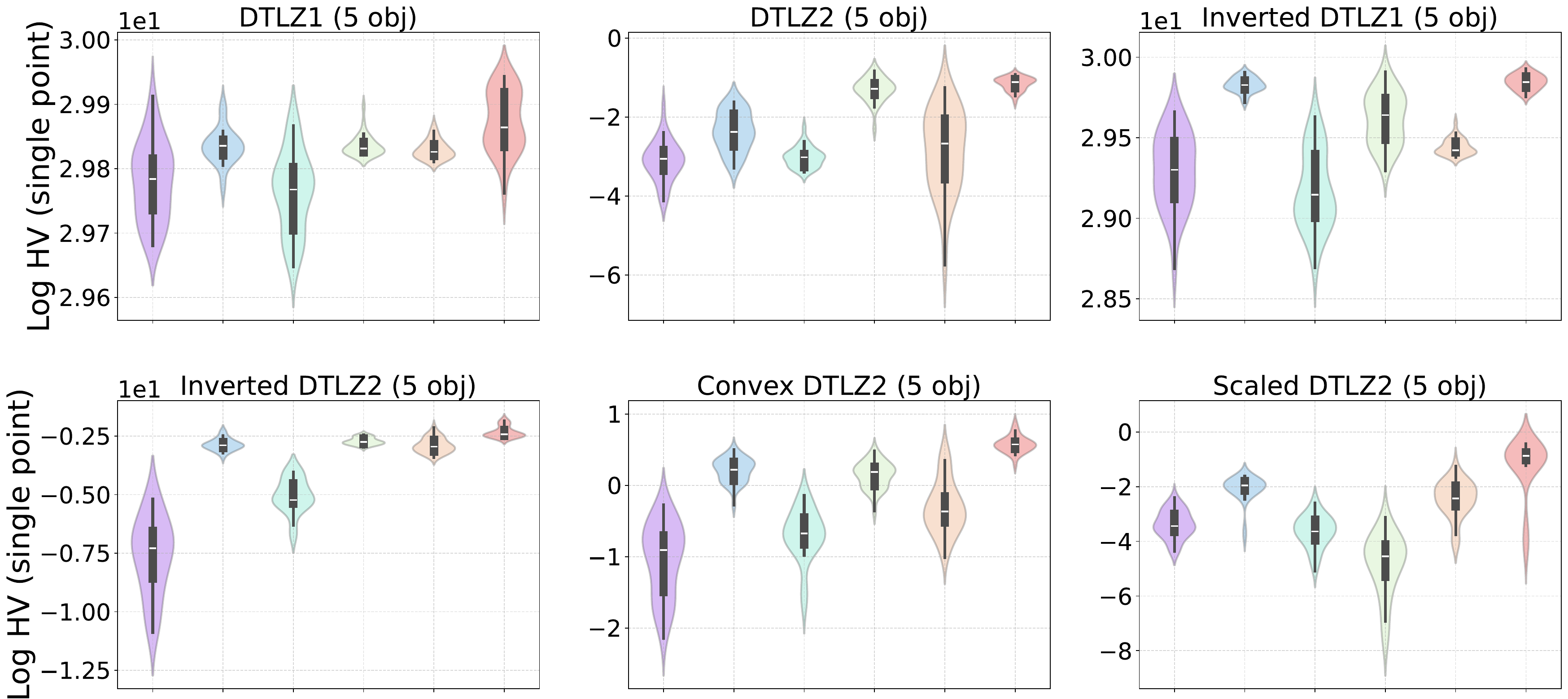}
        \end{minipage}

        \begin{minipage}{\textwidth}
            \centering
            \includegraphics[width=0.7\textwidth]{figures/distance/Legend_M5_True.pdf}
        \end{minipage}
    \end{subfigure}

    \caption{Violin plots of the HV of the best solution (in terms of its HV value) obtained by the six methods on the noisy problems with five objectives. 
    Each violin represents the distribution of maximum HV values obtained by a method over 30 independent runs.
    }
    \label{fig:violin_single_HV_M5_noisy}
\end{figure*}
\FloatBarrier

\begin{figure*}[!ht]
    \centering

    \begin{subfigure}[b]{0.65\textwidth}
    \begin{minipage}{\textwidth}
        \centering
        \includegraphics[width=\textwidth]{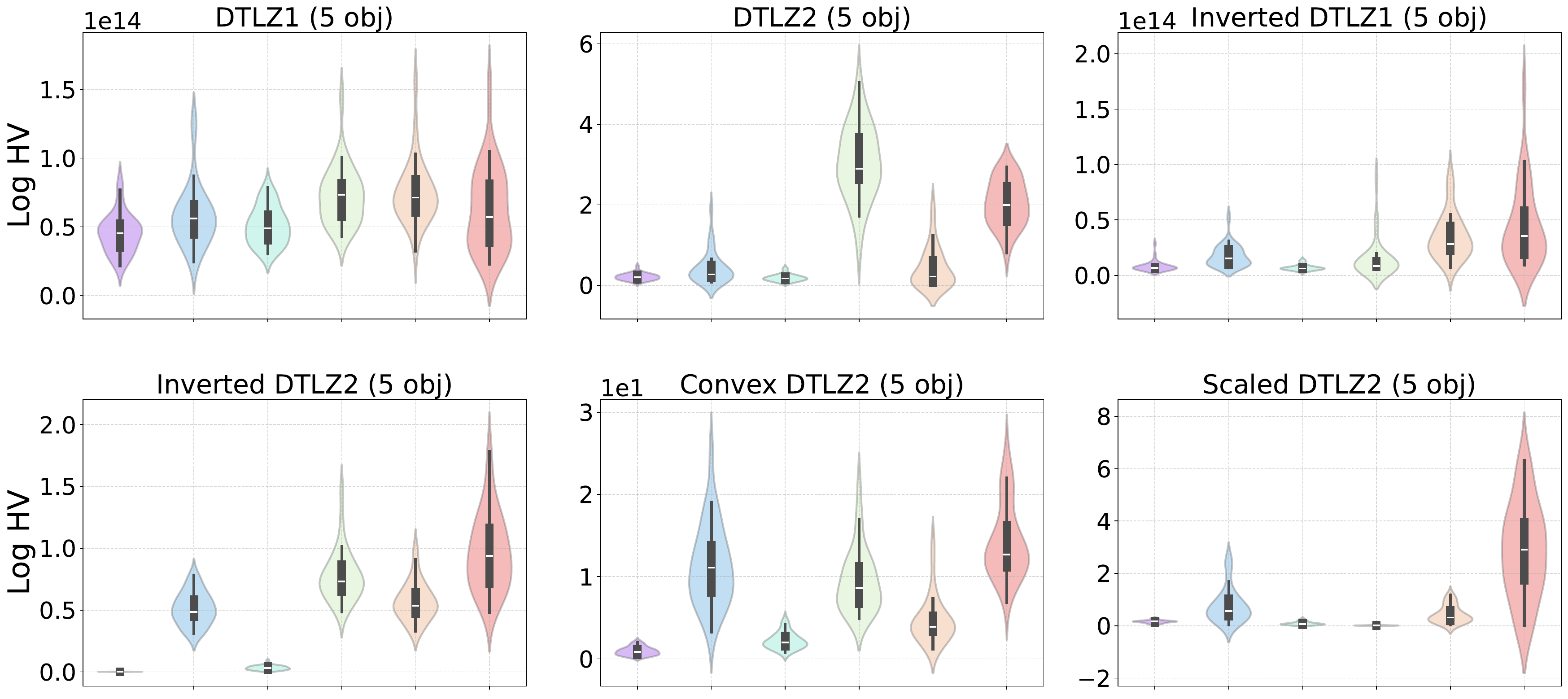}
    \end{minipage}

    \begin{minipage}{\textwidth}
        \centering
        \includegraphics[width=0.7\textwidth]{figures/distance/Legend_M5_True.pdf}
    \end{minipage}
    \end{subfigure}

    \caption{Violin plots of the HV of all evaluated solutions obtained by the six methods on the noisy problems with five objectives. 
    Each violin represents the distribution of maximum HV values obtained by a method over 30 independent runs. 
    \label{fig:violin_all_HV_M5_noisy}
    }
\end{figure*}
\FloatBarrier



\newpage
\subsection{Batch Setting}\label{appendix:sec:parallel}

We compare the proposed SPMO with the peer methods in the batch setting where the batch size $q$ is set to 5 (a commonly used value~\citep{lin2022pareto}). 
Tables~\ref{tbl:Dist_M5_batch},~\ref{tbl:HV_SP_M5_batch} and~\ref{tbl:HV_M5_batch} show the distance-based metric (log distance), the HV of the best solution (in terms of its HV value) and the HV of all evaluated solutions obtained by the SPMO and the peer methods, respectively. 
Figures~\ref{fig:violin_dist_M5_batch},~\ref{fig:violin_single_HV_M5_batch} and~\ref{fig:violin_all_HV_M5_batch} present the violin plots, illustrating the distributions of the corresponding results reported in Tables~\ref{tbl:Dist_M5_batch},~\ref{tbl:HV_SP_M5_batch} and~\ref{tbl:HV_M5_batch}, respectively.

\begin{table}[!ht]
    \centering
    \caption{Results of the distance-based metric (log distance) obtained by the SPMO and the five peer methods with a batch size $q=5$ on the problems with 5 objectives on 30 independent runs. 
    The method with the best mean is highlighted in bold. The symbols ``$+$'', ``$\sim$'' and ``$-$'' indicate that the method is statistically worse than, equivalent to and better than our SPMO, respectively.}
    \resizebox{\textwidth}{!}{%
    \begin{tabular}{lllllllc}
    \toprule
    \bfseries Method
    & \multicolumn{1}{c}{\bfseries DTLZ1} 
    & \multicolumn{1}{c}{\bfseries DTLZ2} 
    & \multicolumn{1}{c}{\bfseries Inverted DTLZ1} 
    & \multicolumn{1}{c}{\bfseries Inverted DTLZ2} 
    & \multicolumn{1}{c}{\bfseries Convex DTLZ2} 
    & \multicolumn{1}{c}{\bfseries Scaled DTLZ2} 
    & {\bfseries Sum up} \\ 
    & \multicolumn{1}{c}{Mean (Std)}
    & \multicolumn{1}{c}{Mean (Std)}
    & \multicolumn{1}{c}{Mean (Std)}
    & \multicolumn{1}{c}{Mean (Std)}
    & \multicolumn{1}{c}{Mean (Std)}
    & \multicolumn{1}{c}{Mean (Std)}
    & \multicolumn{1}{c}{$+$/$\sim$/$-$}  \\ \midrule
    \bfseries Sobol & 3.7e+0 (3.0e--1)$^+$ & 2.4e--1 (4.6e--2)$^+$ & 4.8e+0 (3.1e--1)$^+$ & 6.0e--1 (5.7e--2)$^+$ & -3.1e--1 (2.5e--1)$^+$ & 2.3e--1 (5.1e--2)$^+$ & \bfseries 6/ 0/ 0\\
    \bfseries ParEGO & 3.4e+0 (1.9e--1)$^+$ & 6.7e--2 (8.1e--2)$^+$ & 3.2e+0 (5.4e--1)$^+$ & 2.5e--1 (1.3e--2)$^+$ & -1.7e+0 (2.0e--1)$^+$ & 1.3e--1 (8.9e--2)$^+$ & \bfseries 6/ 0/ 0\\
    \bfseries TS-TCH & 3.9e+0 (2.2e--1)$^+$ & 2.0e--1 (4.1e--2)$^+$ & 4.8e+0 (3.2e--1)$^+$ & 4.6e--1 (1.2e--2)$^+$ & -7.1e--1 (2.1e--1)$^+$ & 2.6e--1 (3.8e--2)$^+$ & \bfseries 6/ 0/ 0\\
    \bfseries EHVI & 3.5e+0 (9.6e--2)$^+$ & 9.3e--3 (3.5e--3)$^+$ & 4.0e+0 (4.4e--1)$^+$ & 2.3e--1 (5.5e--3)$^+$ & -1.4e+0 (2.1e--1)$^+$ & 3.1e--1 (6.0e--2)$^+$ & \bfseries 6/ 0/ 0\\
    \bfseries JES & 3.4e+0 (1.3e--1)$^+$ & 1.1e--1 (8.0e--2)$^+$ & 4.5e+0 (1.7e--1)$^+$ & 2.6e--1 (2.0e--2)$^+$ & -1.0e+0 (4.5e--1)$^+$ & 8.8e--2 (9.3e--2)$^+$ & \bfseries 6/ 0/ 0\\
    \bfseries SPMO & \best {3.1e+0} (\best {3.0e--1})  & \best {9.0e--4} (\best {8.3e--4})  & \best {2.9e+0} (\best {4.9e--1})  & \best {2.1e--1} (\best {5.0e--5})  & \best {-2.1e+0} (\best {7.7e--3})  & \best {1.7e--4} (\best {1.2e--4})  & \\
    \bottomrule
    \end{tabular}
    }
    \vspace*{0.1mm}
    \label{tbl:Dist_M5_batch}
\end{table}
\FloatBarrier 

    
    



\begin{table}[!ht]
    \centering
    \caption{The HV of the best solution (in terms of its HV value) obtained by SPMO and the five peer methods with a batch size $q=5$ on the problems with 5 objectives on 30 independent runs. 
    The method with the best mean is highlighted in bold. The symbols ``$+$'', ``$\sim$'' and ``$-$'' indicate that the method is statistically worse than, equivalent to and better than our SPMO, respectively.
    } 
    \resizebox{\textwidth}{!}{%
    \begin{tabular}{lllllllc}
    \toprule
    \bfseries Method
    & \multicolumn{1}{c}{\bfseries DTLZ1} 
    & \multicolumn{1}{c}{\bfseries DTLZ2} 
    & \multicolumn{1}{c}{\bfseries Inverted DTLZ1} 
    & \multicolumn{1}{c}{\bfseries Inverted DTLZ2} 
    & \multicolumn{1}{c}{\bfseries Convex DTLZ2} 
    & \multicolumn{1}{c}{\bfseries Scaled DTLZ2} 
    & {\bfseries Sum up} \\ 
    & \multicolumn{1}{c}{Mean (Std)}
    & \multicolumn{1}{c}{Mean (Std)}
    & \multicolumn{1}{c}{Mean (Std)}
    & \multicolumn{1}{c}{Mean (Std)}
    & \multicolumn{1}{c}{Mean (Std)}
    & \multicolumn{1}{c}{Mean (Std)}
    & \multicolumn{1}{c}{$+$/$\sim$/$-$}  \\ \midrule
    \bfseries Sobol & 8.7e+12 (3.7e+11)$^+$ & 3.5e--2 (1.5e--2)$^+$ & 5.1e+12 (1.1e+12)$^+$ & 8.7e--4 (1.5e--3)$^+$ & 3.8e--1 (1.7e--1)$^+$ & 3.8e--2 (2.1e--2)$^+$ & \bfseries 6/ 0/ 0\\
    \bfseries ParEGO & 9.1e+12 (1.8e+11)$^+$ & 1.3e--1 (5.6e--2)$^+$ & 8.8e+12 (7.0e+11)$^+$ & 4.0e--2 (2.9e--3)$^+$ & 1.1e+0 (6.2e--2)$^+$ & 8.7e--2 (5.3e--2)$^+$ & \bfseries 6/ 0/ 0\\
    \bfseries TS-TCH & 8.4e+12 (3.8e+11)$^+$ & 3.2e--2 (1.3e--2)$^+$ & 4.9e+12 (1.2e+12)$^+$ & 7.1e--3 (1.1e--3)$^+$ & 6.5e--1 (1.4e--1)$^+$ & 2.8e--2 (1.5e--2)$^+$ & \bfseries 6/ 0/ 0\\
    \bfseries EHVI & 9.0e+12 (9.5e+10)$^+$ & \best {1.9e--1} (\best {7.2e--3})$^\sim$ & 7.5e+12 (9.7e+11)$^+$ & 4.5e--2 (1.4e--3)$^+$ & 1.0e+0 (9.7e--2)$^+$ & 1.5e--2 (1.4e--2)$^+$ & \bfseries 5/ 1/ 0\\
    \bfseries JES & 9.0e+12 (1.2e+11)$^+$ & 1.0e--1 (3.9e--2)$^+$ & 6.2e+12 (5.0e+11)$^+$ & 3.9e--2 (4.5e--3)$^+$ & 8.5e--1 (2.4e--1)$^+$ & 1.1e--1 (5.5e--2)$^+$ & \bfseries 6/ 0/ 0\\
    \bfseries SPMO & \best {9.3e+12} (\best {2.8e+11})  & 1.9e--1 (1.4e--2)  & \best {9.2e+12} (\best {4.8e+11})  & \best {4.9e--2} (\best {1.2e--5})  & \best {1.3e+0} (\best {5.0e--3})  & \best {1.6e--1} (\best {1.6e--2})  & \\
    \bottomrule
    \end{tabular}
    }
    \vspace*{0.1mm}
    \label{tbl:HV_SP_M5_batch}
\end{table}
\FloatBarrier

\begin{table}[!ht]
    \centering
    \caption{The HV of all the solutions obtained by the six methods on the problems with five objectives on 30 independent runs. 
    The method with the best mean is highlighted in bold. The symbols ``$+$'', ``$\sim$'', and ``$-$'' indicate that a method is statistically worse than, equivalent to, and better than SPMO, respectively.}
    \resizebox{\textwidth}{!}{%
    \begin{tabular}{llllllllc}
    \toprule
    \bfseries Method
    & \multicolumn{1}{c}{\bfseries DTLZ1} 
    & \multicolumn{1}{c}{\bfseries DTLZ2} 
    & \multicolumn{1}{c}{\bfseries Inverted DTLZ1} 
    & \multicolumn{1}{c}{\bfseries Inverted DTLZ2} 
    & \multicolumn{1}{c}{\bfseries Convex DTLZ2} 
    & \multicolumn{1}{c}{\bfseries Scaled DTLZ2} 
    & {\bfseries Sum up} \\ 
    & \multicolumn{1}{c }{Mean (Std)}
    & \multicolumn{1}{c }{Mean (Std)}
    & \multicolumn{1}{c }{Mean (Std)}
    & \multicolumn{1}{c }{Mean (Std)}
    & \multicolumn{1}{c }{Mean (Std)}
    & \multicolumn{1}{c }{Mean (Std)}
    & \multicolumn{1}{c}{$+$/$\sim$/$-$}  \\ \midrule

    \bfseries Sobol & 4.4e+13 (1.4e+13)$^\sim$ & 1.4e--1 (5.6e--2)$^+$ & 7.7e+12 (2.9e+12)$^+$ & 9.5e--4 (1.6e--3)$^+$ & 9.5e--1 (4.6e--1)$^+$ & 1.6e--1 (6.6e--2)$^+$ & \bfseries 5/ 1/ 0\\
    \bfseries ParEGO & 2.9e+13 (1.2e+13)$^\sim$ & 6.5e--1 (5.3e--1)$^+$ & 1.5e+13 (7.8e+12)$^+$ & 5.0e--1 (9.6e--2)$^\sim$ & 9.6e+0 (4.8e+0)$^+$ & 4.2e--1 (3.9e--1)$^+$ & \bfseries 4/ 2/ 0\\
    \bfseries TS-TCH & \best {4.5e+13} (\best {1.3e+13})$^\sim$ & 8.9e--2 (5.3e--2)$^+$ & 6.8e+12 (1.8e+12)$^+$ & 5.0e--2 (1.7e--2)$^+$ & 3.2e+0 (1.7e+0)$^+$ & 7.4e--2 (5.2e--2)$^+$ & \bfseries 5/ 1/ 0\\
    \bfseries EHVI & 2.9e+13 (1.8e+13)$^\sim$ & \best {2.1e+0} (\best {5.9e--1})$^-$ & 1.9e+13 (1.7e+13)$^+$ & \best {6.1e--1} (\best {1.5e--1})$^\sim$ & 7.4e+0 (2.2e+0)$^+$ & 2.0e--2 (1.7e--2)$^+$ & \bfseries 3/ 2/ 1\\
    \bfseries JES & 3.1e+13 (2.8e+13)$^\sim$ & 6.9e--1 (5.2e--1)$^+$ & \best {7.4e+13} (\best {9.1e+13})$^-$ & 4.7e--1 (1.1e--1)$^\sim$ & 5.7e+0 (4.6e+0)$^+$ & 7.5e--1 (6.3e--1)$^+$ & \bfseries 3/ 2/ 1\\
    \bfseries SPMO & 3.7e+13 (2.2e+13)  & 1.5e+0 (7.0e--1)  & 3.7e+13 (3.5e+13)  & 5.9e--1 (3.3e--1)  & \best {2.6e+1} (\best {1.7e+1})  & \best {1.9e+0} (\best {1.6e+0})  & \\
    \bottomrule
    \end{tabular}
    }
    \vspace*{0.1mm}
    \label{tbl:HV_M5_batch}
\end{table}
\FloatBarrier

\vspace{-2em}
\begin{figure}[!ht]
    \centering

    \begin{subfigure}[b]{0.65\textwidth}
        \begin{minipage}{\textwidth}
            \centering
            \includegraphics[width=\textwidth]{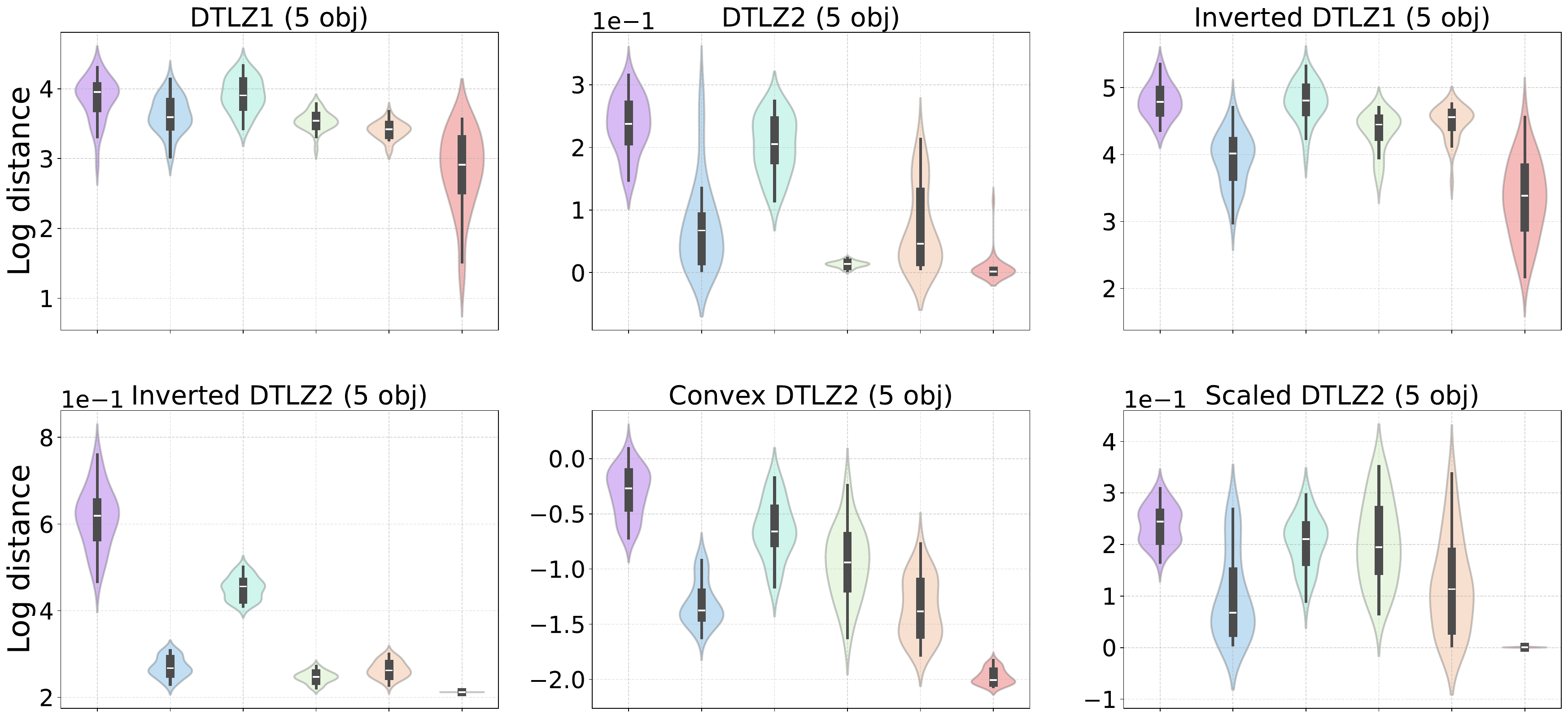}
        \end{minipage}

        \begin{minipage}{\textwidth}
            \centering
            \includegraphics[width=0.7\textwidth]{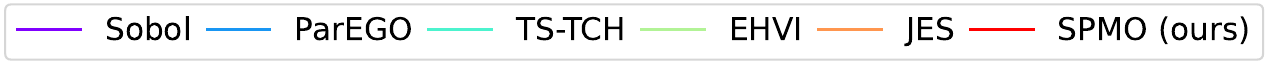}
        \end{minipage}
    \end{subfigure}

    \caption{Violin plots of the distance-based metric (log distance) obtained by the proposed SPMO and the peer methods on the problems with five objectives. 
    Each violin represents the distribution of the distance-based metric obtained by a method over 30 independent runs. 
    }
    \label{fig:violin_dist_M5_batch}
\end{figure}
\FloatBarrier 

\begin{figure}[!ht]
    \centering
    
    \begin{subfigure}[b]{0.65\textwidth}
    \begin{minipage}{\textwidth}
        \centering
        \includegraphics[width=\textwidth]{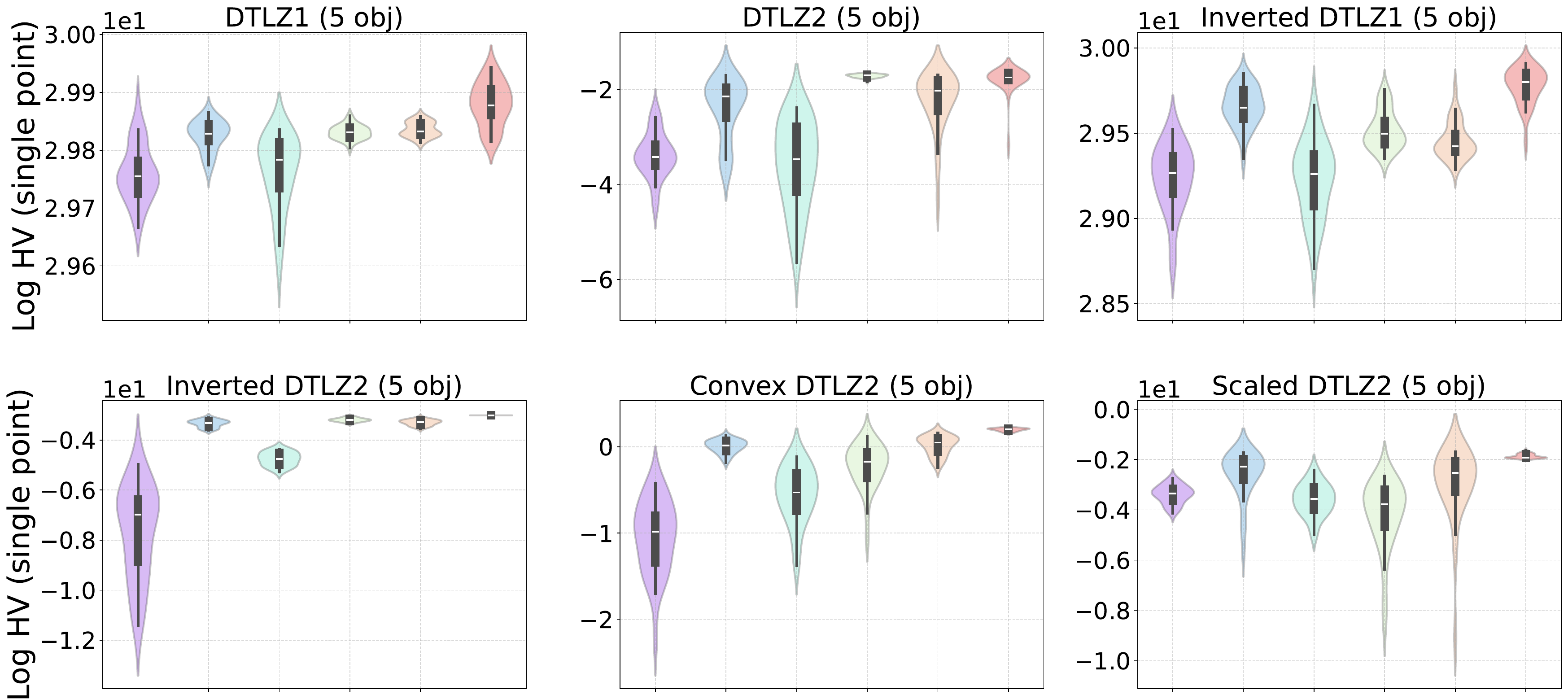}
    \end{minipage}

    \begin{minipage}{\textwidth}
        \centering
        \includegraphics[width=0.7\textwidth]{figures/distance/legend_M5_batch.pdf}
    \end{minipage}
    \end{subfigure}
    
    \caption{Violin plots of the HV of the best solution (in terms of its HV value) obtained by the six methods with a batch size $q=5$ on the problems with five objectives. 
    Each violin represents the distribution of maximum HV values obtained by a method over 30 independent runs. 
    \label{fig:violin_single_HV_M5_batch}
    }
\end{figure}
\FloatBarrier

\begin{figure}[!ht]
    \centering

    \begin{subfigure}[b]{0.65\textwidth}
    \begin{minipage}{\textwidth}
        \centering
        \includegraphics[width=\textwidth]{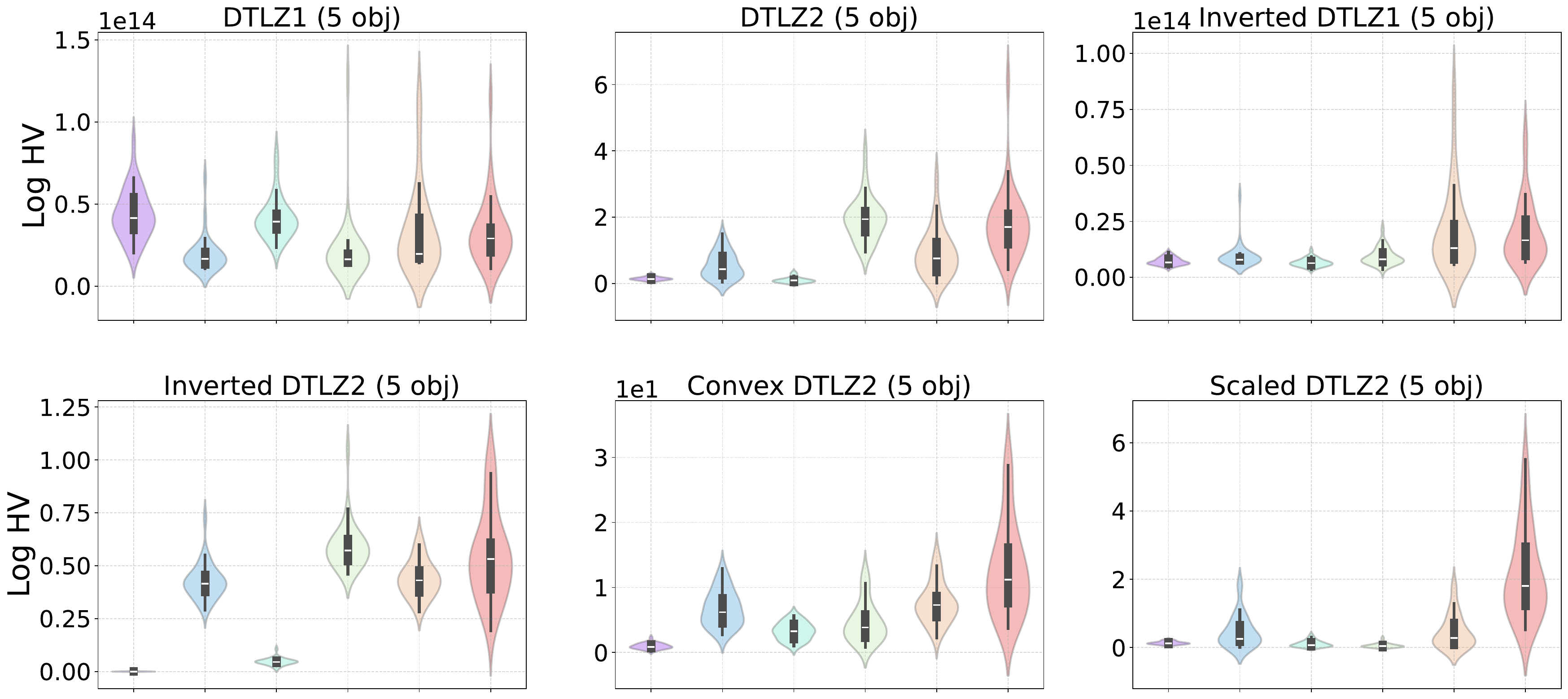}
    \end{minipage}

    \begin{minipage}{\textwidth}
        \centering
        \includegraphics[width=0.7\textwidth]{figures/distance/legend_M5_batch.pdf}
    \end{minipage}
    \end{subfigure}

    \caption{Violin plots of the HV of all evaluated solutions obtained by the six methods on the problems with five objectives. 
    Each violin represents the distribution of maximum HV values obtained by a method over 30 independent runs. 
    \label{fig:violin_all_HV_M5_batch}
    }
\end{figure}
\FloatBarrier

\newpage
\subsection{Sensitivity analysis}\label{appendix:sec:sensitivity}

In this section, we conduct a sensitivity analysis to assess the effect of different utopian points. 
We consider three different settings. The first one is slightly better than the ideal point, i.e., with a difference of $0.01$, the second is fairly better than the ideal point (i.e. $0.1$), and the last one is significantly better than the ideal point (i.e. $1.0$). 
Tables~\ref{tbl:Dist_M5_s},~\ref{tbl:HV_SP_M5_s} and~\ref{tbl:HV_M5_s} show the distance-based metric (log distance), the HV of the best solution (in terms of its HV value), and the HV of all evaluated solutions obtained by the SPMO with four
different utopian points, respectively. 
Figures~\ref{fig:violin_dist_M5_s},~\ref{fig:violin_single_HV_M5_s} and~\ref{fig:violin_all_HV_M5_s} present the violin plots, illustrating the distributions of the corresponding results reported in Tables~\ref{tbl:Dist_M5_s},~\ref{tbl:HV_SP_M5_s} and~\ref{tbl:HV_M5_s}, respectively.

\begin{table}[!ht]
    \centering
    \caption{Results of the distance-based metric (log distance) obtained by the SPMO with four different utopian points on the problems with five objectives on 30 independent runs. 
    The method with the best mean is highlighted in bold. The symbols ``$+$'', ``$\sim$'' and ``$-$'' indicate that the method is statistically worse than, equivalent to and better than SPMO (current), respectively.}
    \resizebox{\textwidth}{!}{%
    \begin{tabular}{llllllllc}
    \toprule
    \bfseries Method
    & \multicolumn{1}{c}{\bfseries DTLZ1} 
    & \multicolumn{1}{c}{\bfseries DTLZ2} 
    & \multicolumn{1}{c}{\bfseries Inverted DTLZ1} 
    & \multicolumn{1}{c}{\bfseries Inverted DTLZ2} 
    & \multicolumn{1}{c}{\bfseries Convex DTLZ2} 
    & \multicolumn{1}{c}{\bfseries Scaled DTLZ2} 
    & {\bfseries Sum up} \\ 
    & \multicolumn{1}{c }{Mean (Std)}
    & \multicolumn{1}{c }{Mean (Std)}
    & \multicolumn{1}{c }{Mean (Std)}
    & \multicolumn{1}{c }{Mean (Std)}
    & \multicolumn{1}{c }{Mean (Std)}
    & \multicolumn{1}{c }{Mean (Std)}
    & \multicolumn{1}{c}{$+$/$\sim$/$-$}  \\ \midrule

    \bfseries SPMO\_0.01 & 3.1e+0 (4.9e--1)$^\sim$ & 6.2e--4 (1.4e--3)$^-$ & \best {2.8e+0} (\best {4.2e--1})$^\sim$ & 2.1e--1 (3.6e--5)$^\sim$ & -2.1e+0 (1.6e--2)$^\sim$ & 6.4e--5 (3.9e--5)$^-$ & \bfseries 0/ 4/ 2\\
    \bfseries SPMO\_0.1 & 2.9e+0 (3.6e--1)$^-$ & \best {3.0e--4} (\best {2.3e--4})$^-$ & 3.0e+0 (5.3e--1)$^\sim$ & 2.1e--1 (4.6e--5)$^\sim$ & -2.1e+0 (2.6e--2)$^\sim$ & \best {5.3e--5} (\best {2.4e--5})$^-$ & \bfseries 0/ 3/ 3\\
    \bfseries SPMO\_1.0 & \best {2.8e+0} (\best {6.1e--1})$^\sim$ & 3.3e--4 (2.0e--4)$^-$ & 2.9e+0 (5.0e--1)$^\sim$ & \best {2.1e--1} (\best {3.4e--5})$^\sim$ & -2.1e+0 (1.6e--2)$^\sim$ & 5.7e--5 (2.4e--5)$^-$ & \bfseries 0/ 4/ 2\\
    \bfseries SPMO & 3.1e+0 (3.0e--1)  & 9.0e--4 (8.3e--4)  & 2.9e+0 (4.9e--1)  & 2.1e--1 (5.0e--5)  & \best {-2.1e+0} (\best {7.7e--3})  & 1.7e--4 (1.2e--4)  & \\
    \bottomrule
    \end{tabular}
    }
    \vspace*{0.1mm}
    \label{tbl:Dist_M5_s}
\end{table}

    



\begin{table}[!ht]
    \centering
    \caption{The HV of the best solution (in terms of its HV value) obtained by the proposed SPMO with four different utopian points on the problems with five objectives on 30 independent runs. 
    The method exhibiting the best mean is highlighted in bold. The symbols ``$+$'', ``$\sim$'', and ``$-$'' denote that a method is statistically worse than, equivalent to, or better than SPMO (current), respectively.}
    \resizebox{\textwidth}{!}{%
    \begin{tabular}{lllllllc}
    \toprule
    \bfseries Method
    & \multicolumn{1}{c}{\bfseries DTLZ1} 
    & \multicolumn{1}{c}{\bfseries DTLZ2} 
    & \multicolumn{1}{c}{\bfseries Inverted DTLZ1} 
    & \multicolumn{1}{c}{\bfseries Inverted DTLZ2} 
    & \multicolumn{1}{c}{\bfseries Convex DTLZ2} 
    & \multicolumn{1}{c}{\bfseries Scaled DTLZ2} 
    & {\bfseries Sum up} \\ 
    & \multicolumn{1}{c}{Mean (Std)}
    & \multicolumn{1}{c}{Mean (Std)}
    & \multicolumn{1}{c}{Mean (Std)}
    & \multicolumn{1}{c}{Mean (Std)}
    & \multicolumn{1}{c}{Mean (Std)}
    & \multicolumn{1}{c}{Mean (Std)}
    & \multicolumn{1}{c}{$+$/$\sim$/$-$}  \\ \midrule

    \bfseries SPMO\_0.01 & 9.4e+12 (3.7e+11)$^\sim$ & \best {1.9e--1} (\best {1.0e--2})$^\sim$ & \best {9.3e+12} (\best {3.8e+11})$^\sim$ & 4.9e--2 (9.0e--6)$^\sim$ & \best {1.3e+0} (\best {7.8e--3})$^\sim$ & 1.6e--1 (1.4e--2)$^\sim$ & \bfseries 0/ 6/ 0\\
    \bfseries SPMO\_0.1 & 9.5e+12 (2.7e+11)$^-$ & 1.9e--1 (1.6e--2)$^\sim$ & 9.0e+12 (5.6e+11)$^\sim$ & 4.9e--2 (1.1e--5)$^\sim$ & 1.2e+0 (1.1e--2)$^\sim$ & 1.6e--1 (1.0e--2)$^\sim$ & \bfseries 0/ 5/ 1\\
    \bfseries SPMO\_1.0 & \best {9.5e+12} (\best {3.8e+11})$^-$ & 1.9e--1 (1.3e--2)$^\sim$ & 9.1e+12 (4.2e+11)$^\sim$ & \best {4.9e--2} (\best {8.5e--6})$^\sim$ & 1.3e+0 (7.8e--3)$^\sim$ & 1.6e--1 (1.4e--2)$^\sim$ & \bfseries 0/ 5/ 1\\
    \bfseries SPMO & 9.3e+12 (2.8e+11)  & 1.9e--1 (1.4e--2)  & 9.2e+12 (4.8e+11)  & 4.9e--2 (1.2e--5)  & 1.3e+0 (5.0e--3)  & \best {1.6e--1} (\best {1.6e--2})  & \\
    \bottomrule
    \end{tabular}
    }
    \vspace*{0.1mm}
    \label{tbl:HV_SP_M5_s}
\end{table}
\FloatBarrier

\begin{table}[!ht]
    \centering
    \caption{The HV of all the solutions obtained by the SPMO with four different utopian points on the problems with five objectives on 30 independent runs. 
    The method with the best mean is highlighted in bold. The symbols ``$+$'', ``$\sim$'', and ``$-$'' indicate that a method is statistically worse than, equivalent to, and better than SPMO (current), respectively.}
    \resizebox{\textwidth}{!}{%
    \begin{tabular}{llllllllc}
    \toprule
    \bfseries Method
    & \multicolumn{1}{c}{\bfseries DTLZ1} 
    & \multicolumn{1}{c}{\bfseries DTLZ2} 
    & \multicolumn{1}{c}{\bfseries Inverted DTLZ1} 
    & \multicolumn{1}{c}{\bfseries Inverted DTLZ2} 
    & \multicolumn{1}{c}{\bfseries Convex DTLZ2} 
    & \multicolumn{1}{c}{\bfseries Scaled DTLZ2} 
    & {\bfseries Sum up} \\ 
    & \multicolumn{1}{c }{Mean (Std)}
    & \multicolumn{1}{c }{Mean (Std)}
    & \multicolumn{1}{c }{Mean (Std)}
    & \multicolumn{1}{c }{Mean (Std)}
    & \multicolumn{1}{c }{Mean (Std)}
    & \multicolumn{1}{c }{Mean (Std)}
    & \multicolumn{1}{c}{$+$/$\sim$/$-$}  \\ \midrule

    \bfseries SPMO\_0.01 & 3.1e+13 (1.7e+13)$^\sim$ & 2.7e+0 (1.2e+0)$^-$ & 3.5e+13 (3.7e+13)$^\sim$ & 6.5e--1 (7.0e--1)$^\sim$ & 2.5e+1 (1.8e+1)$^\sim$ & \best {2.2e+0} (\best {1.5e+0})$^\sim$ & \bfseries 0/ 5/ 1\\
    \bfseries SPMO\_0.1 & 3.5e+13 (1.5e+13)$^\sim$ & 2.7e+0 (1.2e+0)$^-$ & 2.6e+13 (1.7e+13)$^\sim$ & 6.1e--1 (4.5e--1)$^\sim$ & 2.7e+1 (1.7e+1)$^\sim$ & 1.8e+0 (1.0e+0)$^\sim$ & \bfseries 0/ 5/ 1\\
    \bfseries SPMO\_1.0 & 3.3e+13 (1.4e+13)$^\sim$ & \best {3.0e+0} (\best {1.6e+0})$^-$ & 3.2e+13 (1.9e+13)$^\sim$ & \best {7.1e--1} (\best {6.3e--1})$^\sim$ & \best {2.9e+1} (\best {2.2e+1})$^\sim$ & 2.1e+0 (1.5e+0)$^\sim$ & \bfseries 0/ 5/ 1\\
    \bfseries SPMO & \best {3.7e+13} (\best {2.2e+13})  & 1.5e+0 (7.0e--1)  & \best {3.7e+13} (\best {3.5e+13})  & 5.9e--1 (3.3e--1)  & 2.6e+1 (1.7e+1)  & 1.9e+0 (1.6e+0)  & \\
    \bottomrule
    \end{tabular}
    }
    \vspace*{0.1mm}
    \label{tbl:HV_M5_s}
\end{table}
\FloatBarrier

\vspace{-2em}
\begin{figure}[!ht]
    \centering

    \begin{subfigure}[b]{0.65\textwidth}
        \begin{minipage}{\textwidth}
            \centering
            \includegraphics[width=\textwidth]{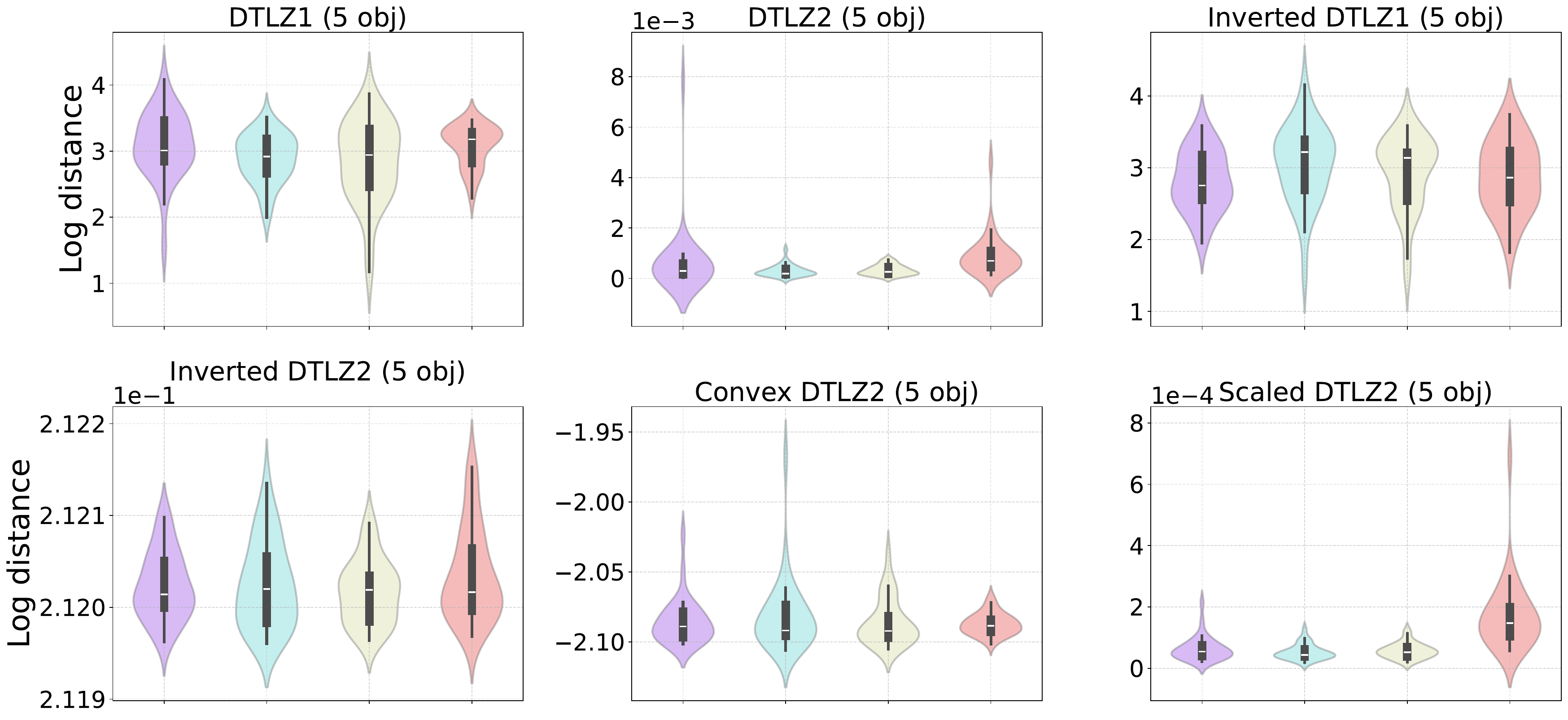}
        \end{minipage}

        \begin{minipage}{\textwidth}
            \centering
            \includegraphics[width=0.7\textwidth]{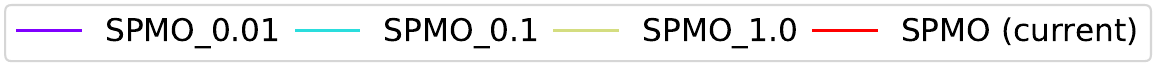}
        \end{minipage}
    \end{subfigure}

    \caption{Violin plots of the distance-based metric (log distance) obtained by the four methods on the problems with five objectives. 
    Each violin represents the distribution of the distance-based metric obtained by a method over 30 independent runs.
    }
    \label{fig:violin_dist_M5_s}
\end{figure}
\FloatBarrier

\begin{figure}[!ht]
    \centering

    \begin{subfigure}[b]{0.65\textwidth}
    \begin{minipage}{\textwidth}
        \centering
        \includegraphics[width=\textwidth]{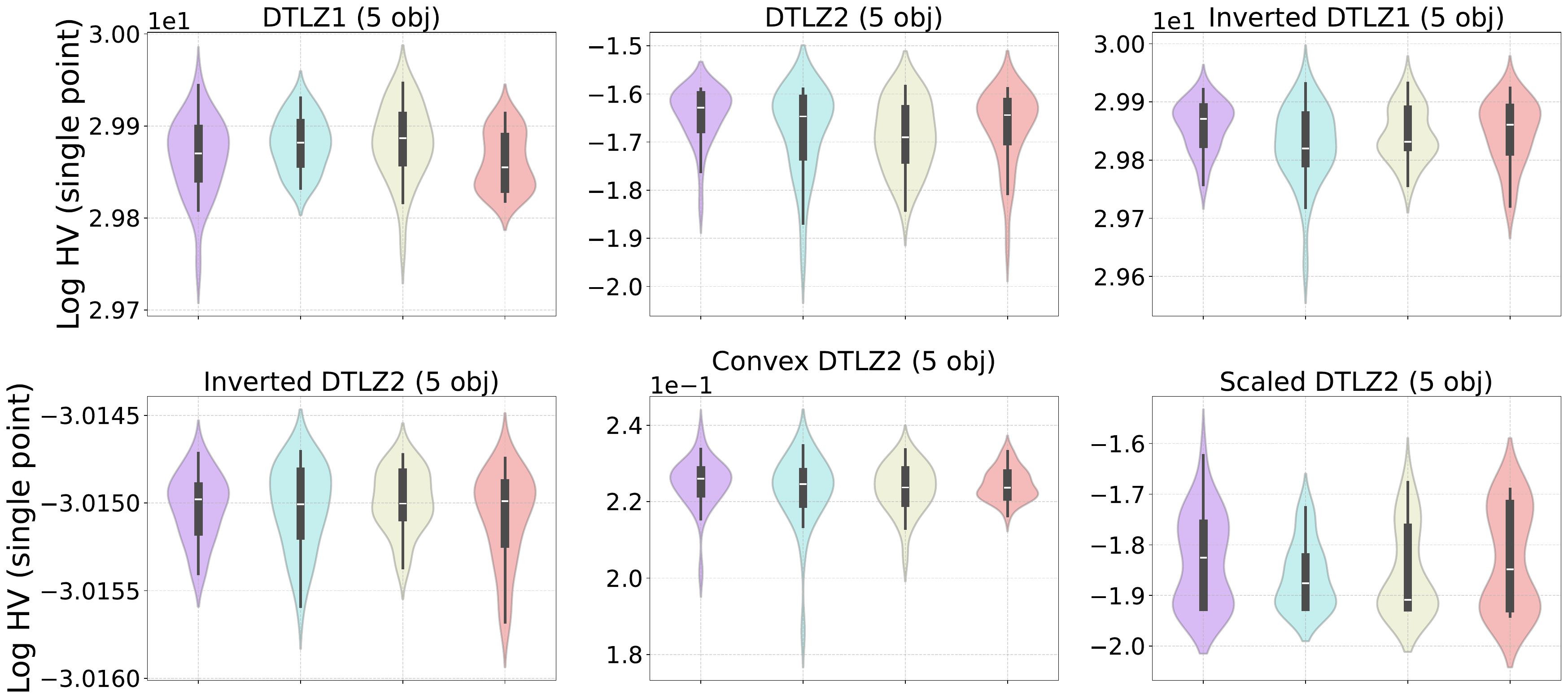}
    \end{minipage}

    \begin{minipage}{\textwidth}
        \centering
        \includegraphics[width=0.7\textwidth]{figures/distance/legend_s.pdf}
    \end{minipage}
    \end{subfigure}
    
    \caption{Violin plots of the HV of the best solution (in terms of its HV value) obtained by the four methods on the problems with five objectives. 
    Each violin represents the distribution of maximum HV values obtained by a method over 30 independent runs. 
    \label{fig:violin_single_HV_M5_s}
    }
\end{figure}
\FloatBarrier

\begin{figure}[!ht]
    \centering

    \begin{subfigure}[b]{0.65\textwidth}
    \begin{minipage}{\textwidth}
        \centering
        \includegraphics[width=\textwidth]{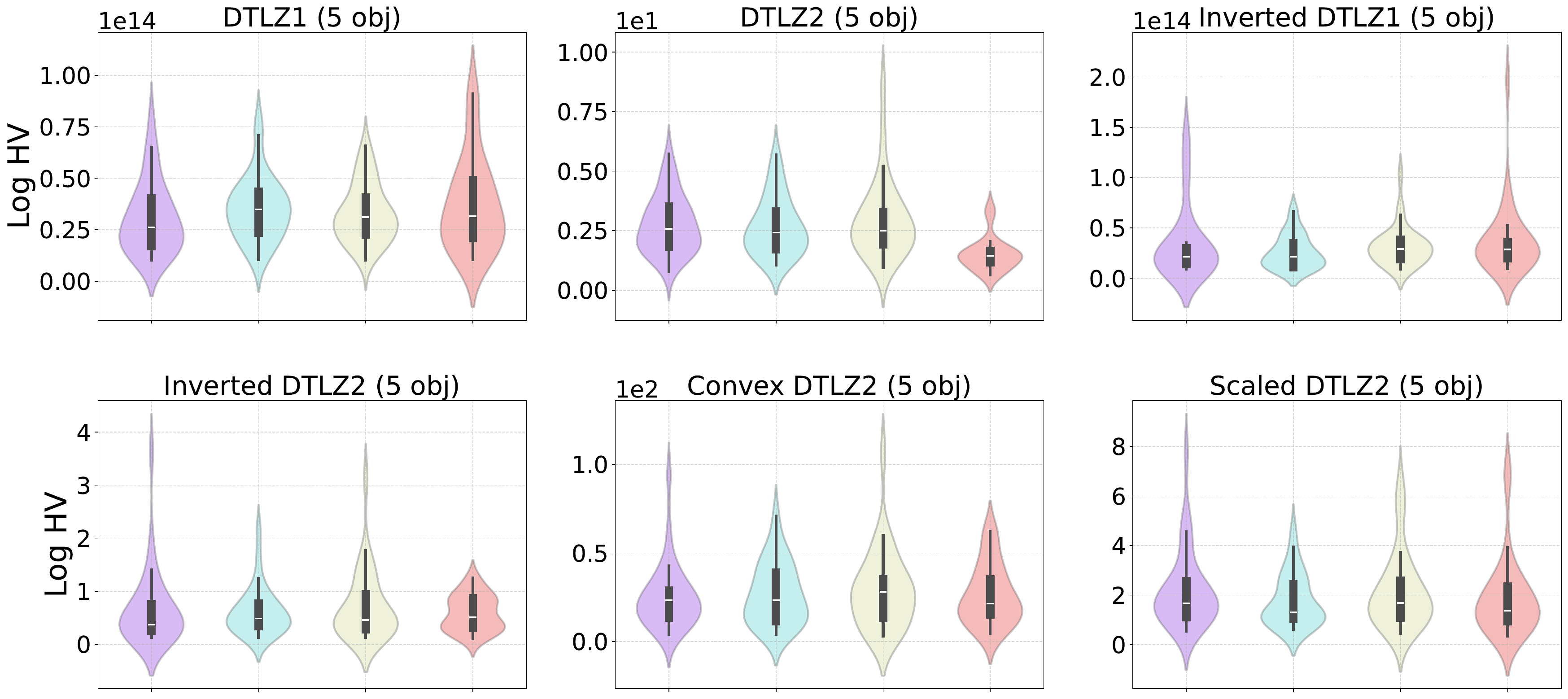}
    \end{minipage}

    \begin{minipage}{\textwidth}
        \centering
        \includegraphics[width=0.7\textwidth]{figures/distance/Legend_s.pdf}
    \end{minipage}
    \end{subfigure}

    \caption{Violin plots of the HV of all evaluated solutions obtained by the four methods the problems with five objectives. 
    Each violin represents the distribution of maximum HV values obtained by a method over 30 independent runs. 
    \label{fig:violin_all_HV_M5_s}
    }
\end{figure}
\FloatBarrier

\newpage
\subsection{Comparison of Single-Point Metrics within SPMO}\label{appendix:sec:metric}

In the proposed SPMO framework, we employ a distance metric (i.e., the distance of a solution to the utopian point). However, different metrics can be adopted provided that it can reflect the quality of a solution in achieving a good trade-off between objectives, such as the weighted sum and Tchebycheff scalarisation with the same weights $(\frac{1}{m},\dots,\frac{1}{m})$, where $m$ denotes the number of objectives. 
Here, we compare these three variants of SPMO, i.e., SPMO$_{dist}$, SPMO$_{Tch}$, and SPMO$_{ws}$. 
Tables~\ref{tbl:Dist_M5_metric},~\ref{tbl:HV_SP_M5_metric} and~\ref{tbl:HV_M5_metric} show the distance-based metric (log distance), the HV of the best solution (in terms of its HV value) and the HV of all evaluated solutions obtained by the SPMO with three different single-point metrics, respectively. 
Figures~\ref{fig:violin_dist_M5_metric},~\ref{fig:violin_single_HV_M5_metric} and~\ref{fig:violin_all_HV_M5_metric} present the violin plots, illustrating the distributions of the corresponding results reported in Tables~\ref{tbl:Dist_M5_metric},~\ref{tbl:HV_SP_M5_metric} and~\ref{tbl:HV_M5_metric}, respectively.

\begin{table}[!ht]
    \centering
    \caption{Results of the distance-based metric (log distance) obtained by the SPMO using three different single-point metrics on the problems with five objectives on 30 independent runs. The method with the best mean is highlighted in bold. The symbols ``$+$'', ``$\sim$'' and ``$-$'' indicate that the method is statistically worse than, equivalent to and better than our SPMO (i.e., SPMO$_{dist}$), respectively.}
    \resizebox{\textwidth}{!}{%
    \begin{tabular}{llllllllc}
    \toprule
    \bfseries Method
    & \multicolumn{1}{c}{\bfseries DTLZ1} 
    & \multicolumn{1}{c}{\bfseries DTLZ2} 
    & \multicolumn{1}{c}{\bfseries Inverted DTLZ1} 
    & \multicolumn{1}{c}{\bfseries Inverted DTLZ2} 
    & \multicolumn{1}{c}{\bfseries Convex DTLZ2} 
    & \multicolumn{1}{c}{\bfseries Scaled DTLZ2} 
    & {\bfseries Sum up} \\ 
    & \multicolumn{1}{c }{Mean (Std)}
    & \multicolumn{1}{c }{Mean (Std)}
    & \multicolumn{1}{c }{Mean (Std)}
    & \multicolumn{1}{c }{Mean (Std)}
    & \multicolumn{1}{c }{Mean (Std)}
    & \multicolumn{1}{c }{Mean (Std)}
    & \multicolumn{1}{c}{$+$/$\sim$/$-$}  \\ \midrule

    \bfseries SPMO$_{Tch}$ & 3.5e+0 (3.2e--1)$^+$ & \best {3.5e--5} (\best {4.3e--5})$^-$ & 2.9e+0 (6.5e--1)$^\sim$ & 3.1e--1 (3.3e--2)$^+$ & -1.8e+0 (1.7e--1)$^+$ & \best {2.9e--5} (\best {1.9e--5})$^-$ & \bfseries 3/ 1/ 2\\
    \bfseries SPMO$_{ws}$ & 3.6e+0 (1.3e--1)$^+$ & 2.1e--4 (1.5e--4)$^-$ & 3.1e+0 (5.4e--1)$^\sim$ & 2.2e--1 (2.7e--3)$^+$ & -1.2e+0 (2.7e--1)$^+$ & 3.2e--4 (3.4e--4)$^+$ & \bfseries 4/ 1/ 1\\
    \bfseries SPMO & \best {3.1e+0} (\best {3.0e--1})  & 9.0e--4 (8.3e--4)  & \best {2.9e+0} (\best {4.9e--1})  & \best {2.1e--1} (\best {5.0e--5})  & \best {-2.1e+0} (\best {7.7e--3})  & 1.7e--4 (1.2e--4)  & \\
    \bottomrule
    \end{tabular}
    }
    \vspace*{0.1mm}
    \label{tbl:Dist_M5_metric}
\end{table}

    
    



\begin{table}[!ht]
    \centering
    \caption{The HV of the best solution (in terms of its HV value) obtained by the proposed SPMO using three different single-point metrics on the problems with five objectives on 30 independent runs. 
    The method exhibiting the best mean is highlighted in bold. The symbols ``$+$'', ``$\sim$'', and ``$-$'' denote that a method is statistically worse than, equivalent to, or better than SPMO (i.e., SPMO$_{dist}$), respectively.}
    \resizebox{\textwidth}{!}{%
    \begin{tabular}{lllllllc}
    \toprule
    \bfseries Method
    & \multicolumn{1}{c}{\bfseries DTLZ1} 
    & \multicolumn{1}{c}{\bfseries DTLZ2} 
    & \multicolumn{1}{c}{\bfseries Inverted DTLZ1} 
    & \multicolumn{1}{c}{\bfseries Inverted DTLZ2} 
    & \multicolumn{1}{c}{\bfseries Convex DTLZ2} 
    & \multicolumn{1}{c}{\bfseries Scaled DTLZ2} 
    & {\bfseries Sum up} \\ 
    & \multicolumn{1}{c}{Mean (Std)}
    & \multicolumn{1}{c}{Mean (Std)}
    & \multicolumn{1}{c}{Mean (Std)}
    & \multicolumn{1}{c}{Mean (Std)}
    & \multicolumn{1}{c}{Mean (Std)}
    & \multicolumn{1}{c}{Mean (Std)}
    & \multicolumn{1}{c}{$+$/$\sim$/$-$}  \\ \midrule
    \bfseries SPMO$_{Tch}$ & 9.1e+12 (3.0e+11)$^+$ & 1.6e--1 (1.9e--2)$^+$ & 9.1e+12 (5.7e+11)$^\sim$ & 2.7e--2 (5.9e--3)$^+$ & 1.2e+0 (4.2e--2)$^+$ & 1.6e--1 (2.2e--2)$^\sim$ & \bfseries 4/ 2/ 0\\
    \bfseries SPMO$_{ws}$ & 9.1e+12 (1.3e+11)$^+$ & 1.7e--1 (1.4e--2)$^+$ & 9.0e+12 (6.5e+11)$^\sim$ & 4.8e--2 (6.7e--4)$^+$ & 9.7e--1 (1.2e--1)$^+$ & \best {1.7e--1} (\best {1.5e--2})$^\sim$ & \bfseries 4/ 2/ 0\\
    \bfseries SPMO & \best {9.3e+12} (\best {2.8e+11})  & \best {1.9e--1} (\best {1.4e--2})  & \best {9.2e+12} (\best {4.8e+11})  & \best {4.9e--2} (\best {1.2e--5})  & \best {1.3e+0} (\best {5.0e--3})  & 1.6e--1 (1.6e--2)  & \\
    \bottomrule
    \end{tabular}
    }
    \vspace*{0.1mm}
    \label{tbl:HV_SP_M5_metric}
\end{table}
\FloatBarrier

\begin{table}[!ht]
    \centering
    \caption{The HV of all the solutions obtained by the SPMO using three different single-point metrics on the problems with five objectives on 30 independent runs. 
    The method with the best mean is highlighted in bold. The symbols ``$+$'', ``$\sim$'', and ``$-$'' indicate that a method is statistically worse than, equivalent to, and better than SPMO (i.e., SPMO$_{dist}$), respectively.}
    \resizebox{\textwidth}{!}{%
    \begin{tabular}{llllllllc}
    \toprule
    \bfseries Method
    & \multicolumn{1}{c}{\bfseries DTLZ1} 
    & \multicolumn{1}{c}{\bfseries DTLZ2} 
    & \multicolumn{1}{c}{\bfseries Inverted DTLZ1} 
    & \multicolumn{1}{c}{\bfseries Inverted DTLZ2} 
    & \multicolumn{1}{c}{\bfseries Convex DTLZ2} 
    & \multicolumn{1}{c}{\bfseries Scaled DTLZ2} 
    & {\bfseries Sum up} \\ 
    & \multicolumn{1}{c }{Mean (Std)}
    & \multicolumn{1}{c }{Mean (Std)}
    & \multicolumn{1}{c }{Mean (Std)}
    & \multicolumn{1}{c }{Mean (Std)}
    & \multicolumn{1}{c }{Mean (Std)}
    & \multicolumn{1}{c }{Mean (Std)}
    & \multicolumn{1}{c}{$+$/$\sim$/$-$}  \\ \midrule

    \bfseries SPMO$_{Tch}$ & 1.5e+13 (6.2e+12)$^+$ & \best {4.9e+0} (\best {7.6e+0})$^-$ & 2.1e+13 (1.5e+13)$^+$ & \best {2.1e+0} (\best {2.6e+0})$^-$ & 1.3e+1 (7.4e+0)$^+$ & \best {3.3e+0} (\best {3.4e+0})$^-$ & \bfseries 3/ 0/ 3\\
    \bfseries SPMO$_{ws}$ & 1.3e+13 (3.1e+12)$^+$ & 1.6e+0 (7.5e--1)$^\sim$ & 2.4e+13 (1.8e+13)$^+$ & 7.4e--1 (7.5e--1)$^\sim$ & 6.4e+0 (3.8e+0)$^+$ & 1.2e+0 (3.5e--1)$^\sim$ & \bfseries 3/ 3/ 0\\
    \bfseries SPMO & \best {3.7e+13} (\best {2.2e+13})  & 1.5e+0 (7.0e--1)  & \best {3.7e+13} (\best {3.5e+13})  & 5.9e--1 (3.3e--1)  & \best {2.6e+1} (\best {1.7e+1})  & 1.9e+0 (1.6e+0)  & \\
    \bottomrule
    \end{tabular}
    }
    \vspace*{0.1mm}
    \label{tbl:HV_M5_metric}
\end{table}
\FloatBarrier

\newpage
\begin{figure}[!ht]
    \centering

    \begin{subfigure}[b]{0.63\textwidth}
        \begin{minipage}{\textwidth}
            \centering
            \includegraphics[width=\textwidth]{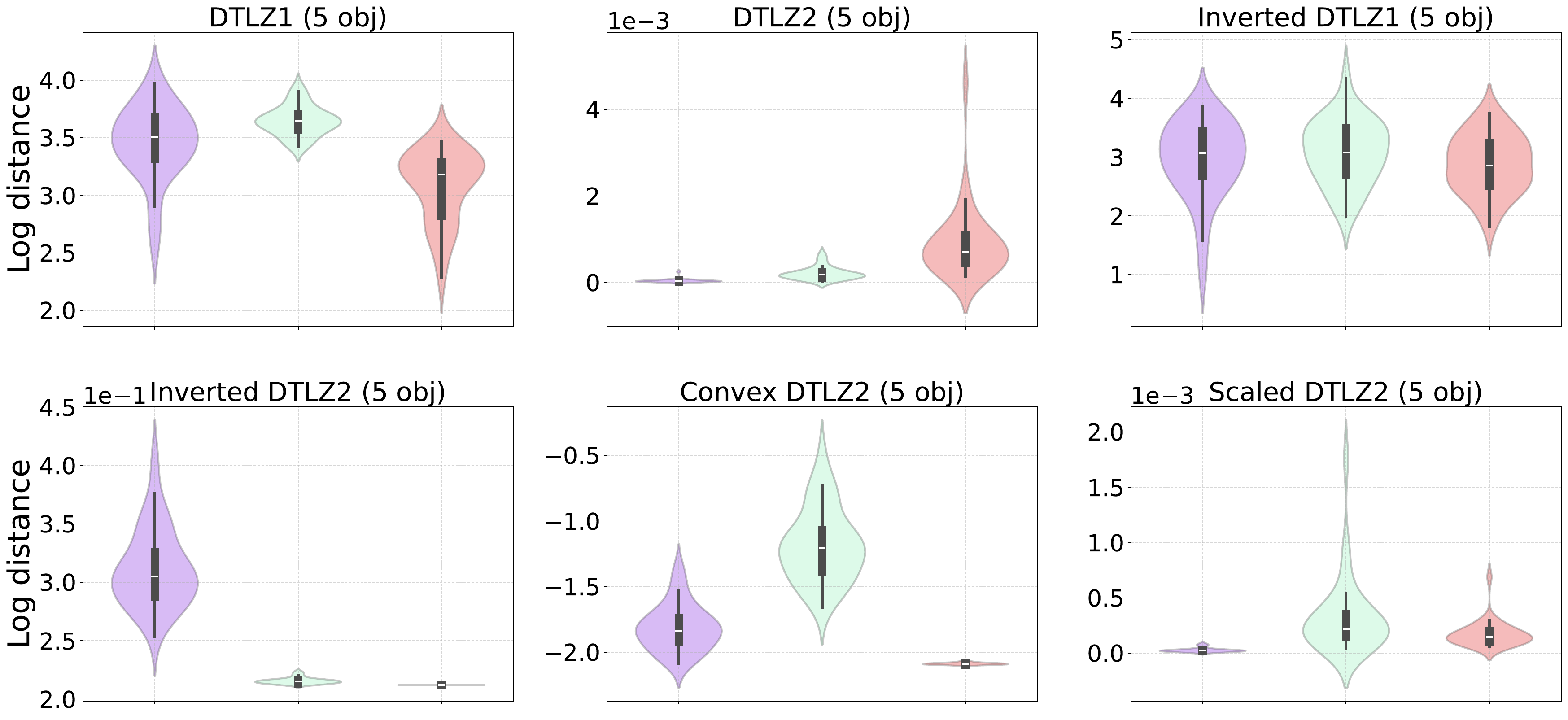}
        \end{minipage}

        \begin{minipage}{\textwidth}
            \centering
            \includegraphics[width=0.7\textwidth]{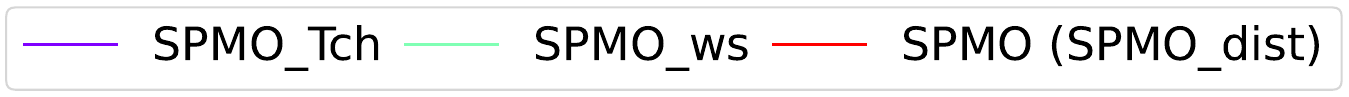}
        \end{minipage}
    \end{subfigure}

    \caption{Violin plots of the distance-based metric (log distance) obtained by the SPMO using three different single-point metrics on the problems with five objectives. 
    Each violin represents the distribution of the distance-based metric obtained by a method over 30 independent runs. 
    }
    \label{fig:violin_dist_M5_metric}
\end{figure}
\FloatBarrier

\begin{figure}[!ht]
    \centering

    \begin{subfigure}[b]{0.63\textwidth}
    \begin{minipage}{\textwidth}
        \centering
        \includegraphics[width=\textwidth]{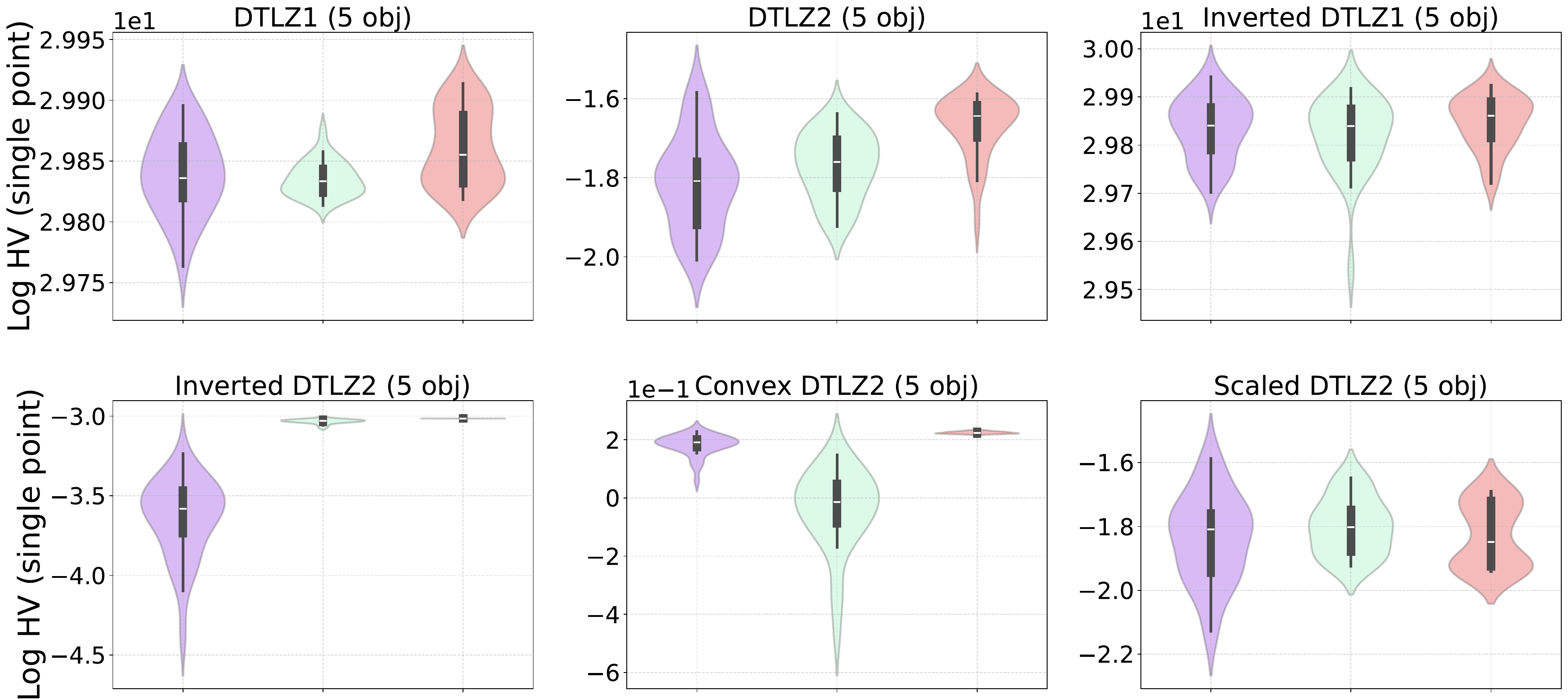}
    \end{minipage}

    \begin{minipage}{\textwidth}
        \centering
        \includegraphics[width=0.7\textwidth]{figures/distance/legend_metric.pdf}
    \end{minipage}
    \end{subfigure}
    
    \caption{Violin plots of the HV of the best solution (in terms of its HV value) obtained by the SPMO using three different single-point metrics on the problems with 5 objectives. 
    Each violin represents the distribution of the single-point HV obtained by a method over 30 independent runs. 
    \label{fig:violin_single_HV_M5_metric}
    }
\end{figure}
\FloatBarrier

\begin{figure}[!ht]
    \centering

    \begin{subfigure}[b]{0.63\textwidth}
    \begin{minipage}{\textwidth}
        \centering
        \includegraphics[width=\textwidth]{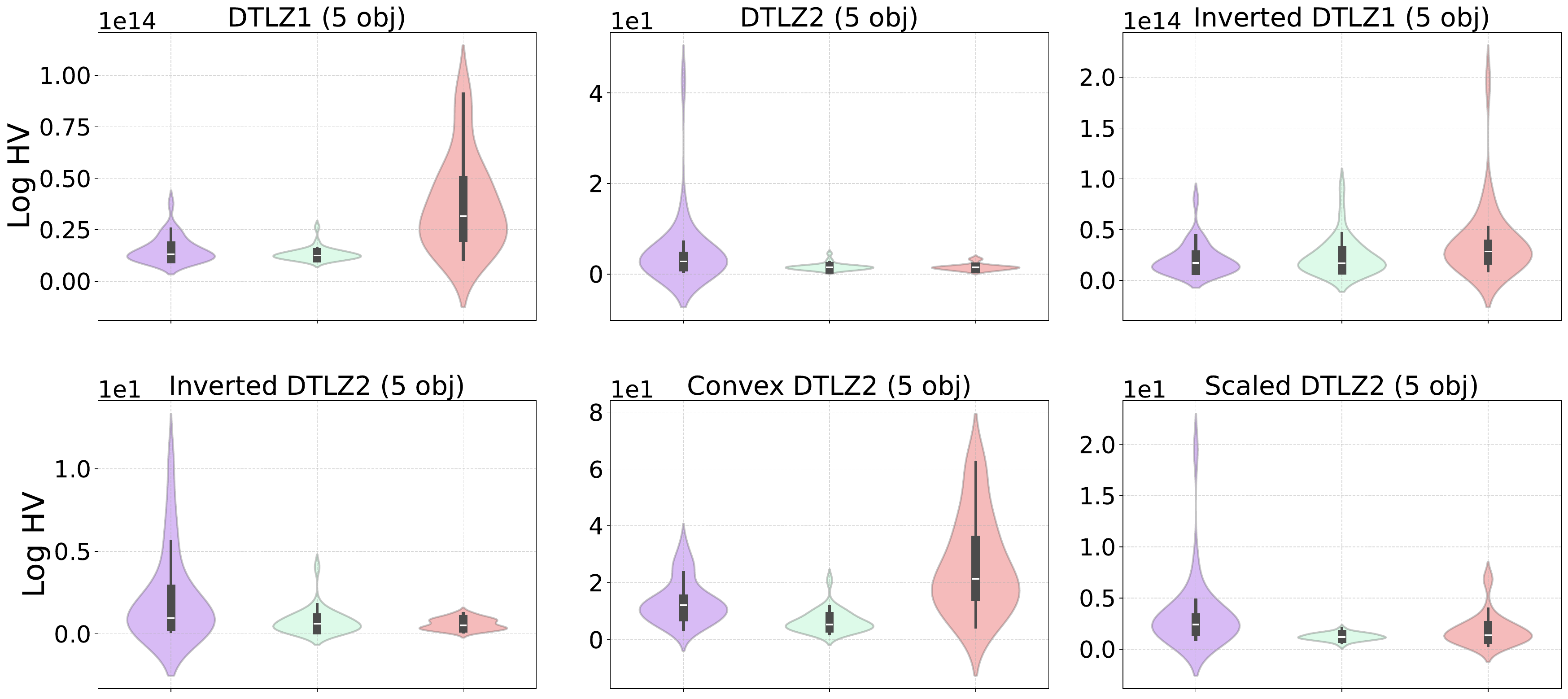}
    \end{minipage}

    \begin{minipage}{\textwidth}
        \centering
        \includegraphics[width=0.7\textwidth]{figures/distance/legend_metric.pdf}
    \end{minipage}
    \end{subfigure}

    \caption{Violin plots of the HV of all evaluated solutions obtained by the SPMO using three different single-point metrics on the problems with five objectives. 
    Each violin represents the distribution of maximum HV values obtained by a method over 30 independent runs. 
    \label{fig:violin_all_HV_M5_metric}
    }
\end{figure}
\FloatBarrier

\newpage
\subsection{Acquisition Wall Time}\label{appendix:sec:time}

Table~\ref{tbl:wall_time} presents the mean acquisition optimisation wall time of the eight methods. 
As shown, when the number of objectives is 3 or 5, the time of all the methods is acceptable with a maximum of 98 seconds. 
As the number of objectives increases to 10, hypervolume-based methods (i.e., EHVI and NEHVI) become very expensive (taking about half an hour and more than 3 hours, respectively). 
The proposed SPMO method shows high computational efficiency, achieving the lowest time requirement in four out of the six instances. 


\begin{table*}[!ht]
    \centering
    \caption{Mean acquisition optimisation wall time in seconds based on the $2(d + 1)$ initial Sobol samples on DTLZ1 problems with $m=3,5,10$ objectives, where $d = m + 4$, over 30 runs. 
    Experiments are conducted using a CPU (Intel Xeon CPU Platinum 8360Y @ 2.40 GHz) and a GPU (NVIDIA A100). 
    Note that N/A means the wall time of NEHVI on DTLZ1 with 10 objectives exceeds 3 hours. 
    }
    \resizebox{0.8\textwidth}{!}{%
    \begin{tabular}{ccccccccc}
    \toprule
    \bfseries Device\textbackslash Method & \bfseries ParEGO & \bfseries NParEGO & \bfseries TS-TCH & \bfseries EHVI & \bfseries NEHVI & \bfseries C-EHVI & \bfseries JES  & \bfseries SPMO (ours) \\
    \midrule
    \bfseries CPU (3 obj)  & 3.46  & 3.49  & 12.29  & 4.23  & 6.15  & 10.23 & 9.24  & 2.40 \\
    \bfseries GPU (3 obj)  & 5.19  & 3.94 &  14.74 & 3.05  & 5.73  & 9.48 & 14.02  & 2.24 \\
    
    \bfseries CPU (5 obj)  & 2.59  & 2.14  & 28.32  & 36.03  & 97.98  & 29.89 & 25.16  & 2.58 \\
    \bfseries GPU (5 obj)  & 9.02  & 10.33 & 22.29  & 10.32  & 63.90  & 23.89 & 25.80  & 5.31 \\
    \bfseries CPU (10 obj) & 8.76  & 5.97  & 87.44  & 1134.35 & N/A     & 77.94& 388.37 & 6.95 \\
    \bfseries GPU (10 obj) & 41.33 & 30.30 & 64.78  & 2426.04 & N/A     & 86.45 & 676.31 & 24.05 \\
    \bottomrule
    \end{tabular}
    \label{tbl:wall_time}
    }
    
\end{table*}

\section{Applicability of the Proposed SPMO}\label{appendix:sec:discussion}

In conventional multi-objective optimisation, algorithms are designed to approximate the entire Pareto front, so that a decision-maker can later select a preferred solution based on their own preferences. 
This is the ideal situation, as a well-represented Pareto front provides the most comprehensive view of the possible trade-offs~\citep{jiang2025multi}. However, under tight evaluation budgets - especially when many objectives are involved - it is often unrealistic, if not impossible, to obtain a good approximation of the Pareto front.
In such settings, the decision-maker may benefit more from receiving a well-balanced solution that is close to the Pareto front, rather than a well-distributed solution set far from the front. 
The proposed SPMO framework is designed for this purpose: instead of spreading search effort across the whole front, it directs the optimisation towards a well-balanced solution. 
With a focused search effort, the obtained solution is often closer to the front, and thus has a higher likelihood of being selected by the decision-maker.


It is worth pointing out that if the optimisation problem under consideration is extremely simple (e.g., smooth, unimodal landscape with a very limited search space), on which finding a good representation of the entire Pareto front is possible under tight budgets, then our approach may not be desirable. 
Moreover, exploring the Pareto front can help decision-makers better understand the optimisation problem and facilitate the elicitation or refinement of their preferences. In such cases, our framework is not applicable.

\section{Extensions}\label{appendix:sec:extensions}

The preceding discussion has addressed multi-objective optimisation problems in which evaluations are performed either sequentially or in batches, under both noiseless and noisy scenarios. 
However, not all multi-objective settings conform to these scenarios. 
To accommodate a broader class of problems, we propose several extensions that enable the methodology to handle more optimisation scenarios. 

\paragraph{High-Dimensional Bayesian Optimisation (HDBO).} 

High-dimensional black-box optimisation problems are highly challenging and frequently encountered in a wide range of applications. 
The dimensionality may range from tens to a billion~\citep{gonzalez2024survey,papenmeier2023bounce,santoni2024comparison,wang2016bayesian,hoang2025high}. 
To tackle such optimisation problems, various HDBO methods have been proposed~\citep{binois2022survey,chen2024pg,nayebi2019framework,wang2018batched,wang2016bayesian,xu2025standard}.
They can loosely be categorised into four classes~\citep{santoni2024comparison}, i.e., variable selection~\citep{eriksson2021high}, additive models~\citep{delbridge2020randomly,han2021high,wang2018batched,ziomek2023random}, embeddings~\citep{antonov2022high,letham2020re,raponi2020high}, and trust regions~\citep{daulton2022multi,diouane2023trego,eriksson2019scalable}. 
However, recent studies show that standard Gaussian processes without the above techniques can perform well in high-dimensional spaces~\citep{hvarfner2024vanilla,papenmeier2025understanding,xu2025standard} and suggest that the main issue in high-dimensional BO is the gradient vanishing. 
Our work can be naturally extended to the high-dimensional setting by mitigating the gradient vanishing issue~\citep{papenmeier2025understanding,xu2025standard}.

\paragraph{Multi-Fidelity Bayesian Optimisation (MFBO).} 

In many real-world optimisation scenarios, the evaluation is often available at multiple fidelity levels, where increasing fidelity typically leads to improved accuracy at the expense of higher computational cost. 
Many MFBO methods have been proposed to tackle such optimisation problems~\citep{belakaria2020multi,kandasamy2017multi,li2020multi,moss2021gibbon,song2019general,takeno2020multi,wu2020practical,zhang2017information}. 
Our proposed SPMO can be potentially extended to the multi-fidelity setting by integrating prior techniques, e.g., building multiple surrogate models of different levels of fidelity. 






\end{document}